\documentclass[10pt]{article} 

\usepackage[preprint]{rlj} 

\usepackage{amssymb, amsthm}            
\usepackage{mathtools}          
\usepackage{mathrsfs}           
\usepackage{graphicx}           
\usepackage{subcaption}         
\usepackage[space]{grffile}     
\usepackage{url}                
\usepackage{lipsum}             
\usepackage{float}
\usepackage{amsmath,amssymb,amsthm}
\usepackage{multirow,booktabs}
\usepackage{xcolor,soul}
\usepackage[ruled,longend]{algorithm2e}

\SetCommentSty{mycommfont}

\usepackage[capitalize,noabbrev]{cleveref}
\usepackage{tikz,bm}
\usetikzlibrary{arrows.meta, positioning, shapes, calc}

\newtheorem{theorem}{Theorem}[section]
\newtheorem{lemma}[theorem]{Lemma}

\newtheorem{remark}{Remark}

\newtheorem{assumption}{Assumption}

\crefname{assumption}{assumption}{assumptions}
\crefname{subsection}{subsection}{subsections}
\crefname{lemma}{lemma}{lemma}
\crefname{property}{property}{properties}
\crefname{table}{table}{tables}
\crefname{remark}{remark}{remarks}
\Crefname{remark}{Remark}{Remarks}

\usepackage{pifont}
\newcommand{\cmark}{\ding{51}}
\newcommand{\xmark}{\ding{55}}

\usepackage{threeparttable}
\usepackage{notation}
\usepackage{nicefrac}
\usepackage{doi}
\usepackage{caption}
\allowdisplaybreaks

\title{A Finite-Sample Analysis of an Actor-Critic Algorithm for Mean-Variance Optimization in a Discounted MDP}

\setrunningtitle{Mean-Variance SPSA Actor Critic }
 
\author{Tejaram Sangadi\textsuperscript{1}, Prashanth L.A.\textsuperscript{2}, Krishna Jagannathan\textsuperscript{1}}

\emails{ee20d426@smail.iitm.ac.in, prashla@cse.iitm.ac.in, krishnaj@ee.iitm.ac.in}

\affiliations{
$^{1}$\textbf{Department of Electrical Engineering, Indian Institute of Technology Madras}\\
$^{2}$\textbf{Department of Computer Science and Engineering, Indian Institute of Technology Madras}
}

\contribution{
    We consider mean-variance optimization in a discounted MDP, and derive finite-sample guarantees for an actor-critic algorithm, with a critic based on linear function approximation, and an actor based on SPSA. 
    }
    {
    We consider a mean-variance MDP with the variance of the \emph{return,} whose expectation is the usual risk-neutral objective. For this problem, existing work \citep{l.a.VarianceconstrainedActorcriticAlgorithms2016} provides only asymptotic convergence guarantees. 
    }

\contribution{
   For mean-variance policy evaluation, we employ TD learning with linear function approximation. We derive finite-sample bounds that hold (i) in the mean-squared sense and (ii) with high probability under tail iterate averaging,
 with and without regularization. Notably, our analysis for the regularized TD variant holds for a universal step size. 
    }
    {
     Non-asymptotic policy evaluation bounds are not available for variance of the return in a discounted MDP.
    }

\contribution{
    We employ an SPSA-based actor for policy optimization, and obtain an $O(n^{-1/4})$ bound in the number of actor iterations.
    }
    {
   Notably, we resort to an SPSA-based actor, since the policy gradient theorem for variance is not amenable for direct use in an actor-critic algorithm; see \cite{l.a.VarianceconstrainedActorcriticAlgorithms2016}. Further, finite-sample bounds for a SPSA-based actor-critic algorithm are not available, even in the risk-neutral RL setting, to the best of our knowledge.
    }


\summary{
In many practical applications of reinforcement learning (RL), such as finance and mobility, safety considerations are paramount. Rather than solely maximizing expected rewards, one must also account for risk to ensure reliable decision-making. Traditional RL primarily focuses on expected reward maximization, a well-studied paradigm with both empirical and theoretical breakthroughs. In this paper, we adopt an alternative approach that integrates risk-awareness into policy optimization. Despite extensive research in risk-neutral RL, analyzing risk-sensitive RL algorithms remains challenging, as each risk metric requires a distinct analytical framework. We focus on variance---an intuitive and widely used risk measure---and analyze the \textbf{M}ean-\textbf{V}ariance \textbf{S}imultaneous \textbf{P}erturbation \textbf{S}tochastic \textbf{A}pproximation \textbf{A}ctor-\textbf{C}ritic (MV-SPSA-AC) algorithm, establishing finite-sample theoretical guarantees for the discounted reward Markov Decision Process (MDP) setting.  Our analysis covers both policy evaluation and policy improvement within the actor-critic framework. We study a Temporal Difference (TD) learning algorithm with linear function approximation (LFA) for policy evaluation and derive finite-sample bounds that hold in both the mean-squared sense and with high probability under tail iterate averaging, with and without regularization. Additionally, we analyze the actor update using a simultaneous perturbation-based approach and establish convergence guarantees. These results contribute to the theoretical understanding of risk-sensitive actor-critic methods in RL, offering insights into variance-based risk-aware policy optimization.

}

\begin{document}

\makeCover  
\clearpage
\pagestyle{plain}  
\maketitle  

\begin{abstract}
Motivated by applications in risk-sensitive reinforcement learning, we study mean-variance optimization in a discounted reward Markov Decision Process (MDP). Specifically, we analyze a Temporal Difference (TD) learning algorithm with linear function approximation (LFA) for policy evaluation. We derive finite-sample bounds that hold (i) in the mean-squared sense and (ii) with high probability under tail iterate averaging, both with and without regularization. Our bounds exhibit an exponentially decaying dependence on the initial error and a convergence rate of $O(1/t)$ after $t$ iterations. Moreover, for the regularized TD variant, our bound holds for a universal step size. Next, we integrate a Simultaneous Perturbation Stochastic Approximation (SPSA)-based actor update with an LFA critic and establish an $O(n^{-1/4})$ convergence guarantee, where $n$ denotes the iterations of the SPSA-based actor-critic algorithm. These results establish finite-sample theoretical guarantees for risk-sensitive actor-critic methods in reinforcement learning, with a focus on variance as a risk measure.
\end{abstract}

\section{Introduction}
In the standard reinforcement learning (RL) setting, the objective is to learn a policy that maximizes the value function, which is the expected value of the cumulative reward obtained over a finite or infinite time horizon. However, in many practical scenarios such as finance, automated driving and drug testing, a risk sensitive learning paradigm is crucial, where the value function (an expectation) must be balanced with an appropriate risk metric associated with the reward distribution. One approach is to formulate a constrained optimization problem, using the risk metric as a constraint and the value function as the objective. Variance is a popular risk measure and is typically incorporated into risk-sensitive optimization as a constraint while optimizing for the expected value. This mean-variance formulation was introduced in the seminal work of \citet{markowitzPortfolioSelection1952}. Mean-variance optimization in RL has been studied in several works; see, e.g., \citet{mannorAlgorithmicAspectsMean2013,tamarLearningVarianceRewardToGo2016,l.a.VarianceconstrainedActorcriticAlgorithms2016}.  We study mean-variance optimization in a discounted reward Markov decision process (MDP). Our key contribution is the analysis of an actor-critic algorithm for mean-variance optimization, along with finite-sample guarantees in this setting.

\noindent\textbf{Main Contributions.}  
We study a discounted reward MDP with variance as the risk criterion and present two main contributions. Since one common approach to variance estimation is based on the difference between the second moment and the square of the first moment, estimating both moments is essential. Our first key contribution concerns the sub-problem of jointly evaluating the value function (first moment) and the second moment of the discounted cumulative reward. For simplicity, we refer to the second moment of the discounted cumulative reward as the square-value function. To address the curse of dimensionality in large state-action spaces, we analyze temporal difference (TD) learning with linear function approximation (LFA) for these estimates. 
 
We present finite-sample bounds that quantify the deviation of the iterates from the fixed point, both in expectation and with high probability. The fixed point is joint in the sense that it includes both the value function and the square-value function. We present bounds for a constant step-size with and without tail-averaging; see Table \ref{tab:summary-expec} for a summary. Next, we establish \(O(1/t)\) finite-time convergence bounds for tail-averaged TD iterates, where \(t\) denotes the number of iterations of the TD algorithm. Furthermore, we present a finite-sample analysis of the regularized TD algorithm. From this analysis, we establish an \(O(1/t)\) bound, similar to the unregularized case. An advantage of regularization is that the step-size choice is universal, i.e., it does not require knowledge of the eigenvalues of the underlying linear system, whereas the unregularized TD bounds depend on such eigenvalue information, which is typically unknown in practice.

While finite-sample analysis of TD with LFA has been studied in several recent works  \citep[cf.][]{prashanthConcentrationBoundsTemporal2021,dalalFiniteSampleAnalyses2018,bhandariFiniteTimeAnalysis2021,samsonovImprovedHighProbabilityBounds2024, agrawalPolicyEvaluationVariance2024}, to the best of our knowledge, no prior work has established finite-sample bounds for policy evaluation of variance in the discounted reward MDP setting. Our bounds explicitly characterize their dependence on the discount factor, feature bounds, and rewards. Compared to existing finite-sample bounds for TD learning, the analysis of mean-variance-style TD updates is more intricate, as it requires tracking the solution of an additional projected fixed point by solving a separate Bellman equation for the square-value function. Furthermore, the Bellman equation associated with the square-value function includes a cross-term involving the value function (see \eqref{eq:T2p} in the supplementary material). Due to this cross-term, obtaining a standard \( O(1/t) \) mean-squared error bound is challenging when using a constant step size, unless the spectral properties of the underlying linear system are known.
To overcome this dependence, we investigate a regularized version of the mean-variance TD updates. To the best of our knowledge, ours is the first work to obtain a $O(1/t)$ MSE bound with a universal step size for mean-variance TD. Prior works on TD-type algorithms for other notions of variance, cf. \cite{agrawalPolicyEvaluationVariance2024,eldowaFiniteSampleAnalysis2022}, present $O(1/t)$ bounds with a step size choice that requires underlying eigenvalue information.

\begin{table}[!t]
    \caption{Summary of the MSE bounds for a TD-critic.}
    \label{tab:summary-expec}
    \centering
    \small 
    \renewcommand{\arraystretch}{1.2} 
    \begin{threeparttable}
    \begin{tabular}{l c c c c c} 
        \toprule \multirow{2}{*}{\textbf{Paper}} &  \multirow{2}{*}{\textbf{Iterate}} & \multirow{2}{*}{\textbf{Objective}} & \multirow{2}{*}{\textbf{Rate}} & \multirow{2}{*}{\textbf{Step size}} & \textbf{Universal} \\ 
        &&&&&\textbf{step size}\\
        \midrule
        \multirow{2}{*}{\citet{l.a.VarianceconstrainedActorcriticAlgorithms2016}} & \multirow{2}{*}{Last iterate} & Mean-& \multirow{2}{*}{--\textsuperscript{1}} & \multirow{2}{*}{$\tfrac{c_{0}c}{c+t}$} & \multirow{2}{*}{\xmark}
        \\ && variance &&& \\
        \citet{dalalFiniteSampleAnalyses2018} & Last iterate & Mean & $O(1/t^\sigma)$ & $1/t^{\sigma}$ & \cmark \\
        \citet{bhandariFiniteTimeAnalysis2021}\textsuperscript{2} & Full average & Mean & $O(1/t)$ & $1/\sqrt{T}$ & \cmark \\
        \multirow{2}{*}{\citet{eldowaFiniteSampleAnalysis2022}} & \multirow{2}{*}{Full average} & Mean- & \multirow{2}{*}{$O(1/t)$} & \multirow{2}{*}{constant} & \multirow{2}{*}{\xmark} \\ && variance\textsuperscript{3} &&&\\
        \citet{patilFiniteTimeAnalysis2023} & Tail average & Mean & $O(1/t)$ & {constant} & {\cmark} \\
        \multirow{2}{*}{\citet{agrawalPolicyEvaluationVariance2024}} & \multirow{2}{*}{Tail average} & Mean- & \multirow{2}{*}{$O(1/t)$} & \multirow{2}{*}{constant} & \multirow{2}{*}{\xmark} \\ && variance\textsuperscript{4} &&& \\ \citet{mitraSimpleFinitetimeAnalysis2025} & Weighted average\textsuperscript{5} & Mean & $O(1/t)$ & constant & \xmark \\
        \multirow{2}{*}{\textbf{This work}} & \multirow{2}{*}{Tail average} & Mean- & \multirow{2}{*}{$O(1/t)$} & \multirow{2}{*}{constant} & \multirow{2}{*}{\xmark} \\ && variance &&& \\ 
        \multirow{2}{*}{\textbf{This work}} & Regularized & Mean- & \multirow{2}{*}{$O(1/t)$} & \multirow{2}{*}{constant} & \multirow{2}{*}{\cmark} \\ 
        &tail average &variance&&&\\
        \bottomrule
    \end{tabular}
    \begin{tablenotes}
      \item[]  \textsuperscript{1} Asymptotic convergence of mean-variance TD shown. Here, $c_0$ and $c$ are arbitrary constants depending on the minimum eigenvalue.
        \textsuperscript{2} T = number of TD iterations.
        \textsuperscript{3} Variance of per-step reward as the risk measure.
        \textsuperscript{4} Asymptotic variance for average-reward MDP as the risk measure.
        \textsuperscript{5} Weights are determined by \( (1 - \alpha A)^{-(t+1)} \) with \( A = 0.5\omega(1 - \gamma) \), which makes them indirectly dependent on the minimum eigenvalue \( \omega \) and the discount factor \( \gamma \). Here, \( \alpha \) is step size dependent on the minimum eigenvalue $\omega$.
    \end{tablenotes}
    \end{threeparttable}
\end{table}

Our second key contribution lies in analyzing an actor-critic algorithm for mean-variance and deriving finite-sample guarantees. The critic part uses the aforementioned LFA-based policy evaluation for a fixed policy parameter. The actor uses an SPSA-based gradient estimator \citep{spallMultivariateStochasticApproximation1992}, departing from the more common risk-neutral approach of employing a likelihood ratio-based gradient estimator supported by the policy gradient theorem (see \Cref{sec:actor} for a discussion on SPSA's necessity). SPSA estimates policy gradients for the value and square-value functions using two policy trajectories: one generated using the current policy parameter and another using a randomly perturbed parameter.

We provide non-asymptotic convergence rates for an SPSA-based actor in the mean-variance framework. This result quantifies convergence to the stationary point in terms of the gradient norm of the Lagrangian, addressing a gap in prior work that focused exclusively on asymptotic guarantees. As an aside, mean-variance optimization has been shown to be NP-hard, even with model information available \citep{mannorAlgorithmicAspectsMean2013}. Actor-critic methods present a viable alternative approach, and our analysis provides the rate of convergence for such an algorithm tailored to the mean-variance setting.
 Specifically, we show an $O(n^{-\frac14})$ performance guarantee for the overall algorithm, where $n$ is the number of actor loop iterations. To the best of our knowledge, there are no finite-sample guarantees for zeroth order actor-critic, even for the risk-neutral setting.
 
Our results are beneficial for three reasons. First, we exhibit $O(1/t)$ bounds for the regularized TD variant with a step size that is universal. In contrast, a universal step size for vanilla mean-variance TD is not feasible owing to certain cross-terms that are unique to the case of mean-variance policy evaluation. Our key observation is that regularization enables the use of a universal step size that is independent of the eigenvalues of the underlying system. Second, our proof is tailored to mean-variance TD, making the constants clear.  In contrast, it is difficult to infer them from the general LSA bounds in \citep{durmusFiniteTimeHighProbabilityBounds2024, mouLinearStochasticApproximation2020}. Third, we provide high-probability bounds that exhibit better scaling w.r.t. the confidence parameter as compared to \citet{samsonovImprovedHighProbabilityBounds2024}.

\noindent \textbf{Related Work.}
This paper performs  a finite-sample analysis of a TD critic, and an SPSA actor for mean-variance optimization in a discounted RL setting. We briefly review relevant works on each of these topics.

\textbf{Critic. } TD learning, originally proposed by \citet{suttonLearningPredictMethods1988}, has been widely used for policy evaluation in RL. \citet{tsitsiklisAnalysisTemporaldifferenceLearning1997} established asymptotic convergence guarantees for TD learning with LFA. Many recent works have focused on providing non-asymptotic convergence guarantees for TD learning \citep{bhandariFiniteTimeAnalysis2021,dalalFiniteSampleAnalyses2018,lakshminarayananLinearStochasticApproximation2018,srikantFiniteTimeErrorBounds2019,prashanthConcentrationBoundsTemporal2021,patilFiniteTimeAnalysis2023,durmusFiniteTimeHighProbabilityBounds2024}. 
In a recent study by \citet{samsonovImprovedHighProbabilityBounds2024}, the authors derived refined error bounds for TD learning by combining proof techniques from \citep{mouLinearStochasticApproximation2020,durmusFiniteTimeHighProbabilityBounds2024} with a stability result for the product of random matrices. In contrast, our results target a different system of linear equations. Moreover, as mentioned before, our bounds for regularized TD feature a universal step size.
The reader is referred to Section \ref{sec:mean-varTDwithLFA} for a detailed comparison of our critic bounds to the current literature.

\textbf{Actor-Critic. } In \citep{leiZerothOrderActorCriticEvolutionary2025}, the authors propose a zeroth-order actor critic in a risk-neutral RL setting. However, they do not provide a finite-sample analysis.
In \citep{l.a.VarianceconstrainedActorcriticAlgorithms2016}, which is the closest related work, the authors propose an SPSA-based actor-critic algorithm for mean-variance optimization, and establish asymptotic convergence. In contrast, we provide a finite-sample analysis of their algorithm with a few variations: (i) We incorporate tail-averaging in TD-critic and derive finite-sample bounds for a universal step size; (ii) We prove a smoothness result for the Lagrangian of the mean-variance problem and use this result to provide a non-asymptotic bound for the SPSA-based actor that employs mini-batching for the critic updates. In \citep{xuImprovingSampleComplexity2020,kumarSampleComplexityActorcritic2023}, the authors analyze risk-neutral actor critic algorithms with a gradient estimate based on the likelihood ratio method. They provide a finite-sample analysis. However, the likelihood ratio method for gradient estimation does not work for the case of variance, and hence, our non-asymptotic analysis involves a significant departure in the proof for the SPSA-based actor that we consider.

\section{Problem formulation} \label{sec:problemformulation}
We consider an MDP  with state space $\mathcal{S}$ and action space $\mathcal{A}$, both assumed to be finite. The reward function $r(s,a)$ maps state-action pairs $(s,a)$ to a reward, with $s \in \mathcal{S}$ and $a \in \mathcal{A}$. In this work, we consider a stationary randomized policy $\pi$ which maps each state to a probability distribution over the action space. We consider a discounted MDP setting, and use $\disc \in (0,1)$ to denote the discount factor. We use $\mathbb{P}(s'|s,a)$ to denote the probability of transitioning from state $s$ to next state $s'$ given that action $a$ is chosen following a policy $\pi$. 
The transition probability matrix $\TPM$ gives the probability of going from state $s$ to $s'$ given a policy $\pi$. The elements of this matrix of dimension $|\S|\times |\S|$ are given by $  \TPM(s,s') = \sum_{a} \pi(a|s) \mathbb{P}(s'|s,a).$ The value function $V^{\pi}(s)$, which denotes the expected value of cumulative sum of discounted rewards when starting from state $s_{0}=s$ and following the policy $\pi$, is defined as
\begin{equation} \label{eq:valuefn}
V^\pi(s) \triangleq \E\left[\textstyle\sum_{t=0}^{\infty}\disc^{t} r(s_{t},a_{t})\,\middle|\,s_0=s\right].
\end{equation} 
Furthermore, the variance of the infinite horizon discounted reward from state $s_0 = s$, denoted as $\Lambda^\pi(s)$, is defined as $\Lambda^\pi(s) \triangleq U^\pi(s)-V^\pi(s)^2,$
where $U^\pi(s)$ represents the second moment of the cumulative sum of discounted rewards, and is defined as
\begin{equation} \label{eq:secondmoment}
U^\pi(s) \triangleq \E\left[\left(\textstyle \sum_{t=0}^{\infty}\disc^{t} r(s_{t},a_{t})\right)^{2} \,\middle|\, s_0=s\right].
\end{equation}
Henceforth, we shall refer to $U^\pi$ as the square-value function. The well-known mean-variance optimization problem in a discounted MDP context is as follows: For a given state \( s_0 = s \) and threshold \( c > 0 \), our goal is to solve the following constrained optimization problem:
\begin{equation}
\max_\pi V^\pi(s)\quad\quad \text{subject to} \quad\quad \Lambda^\pi(s)\leq c.\label{eq:mean-var-mdp}
\end{equation}
The value function $V^{\pi}(s)$ satisfies the Bellman equation $T_{1}V^{\pi} = V^{\pi}$, where $T_{1}: \mathbb{R}^{|\S|} \rightarrow \mathbb{R}^{|\S|}$ is the Bellman operator, defined by 
\(T_{1}(V^{\pi}(s_{0})) \triangleq \E^{\pi,\TPM}\psq{r(s_{0},a_{0})+\disc V^{\pi}(s')},\)
where the actions are chosen according to the policy $\pi$.
It is well known that $T_1$ is a contraction mapping. In \citet{sobelVarianceDiscountedMarkov1982}, the author derives a Bellman type equation for $\Lambda^\pi(s)$. However, the underlying operator of this equation is not monotone. To workaround this problem, \citet{tamarLearningVarianceRewardToGo2016,l.a.VarianceconstrainedActorcriticAlgorithms2016} use the square-value function $U^\pi(s)$, which satisfies a fixed point relation that is monotone. Given $V^\pi(s), U^\pi(s)$, the variance can be calculated using $\Lambda^\pi(s)$. Using Proposition 6.1 in \citep{l.aRiskSensitiveReinforcementLearning2022}, we expand the square-value function \eqref{eq:secondmoment} as 
\[ \resizebox{\textwidth}{!}{$ U^{\pi}(s) = \sum_{a} \pi(a|s) r(s,a)^{2} + \disc^{2} \sum_{a,s'}\pi(a|s)\mathbb{P}(s'|s,a)U^{\pi}(s') +2\disc \sum_{a,s'} \pi(a|s)\mathbb{P}(s'|s,a)r(s,a)V^{\pi}(s')$}
\]
Similar to the value function, the square-value function also satisfies a Bellman equation $\textstyle T_{2}U^{\pi}=U^{\pi}$, where $\textstyle T_{2}:\mathbb{R}^{|\S|} \rightarrow \mathbb{R}^{|\S|}$ is the Bellman operator, given by
\(\textstyle{
    T_{2}U^{\pi}(s) \define \; \E^{\pi, \TPM} [r(s,a)^{2}\!+\!\disc^{2} U^{\pi}(s')\! +\!2\disc r(s,a) V^{\pi}(s')].}\)
For a given policy $\pi$, the Bellman operators $T_{1}$ and $T_{2}$ can be represented in a compact vector-matrix form as
\(T_{1}(V)=r+\disc \TPM V\), \( T_{2}(U) = \tilde{r}+2 \disc \DiagR \TPM V+\disc^2 \TPM U,\) 
where $U$, $V$, $r$ and $\tilde{r}$ are $|\S|\times 1$ vectors with $r(s_{i}) \!=\! \sum_{a \in \mathcal{A}} \pi(a|s_{i}) r(s_{i},a), \; \tilde{r}(s_{i}) \!=\! \sum_{a \in \mathcal{A}} \pi(a|s_{i}) r(s_{i},a)^{2}$. Here, $\DiagR$ is a $|\S|\times|\S|$ diagonal matrix with $r(s_{i})$ as the diagonal elements for $i \in \{1,\dots,|\S|\}$. Now, we construct an operator $T:\mathbb{R}^{2|\S|} \rightarrow \mathbb{R}^{2|\S|}$, which is given by
\(T(V,U) = (T_{1}(V), T_{2}(U))\tr\)
A sub-problem of \eqref{eq:mean-var-mdp} is policy evaluation, i.e., estimation of $V^\pi(\cdot)$ and $\Lambda^\pi(\cdot)$ for a given policy $\pi$. \citet{l.aRiskSensitiveReinforcementLearning2022,tamarLearningVarianceRewardToGo2016}  establish that the operator $T$  is a contraction mapping with respect to a weighted norm, ensuring a unique fixed point for $T$.  In the next section, we describe a TD algorithm with LFA for policy evaluation, and this algorithm is based on \citep{l.a.VarianceconstrainedActorcriticAlgorithms2016}.

\section{Mean-variance TD-critic} \label{sec:mean-varTDwithLFA}

When the state space size $|\mathcal{S}|$ is large, policy evaluation suffers from the curse of dimensionality, as it requires computing and storing the value function for each state in the MDP. A standard approach to overcome this difficulty is to use TD learning with \emph{function approximation}, wherein the value function is approximated using a simple parametric class of functions. The most common example of this is TD learning with LFA \citep{tsitsiklisAnalysisTemporaldifferenceLearning1997}, where the value function for each state is approximated using a linear parameterized family, i.e., $V^\pi(s)\approx\omega\tr \phi(s),$ where $\omega\in\mathbb R^q$ is a tunable parameter common to all states, and $\phi: \mathcal S \rightarrow \mathbb R^q$ is a feature vector for each state $s\in \mathcal S,$ and typically $q\ll |\mathcal S|.$   

We approximate the value function $V^\pi(s)$ and the square-value function $U^\pi(s)$ using linear functions as follows: $V^\pi(s) \approx v^\top\phi_v(s),   \quad
 U^\pi(s) \approx u^\top\phi_u(s)$,
where the features $\phi_v(\cdot)$ and $\phi_u(\cdot)$ belong to low-dimensional subspaces in $\R^{d_1}$ and $\R^{d_2}$, respectively. Let $\Fiv$ and $\Fiu$ denote $|\mathcal{S}|\times d_{1}$ and $|\mathcal{S}|\times d_{2}$  dimensional matrices, with $i\text{-th}$ and $j\text{-th}$ column respectively as $\begin{pmatrix} 
   \phi_{v}^{i}(s_{1}), \dots, \phi_{v}^{i}(s_{|\S|}) 
\end{pmatrix}^\top,$ 
$\begin{pmatrix} 
   \phi_{u}^{j}(s_{1}), \dots, \phi_{u}^{j}(s_{|\S|}) 
\end{pmatrix}^\top$ where $i \in \{1,\dots,d_{1}\}$ and $j \in \{1,\dots,d_{2}\}$. 
For analytical convenience, in our analysis we set $d_{1}=d_{2}=q$. 
We observe that owing to the function approximation, the actual fixed point remains inaccessible. Instead, the objective is to find the projected fixed points, denoted as $\bar{w}=(\bar{v},\bar{u})\tr $ within the following subspaces:
\( S_{v} \coloneqq \left\{ \Fiv v \; \middle| v \in \mathbb{R}^{d_{1}} \right\}, 
\,\,S_{u} \coloneqq \left\{ \Fiu u \; \middle| u \in \mathbb{R}^{d_{2}} \right\}.
\)
We approximate the value and square-value functions within the subspaces defined above.
Accordingly, we construct projections onto $S_{v}$ and $S_{u}$ with respect to a weighted norm, using the stationary distribution as weights. 
For the analysis, we require the following assumptions that are standard for TD with LFA,  \citep[cf.][]{prashanthConcentrationBoundsTemporal2021,bhandariFiniteTimeAnalysis2021,srikantFiniteTimeErrorBounds2019,patilFiniteTimeAnalysis2024}.
\begin{assumption}
        \label{asm:stationary}
        The Markov chain underlying the policy $\pi$ is irreducible. 
\end{assumption}
\begin{assumption}\label{asm:phiFullRank}
    The matrices $\Fiv$ and $\Fiu$ have full column rank. 
\end{assumption} \vspace{-1ex}
With finite state and action spaces, \Cref{asm:stationary} guarantees the existence of a unique stationary distribution $\chi_{\pi}$ for the Markov chain induced by policy $\pi$. 
\Cref{asm:phiFullRank}, commonly made in the context of TD with LFA (cf. \citet{bhatnagarNaturalActorCritic2009,bhandariFiniteTimeAnalysis2021,prashanthConcentrationBoundsTemporal2021}), mandates that the columns of the feature matrices $\Fiv$ and $\Fiu$ be linearly independent, guaranteeing the uniqueness of the fixed points. Additionally, it also ensures the existence of inverse of the feature covariance matrices ($\Fiv\tr \statdist^{\pi}\Fiv$ and $\Fiu\tr \statdist^{\pi}\Fiu$), to define the projection matrices in \eqref{eq:projection operators}. We denote $\mathbf{\Pi}_{v}$ and $\mathbf{\Pi}_{u}$ as the projection matrices which project from state space $\mathcal{S}$ onto the subspaces $\S_{v}$ and $S_{u}$, respectively. For a given policy $\pi$, projection matrices are defined as: 
\begin{equation} \label{eq:projection operators}
    \mathbf{\Pi}_{v} = \Fiv(\Fiv\tr \statdist^{\pi}\Fiv)^{-1}\Fiv\tr \statdist^{\pi} \textrm{ and } \mathbf{\Pi}_{u} = \Fiu(\Fiu\tr \statdist^{\pi}\Fiu)^{-1}\Fiu\tr \statdist^{\pi},
\end{equation}
where $\mathbf\Pi_v$ and $\mathbf\Pi_u$ project the true value and square-value functions onto the linear spaces spanned by the columns of $\Fiv$ and $\Fiu$, respectively. In the above,
$\statdist^\pi$ is a diagonal matrix with entries from the stationary distribution $\chi$.
In \citep{l.a.VarianceconstrainedActorcriticAlgorithms2016}, the authors established the following projected fixed point relations:
\begin{equation}
\Fiv\bar v = \mathbf{\Pi}_v T_v (\Fiv\bar v), \text{ and } \Fiu\bar u = \mathbf{\Pi}_u T_u(\Fiu\bar u).\label{eq:tdfixedpts}
\end{equation}
\citep[Proposition 6.2]{l.aRiskSensitiveReinforcementLearning2022} establishes that the joint operator  
\(T(V,U) =\) \scalebox{0.8}{$\begin{pmatrix} T_{v} \\ T_{u} \end{pmatrix}$}
is a contraction with respect to a weighted norm. Since the operator  
$\mathbf{\Pi} =$ \scalebox{0.8}{$\begin{pmatrix} \mathbf{\Pi}_v &0\\ 0&\mathbf{\Pi}_u \end{pmatrix}$}
is non-expansive and the matrices \( \Fiv \) and \( \Fiu \) have full column rank,  \citep[Proposition 8]{tamarLearningVarianceRewardToGo2016} ensures that the projected Bellman operator \( \Pi T(V, U) \) is also a contraction with respect to a weighted norm. Consequently, the projected Bellman operator \( \Pi T(V, U) \) admits a unique projected fixed point \( \bar{w} = (\bar{v}, \bar{u})^{\top} \).  
The equations in \eqref{eq:tdfixedpts} can be rewritten as the linear system  
\begin{equation}
\label{eq:Mwxi}
    -\M \bar{w} + \xi = 0, \quad \text{where} \quad  
    \M = \begin{pmatrix}
        \Fiv\tr \statdist(\I-\disc\TPM)\Fiv & 0 \\
        -2\disc\Fiu\tr \statdist \DiagR \TPM \Fiv & \Fiu\tr \statdist(\I-\disc^{2}\TPM)\Fiu 
    \end{pmatrix}, \quad
    \xi = \begin{pmatrix}
        \Fiv\tr \statdist \DiagR \\
        \Fiu\tr \statdist \tilde{r} 
    \end{pmatrix},
\end{equation}  
where \( r = (r(s_1), \dots, r(s_{|\S|}))\tr \), and \(\DiagR\) is a diagonal matrix with components  
\( r(s_i) =~ \sum_{a \in \mathcal{A}} \pi(a|s_i) r(s_i, a) \) for \( i \in \{1,\dots,|\S|\} \).  
Similarly, \( \tilde{r} \) is a vector with components \( \tilde{r}(s_i) =~ \sum_{a \in \mathcal{A}} \pi(a|s_i) r(s_i, a)^2 \).  

\paragraph{Basic algorithm.}
\begin{algorithm}[!ht]
\small
\SetNoFillComment
\DontPrintSemicolon  
\SetAlgoInsideSkip{smallskip} 
\caption{TD with Tail Averaging (Critic)}
\label{alg:TD-Critic-main}
\KwIn{Initialize $w_{0}=(v_{0},u_{0})$, step-size $\beta$, critic batch size $m$, tail index $k$}
\KwOut{Tail-averaged iterate $w_{k+1:m} = (\frac{1}{m-k}\sum_{t=k+1}^{m} v_{t} \,,\,\frac{1}{m-k}\sum_{t=k+1}^{m} u_{t})\tr$}
\For{$t = 0$ \KwTo $m$}{
    Sample action $a_t$ using the policy $\pi(\cdot | s_{t})$, observe the next state $s_{t+1}$ and reward $r_t = r(s_{t},a_{t})$ \;
    \tcc{Update the TD parameters as follows:} 
    \vspace{-3ex}
    \begin{equation}
     v_{t+1} = v_{t} + \lr \; \delta_{t} \; \fiv(s_{t}), \quad u_{t+1}  = u_{t} + \lr \; \epsilon_{t} \; \fiu(s_{t})\label{eq:v-td-update}
    \end{equation}
    \(\begin{aligned}
    \textrm{ where } \delta_{t} & = r_t + \disc v_{t}\tr  \fiv(s_{t+1}) - v_{t}\tr  \fiv(s_{t}), \\
    \epsilon_{t} & = {r_t}^{2} + 2\disc r_t v_{t}\tr  \fiv(s_{t+1})  
    + \disc^{2} u_{t}\tr  \fiu(s_{t+1}) - u_{t}\tr  \fiu(s_{t}). 
    \end{aligned}\)
}
\end{algorithm}
Letting $w_{t} = (v_{t}, u_{t})\tr $, we rewrite \eqref{eq:v-td-update} to obtain the following update iteration:
\begin{equation}  \label{eq:Mt-update}
    w_{t+1} = w_{t} + \lr (r_t \phi_{t}-\M_{t} w_{t}),
\end{equation} where \( \textstyle \phi_t =(\fiv(s_{t}), r(s_t,a_{t})\fiu(s_{t}) )\tr, \M_t \define \begin{pmatrix}
                                \A_{t} &  \zeromat \\
                                \C_{t} & \B_{t} \\ 
                          \end{pmatrix} \textrm{ with }
    \C_{t} \define -2\disc r_{t} \fiu (s_{t}) \fiv (s_{t+1})\tr, \)
\( \A_{t} \define \fiv (s_{t}) \fiv (s_{t})\tr  - \disc \fiv(s_{t}) \fiv(s_{t+1})\tr \textrm{ and }\B_{t} \define \fiu (s_{t}) \fiu (s_{t})\tr  - \disc^{2} \fiu(s_{t}) \fiu(s_{t+1})\tr. 
\)

In \eqref{eq:Mt-update}, we have used  $r_{t}$ to denote $r(s_{t},a_{t})$, for notational convenience. We observe that the expected value of $\M_t$ is equal to $\M$, where $\M$ is defined in \eqref{eq:Mwxi}. An alternative view of the update rule is the following:
\begin{equation}
     w_{t+1} = w_{t}+\lr (-\M w_{t}+\xi+\Delta M_{t}), \label{eq:M-update} 
\end{equation}  
where $\Delta M_{t} = r_{t} \phi_{t}-\M_{t} w_{t}-\E \psq{r_{t} \phi_{t}-\M_{t} w_{t} \mid \F_{t}}$, with $\xi$ as defined in \eqref{eq:Mwxi}. 
Under an i.i.d. observation model (see \Cref{asm:iidNoise}), $\Delta M_t$ is a martingale difference w.r.t. the filtration $\{\F_t\}_{t\ge 0}$, where  $\F_t$ is the sigma field generated by $\{w_0,\dots,w_{t}\}$. 
We remark that we utilize the update iteration \eqref{eq:Mt-update} instead of \eqref{eq:M-update} to obtain finite-sample bounds in the next section. The rationale behind this choice is a technical advantage of not requiring a projection operator to keep the iterates $w_t$ bounded. To elaborate, in the proof of finite-sample bounds, we unroll the iteration in \eqref{eq:Mt-update}  and bound the bias and variance terms. Specifically, letting $z_t = w_t - \bar w$ and $ h_{t}(w_{t}) = r_{t} \phi_{t} - \M_{t} w_{t} $, we get $z_{t+1} = (\I-\lr \M_{t})z_{t}+\lr h_{t}(\bar{w}).$
The second term $h_t(\bar{w})$ does not depend on the iterate $w_t$ and can be bounded directly. On the other hand, unrolling \eqref{eq:M-update} would result in a term $\lr \Delta M_t$ in place of the $h_t(\bar{w})$, and bounding this term requires a projection since $\Delta M_t$ has the iterate $w_t$.

\citet{tsitsiklisAnalysisTemporaldifferenceLearning1997} show asymptotic convergence of $v_t$ to $\bar v$. They achieved this by verifying that the required conditions---on step-size, stability, and noise control---are satisfied with the TD update reinterpreted as as Linear Stochastic Approximation (LSA) iteration.  Similarly, the convergence of $w_t$ to $\bar w$ was established by \citet{l.a.VarianceconstrainedActorcriticAlgorithms2016}.
Several recent works have analyzed the finite-sample behavior of TD learning with LFA, particularly focusing on deriving mean-squared error bounds \citep{bhandariFiniteTimeAnalysis2021}. However, a direct finite-sample analysis of \eqref{eq:Mt-update} is not available in the literature---a gap that we address next.

\paragraph{Bounds for the TD-critic.} 

We make the following assumptions that are common in the finite-sample analysis of temporal difference (TD) learning, \citep[cf.][]{prashanthConcentrationBoundsTemporal2021,bhandariFiniteTimeAnalysis2021,patilFiniteTimeAnalysis2024}.
\begin{assumption}\label{asm:bddFeatures}
    $\forall s \in \mathcal{S}$,  $\|\fiv(s)\|_{2} \leq \fivmax < \infty$, $ \|\fiu(s)\|_{2} \leq \fiumax < \infty $.
\end{assumption}
\begin{assumption}\label{asm:bddRewards}
    $\forall s \in \mathcal{S}, a\in \mathcal{A},$   $\left|r(s,a)\right| \leq R_{\max}<\infty$.
\end{assumption} \Cref{asm:bddFeatures} ensures the existence of the feature covariance matrices $\Fiv\tr \statdist^{\pi}\Fiv$ and $\Fiu\tr \statdist^{\pi}\Fiu$, as well as the projection matrices in \eqref{eq:projection operators}.
 \Cref{asm:bddRewards} bounds the rewards uniformly, ensuring the existence of the value function and the square-value function. We consider an i.i.d observation model, which is made precise in the assumption below.
\begin{assumption}\label{asm:iidNoise}
        The samples $\{s_{t}, r_{t}, s_{t+1}\}_{t \in \mathbb{N}}$ are formed as follows: For each $t$, $(s_{t}, s_{t+1})$ are drawn independently and identically from $\chi(s)\TPM(s,s')$,  where $\chi$ is the stationary distribution underlying policy $\pi$, and $\TPM$ is the transition probability matrix of the Markov chain underlying the given policy $\pi$. Further, $r_t$ is a function of $s_t$ and $a_t$, which is chosen using the given policy $\pi$.
\end{assumption}
The i.i.d. observation model serves as a first step in analyzing TD learning. The resulting finite-time bounds extend to the Markovian setting via the constructions in \citep[Remark 6]{patilFiniteTimeAnalysis2024} and \citep[Section 5]{samsonovImprovedHighProbabilityBounds2024}.  
\paragraph{Mean-Squared Error Bounds.} We first present a mean-squared error bound for the last iterate with a constant step size, with the proof in \Cref{sec:Appendix-expectationbd}.
\begin{theorem}\label{thm:expectation_bound}
Suppose \Crefrange{asm:stationary}{asm:iidNoise} hold. Run TD updates in \eqref{eq:v-td-update} for $t$ iterations with a step size $\lr$ satisfying the following constraint: $\lr \leq \lr_{\max}= \frac{\mu}{c}$ where $\mu=\eigmin$ and
$ c = \max \big \{ 4 (\fivmax)^{4}  \!+\!4 \disc^{2} R_{\max}^{2} {(\fiumax)}^{2}(\fivmax)^{2}  \bm{,} 
4(\fiumax)^{4} \} \!+\! 2 \disc R_{\max} ((\fivmax)^{2}(\fiumax)^{2}\!+\!(\fiumax)^{4})$.
Then, we have 
\begin{equation} \label{eq:td-expec-bd}
\E \psq{\norm{w_{t+1} - \bar{w}}^{2}_{2}} \leq  2\exp\p{-\lr \mu t} \E \psq{\norm{z_{0}}^{2}_2} + \frac{2\lr\sigma^{2}}{\mu},
\end{equation}
where $w_0$ is the initial parameter, $\bar{w}$ is the TD fixed point, $z_{0} = w_{0}-\bar{w}$ is initial error and 
\resizebox{\textwidth}{!}{$\sigma^{2} =  2  R_{\mathsf{max}}^{2} \big((\fivmax)^{2} + R_{\max}^{2} (\fiumax)^{2}\big)+ 2\big((\fivmax)^{4} \p{1+\disc}^{2}+ (\fiumax)^{4} \p{1+\disc^{2}}^{2} + 4 \disc^{2} R_{\max}^{2}(\fivmax)^{2}(\fiumax)^{2} \big) \norm{\bar{w}}_{2}^{2}. $}
\end{theorem}
Notice that the bound in \eqref{eq:td-expec-bd} is for a constant stepsize that requires information about the minimum eigenvalue of the symmetric part of $\M$. In the context of regular TD, such a problematic eigenvalue dependence has been surmounted using tail-averaging, which we introduce next. We remark that tail-averaging for the case of mean-variance TD does not overcome the eigenvalue dependence. However, the benefit of tail averaging is that we obtain a bound that vanishes as as $t\rightarrow \infty$, while the bound in \eqref{eq:td-expec-bd} does not vanish asymptotically. 

\paragraph{Tail averaging.}
The tail-average is computed by averaging the iterates \(\{w_{k+1}, \dots, w_{t}\}\), given by  
\(w_{k+1:t} = \frac{1}{t - k} \sum_{i=k+1}^{t} w_{i},\)
where \(k\) is the tail index, and averaging starts at \(k+1\).
\citet{polyakAccelerationStochasticApproximation1992,fathiTransportentropyInequalitiesDeviation2013} investigated the advantages of iterate averaging, providing the asymptotic and non-asymptotic convergence guarantees in the stochastic approximation literature, respectively. Tail averaging preserves the advantages of iterate averaging, while also ensuring dependence on initial error is forgotten at a faster rate \citep{patilFiniteTimeAnalysis2023, samsonovImprovedHighProbabilityBounds2024}. Now, we present a mean-squared error bound for the tail-averaged variant of the TD-critic, with the proof in \Cref{sec:Appendix-tail-avg-expectationbd}.
\begin{theorem}\label{thm:expectation_bound-tailavg}
Suppose \Crefrange{asm:stationary}{asm:iidNoise} hold. 
Run \Cref{alg:TD-Critic-main} for $t$ iterations with a step size $\lr$ as specified in \Cref{thm:expectation_bound}.
Then, we have the following bound for  the tail average iterate \({w_{k+1:t}= \frac{1}{t-k} \sum_{i=k+1}^{t} w_{i}}\):
\begin{equation}
\E \psq{\norm{w_{k+1:t} - \bar{w}}^{2}_{2}} \leq \frac{10  \exp{(-k \lr \mu)}}{\lr^2 \mu^2 (t-k)^{2}}\E[\norm{z_{0}}^{2}_2]+ \frac{10\sigma^{2}}{\mu^2 (t-k)},\label{eq:tailavg-expt-bd}
\end{equation}
where $z_0,\sigma,\bar w, \mu$ are as defined in \Cref{thm:expectation_bound}.
\end{theorem}
As in the case of regular TD with tail averaging, it can be observed that the initial error (the first term in \eqref{eq:tailavg-expt-bd}) is forgotten exponentially. The second term, with $k=t/2$ (or any other fraction of $t$), decays as $O({1}/{t})$.
Tail averaging is advantageous when compared to full iterate averaging (i.e., $k=1$), as the latter would not result in an exponentially decaying initial error term. The bound for regular TD with tail averaging in \citet{patilFiniteTimeAnalysis2024} uses a universal step-size, which does not require information about the eigenvalues of the underlying feature matrix. However, arriving at $O({1}/{t})$ bound for the case of variance is challenging owing to certain cross-terms that cannot be handled in a manner analogous to regular TD, see \Cref{appendix:critic-outline} for the details.

\paragraph{Regularization for universal step size.}
The results in Theorems \ref{thm:expectation_bound}--\ref{thm:expectation_bound-tailavg} suffer from the disadvantage of a stepsize which requires knowledge of the spectral properties of the underlying matrix $\M$. In practical RL settings, such information is seldom available. To circumvent this shortcoming, we propose a regularization-based TD algorithm that works with a universal step size, for a suitably chosen regularization parameter. Instead of \eqref{eq:Mwxi}, we solve the following regularized linear system for some $\zeta>0$:
\begin{equation}
    -(\M+\zeta \I) \bar w_{\reg} + \xi =0,\label{eq:Mwxi-reg} 
\end{equation}
The corresponding TD updates in \eqref{eq:v-td-update} to solve \eqref{eq:Mwxi-reg}  would become
\begin{equation}
\check v_{t+1} = (\I-\lreg \zeta)\check v_{t}+\lreg \; \check\delta_{t} \; \fiv(s_{t}), \quad 
\check u_{t+1}  = (\I-\lreg \zeta) \check u_{t}+\lreg \; \check\epsilon_{t}\; \fiu(s_{t}), \label{eq:vu-td-update-reg}
\end{equation} where $\check\delta_{t},\check\epsilon_{t}$ are the regularized variants of the corresponding quantities defined in \eqref{eq:v-td-update}, i.e., with $v_t,u_t$ replaced by $\check v_t, \check u_t$ respectively. We combine the updates in \eqref{eq:vu-td-update-reg} as \begin{equation} \label{eq:Mt-update-reg}
    \check w_{t+1} =  \check w_{t} + \check \lr (r_{t} \phi_{t}-(\zeta \I +\M_{t}) \check w_{t}),
\end{equation} where $M_t,r_t,\phi_t$ are defined in \eqref{eq:Mt-update}. We now present a result that shows the regularized tail-averaged variant \eqref{eq:Mt-update-reg} converges at the optimal rate of $O({1}/{t})$ in the mean-squared sense, for a step size that is universal. 
\begin{theorem} \label{thm:regTDlamdasq} 
Suppose  \Crefrange{asm:stationary}{asm:iidNoise} hold. Let $\check w_{k+1:t} = \frac{1}{t-k} \sum_{i=k+1}^{t-k} \check w_{i}$ denote the tail-averaged regularized iterate. For \(\zeta = \frac{1}{\sqrt{t-k}}\) and the step size $\lreg$ satisfying $\lreg \leq \check{\lr}_{\max} =  \frac{\zeta}{\check{c}}$. Then, \[\resizebox{\textwidth}{!}{$\E\left[\norm{\check w_{k+1:t} - \bar{w}}_{2}^{2}\right]  \leq \frac{  5\exp{(-k \lreg \mu)}}{\lreg^2 \mu^2 N^{2}}\E\left[\norm{{\check w}_{0} - \bar{w}_{\mathsf{reg}}}^{2}_2\right] \!+\! \frac{5\check \sigma^{2}}{\mu^2 N} \!+\! \frac{2 (R_{\mathsf{max}}^{2} \p{(\fivmax)^{2} + R_{\mx}^{2} (\fiumax)^{2}})}{\iota^4 N}.$}\]
where $\check{c}$ and $\check{\sigma}$ are defined in \Cref{sec:Appendix-regTDlamdasq}, $\iota$ denotes the minimum singular value of $\M$, $N = t - k$, and $\mu = \eigmin$
\end{theorem}
We first bound $\E\left[\norm{\check w_{k+1:t} - \wreg}_{2}^{2}\right]$ in \Cref{thm:regtd-expectation_bound} in the supplementary material, specialize this bound for the case of $\zeta=\frac{1}{\sqrt{t-k}}$. Next, using the fact that $\norm{\bar w_{\mathsf{reg}} - \bar{w}}^{2}_{2}$ is $O(\zeta^2)$, followed by a triangle inequality, we obtain the bound in the theorem above, see \Cref{sec:Appendix-regTDlamdasq} for the proof.

\paragraph{High-probability bounds.}
For the high probability bound, we consider the following update rule:
$ w_{t+1} = \Gamma(w_{t}+\disc h_{t}(w_{t})),$
where $\Gamma$ projects on to the set $\mathcal{C} \define \{w \in \mathbb{R}^{2q} \mid \norm{w}_2 \leq H \}$. 
\begin{assumption}\label{asm:projection}
 The projection radius $H$ of the set $\mathcal{C}$ satisfies $H > \frac{\norm{\xi}_2}{\mu}$, where $\mu=\eigmin$ and $\xi$ is as defined in \eqref{eq:Mwxi}.
\end{assumption}

Under the additional projection-related assumption, we establish a high-probability bound for the tail-averaged iterate in \Cref{alg:TD-Critic-main}. We then derive a high-probability bound for the regularized tail-averaged iterate. The following theorem provides a high-probability bound for the unregularized (vanilla) mean-variance TD, with proofs for both regularized and unregularized cases given in \Cref{appendix:high-probability-bd}.
\begin{theorem}\label{thm:high-prob-bound}
    	Suppose \Crefrange{asm:stationary}{asm:projection} hold. Run \Cref{alg:TD-Critic-main} for $t$ iterations  with step size $\lr$ as defined in \Cref{thm:expectation_bound-tailavg}. Then, for any  $\delta\in (0,1]$, we have the following bound for the projected tail-averaged iterate $w_{k+1:t}$:
    \begin{align*}
        &\textstyle \mathbb{P}\left(\norm{w_{k+1:t} \!-\! \bar{w}}_{2} \leq \frac{2\tau}{\mu \sqrt{t-k}}\sqrt{\log\left(\frac{1}{\delta}\right)} \!+\! \frac{4 \exp \p{-k \lr \mu}}{\lr \mu N}\E\left[\norm{w_{0} - \bar{w}}_{2}\right] \!+\! \frac{4\tau}{\mu \sqrt{t-k}} \right) \!\geq\! 1 \!-\! \delta,
    \end{align*}
    where $w_0, \bar{w},\lr$ are defined as in \Cref{thm:expectation_bound}, and 
    \[\resizebox{\textwidth}{!}{$
    \tau = \big(2  R_{\mathsf{max}}^{2} \p{(\fivmax)^{2} + R_{\max}^{2} (\fiumax)^{2}} + 2\big( (\fivmax)^{4} \p{1+\disc}^{2}+ (\fiumax)^{4} \p{1+\disc^{2}}^{2} + 4 \disc^{2} R_{\max}^{2}(\fivmax)^{2}(\fiumax)^{2} \big) H^{2} \big)^{\frac{1}{2}}.$}
    \]
\end{theorem} 
The next theorem provides a high-probability bound for the regularized tail-averaged iterate.
\begin{theorem}\label{thm:main-regtd-high-prob} 
Assume that the conditions in \Crefrange{asm:stationary}{asm:projection} hold.
Run the regularized version of \Cref{alg:TD-Critic-main}, specified by \eqref{eq:Mt-update-reg}, for $t$ iterations with a step size $\lreg \leq \lreg_\mx$ as specified in \Cref{thm:regTDlamdasq}.
Then, for any $\delta \in (0,1]$, with probability at least $1 - \delta$, the tail-averaged regularized TD iterate, after projection, satisfies
\[\resizebox{\textwidth}{!}{$\norm{\check{w}_{k+1:t} - \bar{w}_{\mathsf{reg}}}_{2} \leq 
\frac{2\check \tau}{(2\mu+\zeta)\sqrt{N}} \sqrt{\log\left(\frac{1}{\delta}\right)} 
+ \frac{4 \exp (-k \lreg (2\mu + \zeta))}{\check \lr (2\mu+\zeta) N} \mathbb{E} \norm{w_{0}-\bar{w}_{\mathsf{reg}}}_{2} 
+ \frac{4 \check \tau}{(2\mu + \zeta)\sqrt{N}}.$}\]
where $N$, $\check w_0$, $\bar{w}_{\mathsf{reg}}$, and $\mu$ are defined as in Theorem \ref{thm:regTDlamdasq}. Moreover, 
\resizebox{\textwidth}{!}{$
\check{\tau} = \left( 2 R_{\mathsf{max}}^{2} \big( (\fivmax)^{2} + R_{\mx}^{2} (\fiumax)^{2} \big) 
+ 4 \big( \zeta^2+ (\fivmax)^{4} (1+\disc)^{2} + (\fiumax)^{4} (1+\disc^{2})^{2} 
+ 4 \beta^{2} R_{\mx}^{2} (\fivmax)^{2} (\fiumax)^{2} \big) H^{2} \right)^{\frac{1}{2}}.
$}
\end{theorem}
We use a martingale decomposition and Lipschitz concentration of sub-Gaussian random variables to establish the high-probability bounds. This technique has been employed for vanilla TD \citep{prashanthConcentrationBoundsTemporal2021}. Our contribution extends this technique to mean-variance TD and its regularized variant, enabling a universal step size. As in the MSE bound case, owing to the cross terms, a universal step size does not appear to be feasible sans regularization, and we believe this is a useful finding as it deviates from the corresponding result for vanilla TD. In contrast, the authors in \citep{samsonovImprovedHighProbabilityBounds2024} employ Berbee’s coupling lemma to arrive at a sub-exponential tail bound.

\paragraph{Discussion:}
The update rule in \eqref{eq:Mt-update} represents a Linear Stochastic Approximation (LSA), and mean-variance TD is indeed a special case of the general LSA framework. Several previous works, including \citet{srikantFiniteTimeErrorBounds2019}, provide a finite time analysis for LSA. Their bounds can be applied to \eqref{eq:Mt-update}. However, our analysis differs in the following ways:
First, the step size $\epsilon$ in \citet{srikantFiniteTimeErrorBounds2019} depends on the eigenvalues of the transition probability matrix $P$, which can be difficult to obtain. We alleviate this dependency by employing regularization to achieve a universal step size that is independent of spectral information.
Second, we derive explicit constants for the matrix $\mathbf{M}$ (mean-variance TD) instead of the matrix $\mathbf{A}$ (vanilla TD).
Third, our analysis focuses on the recursive structure of the error to the projected fixed point, whereas \citet{srikantFiniteTimeErrorBounds2019} analyze the drift of a Lyapunov function.
Finally, \citet{srikantFiniteTimeErrorBounds2019} provide finite-time bounds for Mean Squared Error, while we additionally establish high-probability bounds.

The current literature on bounds for TD (or more generally, linear stochastic approximation) for Polyak-Ruppert averaging scheme does not achieve $O({1}/{t})$ bounds, to the best of our knowledge. Instead, with a Polyak-Ruppert stepsize $1/k^\alpha$, the bound is $O(
{1}/{t^\alpha})$, with $\alpha <1$, see \citep{prashanthConcentrationBoundsTemporal2021}. Tail-averaging with a “universal” step size was shown to close this gap for vanilla TD. Our contribution is to show that tail-averaging with universal step size may not be feasible to obtain an $O({1}/{t})$ for mean-variance TD, while regularization closes this gap. In \citet{samsonovImprovedHighProbabilityBounds2024}, the authors provide high-probability bounds for a general linear stochastic approximation algorithm, and specialize them to obtain bounds for the regular TD algorithm. For mean-variance TD \eqref{eq:Mt-update}, we could, in principle, apply the bounds from the aforementioned reference. However, the bound that we derive in Theorem \ref{thm:high-prob-bound} enjoys a better dependence on the confidence parameter $\delta.$ Specifically, we obtain a \(\sqrt{\log(1/\delta)}\) actor, corresponding to a sub-Gaussian tail, while the bounds in \citet{samsonovImprovedHighProbabilityBounds2024} feature a \(\log(1/\delta)\)  factor, which is equivalent to a sub-exponential tail. Furthermore, our result makes all constants clear in the case of mean-variance TD.

\section{SPSA-based Actor} \label{sec:actor}
\vspace{-1ex}
In this section, we analyze an actor algorithm based on SPSA-based gradient estimates. Throughout, we consider a parametrized class of stationary randomized policies $\{\pi_\theta,\ \theta\in\R^d\}.$  
We denote the score function as $\psi_{\theta}(s,a) = \nabla_{\theta} \log \pi_{\theta}(a|s)$. We consider smoothly-parameterized polices, i.e., satisfying the following assumptions:
\begin{assumption}\label{asm:scorefnandpolicy}
	 $\forall (s,a)\in \mcs\times\mca$ and $\ton,\tto \in \R^d$, $\exists$ positive constants $L_\psi$,  $C_\psi$ and $C_\pi$ such that\\
 (i) $\pnorm{\psi_{\ton}(s,a)-\psi_{\tto}(s,a)} \leq L_{\psi}\pnorm{\ton-\tto}$; (ii) $\pnorm{\psi_\theta(s,a)}\leq C_\psi$;\\ (iii)  $\|\pi_{\ton}(\cdot|s)-\pi_{\tto}(\cdot|s)\|_{TV} \leq C_\pi\pnorm{\ton-\tto}$, where $\|\cdot\|_{TV}$  denotes the total-variation norm.
\end{assumption}
\begin{algorithm}[!ht]
\small
\DontPrintSemicolon
\SetNoFillComment
\caption{SPSA-based actor with TD critic for mean-variance optimization (MV-SPSA-AC)}
\label{alg:SPSA-Actor-Critic-main}
\KwIn{Initialize $\theta_0 \in \R^d$, perturbation constant $\{p_t\}$, critic batch size $m$, actor step size  $\{\alpha_t\}$, critic step size $\{\lr_t\}$, number of iterations $n$, and tail-index k.}
\For{$t \leftarrow 0$ \KwTo $n-1$}{
Generate  $\Delta(t) \sim \{\pm1\}^d$ (symmetric Bernoulli) \;
    \tcc{\textbf{Critic:} Obtaining tail-averaged TD iterates for policy evaluation}
    Run \Cref{alg:TD-Critic-main} for the unperturbed policy $\pi_{\theta_t}$ to compute \(\textstyle w_{k+1:m} = (v_{k+1:m},\, u_{k+1:m})\tr\)\;
    Run \Cref{alg:TD-Critic-main} for the perturbed policy \(\pi_{\theta_t + p_t \Delta(t)}\) to compute \(\textstyle w_{k+1:m}^{+} = (v_{k+1:m}^{+},\, u_{k+1:m}^{+})\tr\).
\tcc{\textbf{Actor:} Estimating SPSA gradients for policy improvement}
\({\nabla}_i \hat{J}(\theta) = \dfrac{\fiv(s_{0})\tr (v_{k+1:m}^{+} -v_{k+1:m})}{p_t \Delta_i(t)} ; {\nabla}_i \hat{U}(\theta) =     \dfrac{\fiu(s_{0})\tr (u_{k+1:m}^{+}-u_{k+1:m})}{p_t \Delta_i(t)} 
\) \; 
$\begin{aligned}
        \theta_{t+1} = \theta_{t} \!+\! \alpha_{t}(\nabla \hat{J}(\theta_{t})\!-\!  \lambda(\nabla \hat{U}(\theta_{t})\!-\!2 {\hat J}(\theta_{t}) \nabla \hat{J}(\theta_{t}))) 
\end{aligned}$\;
}
\KwOut{Final policy $\theta_R$ chosen uniformly at random from $\{\theta_1,\ldots,\theta_n\}$
} \end{algorithm} 
In the above, (i) and (ii) imply that score function is smooth and bounded. This generally holds for most commonly used policy classes.
Since we asssume finte action space, (iii) holds for any smooth policy. A similar assumption has been made earlier for the analysis of actor-critic algorithms in a risk-neutral RL setting, cf. \citep{xuImprovingSampleComplexity2021}.
By applying the Lagrangian relaxation procedure \citep{bertsekasConstrainedOptimizationLagrange1996} to \eqref{eq:mean-var-mdp}, we get the following unconstrained optimization problem for a fixed $\lambda \geq 0$: 
\begin{align} 
\min_\theta L(\theta) = -V^{\pi_{\theta}}(s_{0}) + \lambda (\Lambda^\pi_{\theta}(s_{0})-c), \label{eq} 
\end{align} 
where $L(\theta)$ represents the Lagrangian function. In this paper, we treat $\lambda$ as a fixed bias-variance tradeoff parameter, and find a `good-enough' policy parameter for the problem \eqref{eq} defined above. For the actor update, we require the gradient of the Lagrangian w.r.t. the policy parameter $\theta$,
\begin{align}
 \nabla_{\theta} L(\theta)&=-\nabla V_\theta(s_0)+\lambda(\nabla U_\theta(s_0)-2 V_\theta(s_0) \nabla V_\theta(s_0)). \label{eq:gradL}
\end{align}
For notational simplicity, we let $ V_\theta(s_0)=J(\theta),  U_\theta(s_0)=U(\theta), \text{ and } \nabla V_{\theta}(s_{0})=\nabla J(\theta)$.
\paragraph{Basic algorithm.}
We describe the \textbf{M}ean \textbf{V}ariance \textbf{SPSA} \textbf{A}ctor \textbf{C}ritic (MV-SPSA-AC) algorithm for mean-variance optimization. 
\Cref{alg:SPSA-Actor-Critic-main} presents the pseudocode of this algorithm. This algorithm is a variant of  the actor-critic algorithm proposed in \citet{l.a.VarianceconstrainedActorcriticAlgorithms2016}, where the authors provide only asymptotic guarantees. 
MV-SPSA-AC algorithm deviates from their algorithm by incorporating tail averaging in the TD critic with LFA, and performing a mini-batch update for the SPSA-based actor. More importantly, we perform a finite-sample analysis.
\paragraph{Need for SPSA.} The variance of the return we consider lacks a simple linear Bellman equation, unlike the value function in risk-neutral RL. To address this, variance is estimated as the difference between the second moment and the square of the first moment of the return. Since the second moment satisfies a simple linear Bellman equation, this approach makes variance estimation feasible. The policy gradient expression for the square-value function is as follows (see \citep{l.a.VarianceconstrainedActorcriticAlgorithms2016} for the derivation): 
\begin{equation}  
    \nabla U(\theta) \!=\! \tfrac{1}{1 - \disc^2} \big( \underbrace{ \!\textstyle \sum_{s,a}\! \tilde{\nu}_{\theta}(s, a) \nabla \log \pi_\theta(a|s) W_\theta(s, a)}_{T_1(\theta)} + 2\disc  \underbrace{\textstyle \!\sum_{s,a,s'}\! \tilde{\nu}_{\theta}(s, a) P(s' | s, a) \nabla V_{\theta}(s')}_{T_2(\theta)} \big). \label{eq:GradUmain1}  
\end{equation}  
As seen from the expression above,  the second term  \( T_2(\theta) \) requires the gradient $\nabla V_{\theta}(s')$ for every state $s' \in \S$. An actor-critic algorithm would require an estimate of the value gradient with every possible start state, making it impractical for implementations.
SPSA-based gradient estimates offer a viable alternative to overcome this issue. $W_{\theta}(s,a)$ is equivalent of action-value function for $U(\theta)$. 
\paragraph{Actor.}
The policy parameter $\theta$ is updated in the negative direction of gradient of the Lagrangian, with step size $\alpha_{t}$ as follows:
 \begin{equation}
        \theta_{t+1} = \theta_{t} \!+\! \alpha_{t}(\nabla \hat{J}(\theta_{t})\!-\!  \lambda(\nabla \hat{U}(\theta_{t})\!-\!2 \hat{ J}(\theta_{t}) \nabla \hat{J}(\theta_{t}))),\label{eq:actor-update}
\end{equation}
where \eqref{eq:spsa-grad-est} is used for computing $\nabla \hat{J}(\theta_{t})$ and $\nabla \hat{U}(\theta_{t})$ respectively. 
In a risk-neutral RL setting, the usual recipe for the actor part is to use the policy gradient theorem to form likelihood ratio-based gradient estimates. In \citet{l.a.VarianceconstrainedActorcriticAlgorithms2016}, it is shown that such an approach does not extend to cover the mean-variance case. The authors there proposed an alternative actor that uses SPSA for gradient estimation. This scheme uses two policy trajectories: one with parameter $\theta_t$ and another with a perturbed parameter $\theta_t + p_t \Delta(t)$, denoted by the superscript ‘+’, where $\Delta(t)$ is a $d$-dimensional vector of independent Rademacher ($\pm 1$) random variables. Using these two trajectories, we form estimates of the gradient of the value and square-value functions as follows:
 \begin{equation}
       {\nabla}_i \hat{J}({\theta_t}) =  
        \frac{\fiv(s_{0})\tr (v_{k+1:m}^{+} -v_{k+1:m})}{p_t \Delta_i(t)}, \;
{\nabla}_i \hat{U}(\theta) =     \frac{\fiu(s_{0})\tr (u_{k+1:m}^{+}-u_{k+1:m})}{p_t \Delta_i(t)}, \label{eq:spsa-grad-est} 
\end{equation}
where \( v_{k+1:m} \) and \( v_{k+1:m}^{+} \) are the tail-averaged critic parameters for the value function under the unperturbed (\(\theta_t\)) and perturbed (\(\theta_t + p_t \Delta(t)\)) policy parameters, respectively. Here, \( m \) is the critic batch size. Similarly, \( u_{k+1:m} \) and \( u_{k+1:m}^{+} \) are the tail-averaged critic parameters for the square-value function under the unperturbed and perturbed policy parameters, respectively. We describe next the  policy evaluation components in the critic.

\paragraph{Critic.} We perform $m$ TD-critic updates to form the  estimates for value function $\hat{J}(\theta) =~ \phi_{v}(s_{0})\tr v_{k+1:m}$ and square-value function $\hat{U}(\theta)=\phi_u (s_{0})\tr u_{k+1:m}$, respectively. Further, we perform  $m$ updates for the perturbed policy $\theta_t+p_t\Delta(t)$ to form the value and square-value function estimates as  
$\hat{J}(\theta+p_{t}\Delta(t)) = \phi_{v}(s_{0})\tr v_{k+1:m}^+$ and $\hat{U}(\theta+p_{t}\Delta(t))=\phi_u (s_{0})\tr u_{k+1:m}^+$, respectively.
We use tail-averaged critic variants for each policy evaluated above.

\textbf{Main results.}
For every policy $\theta$, we assume
 \Cref{asm:stationary} holds, which implies the existence of the stationary distribution $\chi_{\pi_\theta}$, and scalars $\kappa > 0$ and $\rho \in (0,1)$ such that
$
\sup_{s \in S} \left\| \mathbb{P}(s_t  \mid s_0 = s) - \chi_{\pi_\theta} \right\|_{\text{TV}} \leq \kappa \rho^t, \quad \forall t \geq 0.
$ For the analysis of MV-SPSA-AC algorithm, we need to establish that the Lagrangian $L(\cdot)$ is a smooth function of $\theta.$ Further, it can be seen from \eqref{eq:gradL} that the smoothness of $J(\cdot)$ and $U(\cdot)$ would imply to smoothness of $L(\cdot)$. 
In a risk-neutral setting, $J(\cdot)$ is the usual objective, and \citep[Proposition 1]{xuImprovingSampleComplexity2021} established smoothness of $J(\cdot)$ in \eqref{eq:GradJLipschitz}. On the other hand, smoothness of $U(\cdot)$ requires a new proof, and involves significant departures from the one for $J(\cdot)$. The result below states smoothness for $J(\cdot)$ and $U(\cdot)$, with the latter result being a technical contribution of this paper.
\begin{lemma}\label{lemma:Lipschitz}Suppose Assumptions \ref{asm:scorefnandpolicy} holds. Then, for any $\ton,\tto \in \R^d$, we have 
\begin{equation}
 \label{eq:GradJLipschitz}
 \pnorm{\nabla J(\ton)-\nabla J(\tto)} \leq L_J\pnorm{\ton-\tto}, \; 
\pnorm{\nabla U(\ton) - \nabla U(\tto)} \leq L_U \|\ton - \tto\|,  
\end{equation}
where $ L_J= \tfrac{R_{\max}}{(1-\disc)}(4C_\nu C_\psi+L_\psi),$ $C_\nu=\tfrac{1}{2}C_\pi\left( 1 +  \lceil \log_\rho \kappa^{-1} \rceil + (1-\rho)^{-1} \right)$ and \( L_U =\tfrac{1}{1-\disc^2}(\tfrac{R_{\mx}^2}{(1-\disc)^2}(L_{\psi}+4 C_{\psi}C_{\nu}(1+\frac{\disc}{R_{\mx}}))+2L_{J})\).
\end{lemma}
We remark that the smoothness result for the square-value function in \Cref{lemma:Lipschitz}, derived in the context of variance as a risk measure, holds independent significance, as it  may prove useful in variants of actor-only or actor-critic methods for mean-variance optimization. Using smoothness of $J(\cdot)$ and $U(\cdot)$, we arrive at the following result. 
\begin{lemma} \label{lemma:LipschitzGradL} Let $L_o=L_{J}\Big(1+2 \lambda \tfrac{R_{\mx}}{(1-\disc)^2}+2 \lambda  \big(\tfrac{R_{\mx}C_{\psi}}{(1-\disc)^2}\big)^2\Big)+\lambda L_{U}$. For any $\ton, \tto \in \mathbb{R}^d$, we have
  \begin{equation} \label{eq:LforLagrangian}
      \pnorm{\nabla L(\ton)-\nabla L(\tto)} \leq  L_{o}\pnorm{\ton- \tto}.
  \end{equation}
\end{lemma}
\vspace{-2ex}
The smoothness claim in the result above for the Lagrangian is a key technical contribution, as it serves as a building block for the analysis of the actor update. In particular, this smoothness result facilitates an SGD-type analysis for the actor update. For the analysis of \Cref{alg:SPSA-Actor-Critic-main}, we make the following assumption that ensures  the value and square-value functions lie in a linear space. 
\begin{assumption}\label{asm:criticapproxerror}
    For any given policy parameter $\theta$, let $\bar{v}(\theta), \bar{u}(\theta)$ denote solutions to fixed point equations in \eqref{eq:tdfixedpts}. Then, $\E[\phi(s_{0})\tr \bar{v}(\theta)] = J(\theta), \E[\phi(s_{0})\tr \bar{u}(\theta)] = U(\theta).$
\end{assumption}
A similar assumption is made in \citep[Eq.~(13)]{kumarSampleComplexityActorcritic2023}. Our analysis can be easily extended to include an approximation error term if \Cref{asm:criticapproxerror} does not hold. 
The main result that establishes stationary convergence of the algorithm MV-SPSA-AC is given below (see \Cref{actor:proof-sketch} for a proof sketch and \Cref{appendix: spsa-based-actor} for the detailed proof).
\begin{theorem}\label{thm:actor}
Suppose \Crefrange{asm:stationary}{asm:criticapproxerror} hold.
Run MV-SPSA-AC\footnote{We employ the un-regularized variant of TD-critic for deriving the bound above. The modification to use the regularized critic for the analysis is straightforward, and we omit the details. }
 for $n$ iterations with actor step size $\alpha_{t}\equiv\alpha = 1/n^{3/4}$, perturbation constant $ p_{t} \equiv p = 1/n^{1/4},$ critic batch size $m=n$, and critic step size $\lr \leq \lr_{\mx}$ as defined in \Cref{thm:expectation_bound}. Let $\theta_{R}$ be chosen uniformly from $\{\theta_{1},\dots,\theta_{n}\}$. Then,
\begin{align*}
    \E\psq{\norm{\nabla L (\theta_{R})}^2}  \leq C/n^{1/4}, 
\end{align*}
for some constant $C$ that is specified in \Cref{appendix: spsa-based-actor}.
\end{theorem}
\begin{remark}
    We need to account for the biased nature of the SPSA gradient estimators in our analysis. This introduces the perturbation constant \( p_t \), leading to the terms \( \mathcal{O}(\frac{1}{p}) \), \( \mathcal{O}(\frac{1}{p_t^2}) \), and \( \mathcal{O}(p_t) \). Consequently, we face a trade-off that arises due to the bias in the SPSA gradient estimates, acting as a bottleneck. 
\end{remark}
\begin{remark}
\citet{eldowaFiniteSampleAnalysis2022} study the variance of per-step rewards, analyzed as reward volatility \citep{bisiRiskaverseTrustRegion2020, zhangMeanVariancePolicyIteration2021}, which is also equivalent to the discount-normalized variance in \citep{filarVariancepenalizedMarkovDecision1989}. Unlike the variance of the return, this objective lends itself to a REINFORCE-type policy gradient algorithm and does not require a zeroth-order gradient estimation scheme. This is because the gradient of the variance of per-step rewards does not feature a `problematic' term like $T_2(\cdot)$; instead it only has a term analogous to $T_1(\cdot),$ which can be more easily handled similar to the risk-neutral case.   
\end{remark}

The result above establishes the convergence to a stationary point of Lagrangian, and this is significant because $L(\theta)$ encapsulates both the mean and variance of returns. Optimizing $L(\theta)$ ensures a tradeoff between maximizing the value function and minimizing variance. This result is particularly notable as it establishes convergence guarantees for a non-convex function. Mean-variance optimization has been shown to be NP-hard even if the transition dynamics are available, see \citep{mannorAlgorithmicAspectsMean2013}. Policy-gradient and actor-critic algorithms present a viable alternative where the usual convergence guarantees are to a stationary point. For instance, several policy gradient-type algorithms have been shown to converge to an approximate stationary point in the literature, cf.
\citep{xuImprovingSampleComplexity2021, zhangGlobalConvergencePolicy2020}. 

We remark on the sample complexity required for $\epsilon$-accurate convergence of the MV-SPSA-AC algorithm. Theorem~\ref{thm:actor} indicates that the actor loop must run $\Omega(\epsilon^{-4})$ times. However, in each iteration, the critic is executed twice---once for the perturbed and once for the unperturbed trajectories---using $O(n)$ samples per run to estimate the policy gradients.  Thus, the total sample complexity for $\epsilon$-accurate convergence is $O(\epsilon^{-4})$. While this represents slow convergence, the use of biased SPSA gradient estimates typically degrades the rate.  To the best of our knowledge, finite-sample results for zeroth-order actor-critic methods remain unavailable, even in risk-neutral RL \citep{leiZerothOrderActorCriticEvolutionary2025}. Investigating whether sharper analyses or stronger assumptions could improve the convergence rate is an interesting direction for future work.

\section{Concluding remarks} \label{sec:conclusion}
We considered a risk-aware discounted reward MDP through mean-variance optimization. Specifically, we analyzed an mean-variance actor-critic algorithm, and derived finite-sample performance guarantees. We first obtained an $O({1}/{t})$ bound on the convergence of the tail-averaged iterate of the mean-variance TD with LFA. We also obtained a high probability bound that effectively exhibits a sub-Gaussian tail. Next, we employed an SPSA-based actor in conjunction with the above critic, and obtained an $O(n^{-{1}/{4}})$ convergence guarantee in the number $n$ of actor iterations.

\bibliography{main}
\bibliographystyle{rlj}

\beginSupplementaryMaterials

\section{Outline of critic analysis} \label{appendix:critic-outline}
 Below, we sketch the proof of \Cref{thm:expectation_bound} to highlight the main ideas and key differences from the standard TD proof. Full proofs of \Cref{thm:expectation_bound} and \Crefrange{thm:expectation_bound-tailavg}{thm:main-regtd-high-prob} are provided in Appendices \ref{sec:Appendix-expectationbd}--\ref{appendix:high-probability-bd}.

As in proofs of standard TD bounds, we perform a bias-variance decomposition to obtain
\begin{align} 
        \E\psq{\|z_{t+1}\|^{2}} & \leq 2\underbrace{\E\left[\left\|\mathbf{C}^{t:0}z_{0}\right\|^{2}\right]}_{z_{t}^{\text{bias}}} +2\lr^{2} \underbrace{\E \psq{\norm{\sum_{k=0}^{t} \mathbf{C}^{t:k+1} h_{k}(\bar{w})}^{2}}}_{z_{t}^{\text{variance}}},\label{eq:biasvardecomposition-mainpaper}
\end{align}
where  
$   \mathbf{C}^{i:j}  = 
    \begin{cases}
    (\I-\lr \M_{i})(\I-\lr \M_{i-1}) \dots (\I-\lr \M_{j}) \quad &\text{if } i \geq j \\
    \; \I & \text{otherwise.} 
    \end{cases} $ 

To bound the bias term, we expand the matrix product by one step, yielding  
\begin{align*}
    z_{t}^{\text{bias}} & = \E \psq{\norm{\mathbf{C}^{t:0}z_{0}}^{2}} \\
    & = \E \psq {\E \psq{
    \p{\mathbf{C}^{t-1:0}z_{t-1}^{\text{bias}}}\tr  \p{\I-\lr \M_{t}}\tr 
    \p{\I-\lr \M_{t}}
     \p{\mathbf{C}^{t-1:0}z_{t-1}^{\text{bias}}}
      \,\middle|\, \F_{t} } }.
\end{align*}
Next, we establish a result for any \( y \in \mathbb{R}^{2q} \) that aids in handling both the bias and variance terms.
\begin{align*}
 \E \psq{y\tr \p{\I-\lr \M_{t}}\tr \p{\I-\lr \M_{t}} y \; \middle| \; \F_{t}}  &=  \norm{y}_{2}^{2} -\lr \underbrace{y\tr \E \psq{\p{\M_{t}\tr +\M_{t}} | \F_{t}}y}_{\colorcircle{blue!20}{T1}} \\
& \quad +\lr^{2} \underbrace{y\tr \E \psq{\M_{t}\tr \M_{t} \; \middle| \;\F_{t}}y}_{\colorcircle{blue!20}{T2}}  \numberthis \label{eq:ytrCyp}
\end{align*} 
The term T1 is lower-bounded in a standard manner (as in regular TD), i.e.,
\begin{align*}
    y\tr  \E \psq{\p{\M_{t}\tr +\M_{t}}\;\middle|\; \F_{t}} y &= y\tr  \p{\M\tr +\M} y  \geq 2 \mu \norm{y}_{2}^{2}, \numberthis \label{eq:T1bdp}
\end{align*}
where $\mu = \eigmin$ is the minimum eigenvalue of the matrix \(\frac{\M+\M\tr}{2}\).

On the other hand, bounding term T2 involves significant deviations. In particular,
\begin{align*}
     y\tr  \E \psq{\M_{t}\tr \M_{t}\;\middle|\; \F_{t}} y  
    &=  \underbrace{v\tr  \E \psq{\A_{t}\tr \A_{t}+\C_{t}\tr \C_{t}\;\middle|\; \F_{t}}v}_{\colorcircle{gray!20}{S1}}
        + \underbrace{u\tr  \E \psq{\B_{t}\tr \B_{t}\;\middle|\; \F_{t}}u}_{\colorcircle{gray!20}{S2}}\nonumber\\
        &\qquad
        + \underbrace{v\tr  \E \psq{\C_{t}\tr \B_{t}\;\middle|\; \F_{t}} u}_{\colorcircle{gray!20}{S3}} 
        + \underbrace{u\tr  \E \psq{\B_{t}\tr \C_{t}\;\middle|\; \F_{t}} v}_{\colorcircle{gray!20}{S4}}. \numberthis \label{eq:T2p}
\end{align*}
Here, S1 and S2 resemble terms that appear in the finite-sample analysis of regular TD, while S3 and S4 are cross-terms specific to the estimation of the square-value function.

We bound S1, S2 as follows:
\begin{align} \label{ineq: S12bds}
S1 & \leq \p{(\fivmax)^{2} \p{1+\disc}^{2}+ 4 \disc^{2} R_{\max}^{2} {\fiumax}^{2}} v\tr  \matB v, \\
S2 &\leq  (\fiumax)^{2} \p{1+2\disc^{2}+\disc^{4}} u\tr  \matG u. \nonumber
\end{align} 
In the above, $\matB$ and $\matG$ are expectations of the outer product of vectors $\fiv(s_{t})$ and $\fiu(s_{t})$ respectively. 
If the cross-terms were not present, then one could have related T2 to a constant multiple of $v\tr  \matB v + u\tr  \matG u$, leading to a universal step size choice, in the spirit of \citet{patilFiniteTimeAnalysis2024}.
However, cross-terms present a challenge to this approach, and we bound the S3, S4 cross-terms as follows:
\begin{align} \label{ineq: S34bds}
    S3 + S4 \le 2 (\fiumax)^{2} R_{\max} v\tr \p{\disc (\matB+\matG)+\disc^{3} (\matB+\matG) } u.
\end{align}
We overcome the challenge of bounding the cross-terms (S3 and S4) through the following key observations: First, the cross-terms exhibit symmetry and are equal. Consequently, analyzing one term suffices, as the derived upper bound applies to the other term as well. Second, to bound the cross-term, we leverage the following inequality:
\[
- v\tr  \p{\frac{aa\tr +bb\tr }{2}}u \leq v\tr  \p{ ab\tr } u \leq v\tr  \p{\frac{aa\tr +bb\tr }{2}} u.
\]
A similar inequality, also employed in bounding S1 and S2, simplifies the bound in terms of the matrices $\mathbf{B}$ and $\mathbf{G}$, resulting in the expression in \eqref{ineq: S34bds}. \\
Combining the bounds on S1 to S4 in conjunction with the fact that $v \tr  (\matB+\matG) u  \leq \frac{\lambda_{\mathsf{max}}(\matB+\matG)}{2} \norm{y}_{2}^{2}$ (see \Cref{lem:UB for vtr(B+G)u}), we obtain the following bound for a step size $\lr \leq \lr_{\max}$ specified in \Cref{thm:expectation_bound} statement:
\begin{align*}
    \E \psq{y\tr  \p{\I-\lr \M_{t}}\tr 
                                \p{\I-\lr \M_{t}} y \; \middle| \; \F_{t}
                        }  
    \le \p{1-\lr \mu} \norm{y}_{2}^{2}. \label{BoundingBiasUsefulResultp} \numberthis
\end{align*} 
Using the bound above, the bias term in \eqref{eq:biasvardecomposition-mainpaper} is handled as follows:
\begin{align*}
    z_{t}^{bias} \le  \exp\p{-\lr \mu t} \E \psq{\norm{z_{0}}^{2}}.
\end{align*}
Using $\norm{h_{k}(\bar{w})}^{2} \le \sigma^2$, we bound the variance term as follows:
\begin{align*}
\E \psq{\norm{\sum_{k=0}^{t} \Cij^{t:k+1} h_{k}(\bar{w})}_{2}^{2}}  
  &\le \sigma^{2} \sum_{k=0}^{t} \E \psq{ \E \psq{\norm{(\I-\lr \M_{t})}^{2} \; \middle| \;\F_{t}} \norm{\Cij^{t-1:k+1}}_{2}^{2}} \\
& \le \sigma^{2} \sum_{k=0}^{t} \p{1-\lr \mu} \E \psq{ \norm{\Cij^{t-1:k+1}}_{2}^{2}} \\
& \le\sigma^{2} \sum_{k=0}^{t} \p{1-\lr \mu}^{t-k} \le \frac{\sigma^{2}} {\lr \mu}. \numberthis \label{eq:variancebdp}
\end{align*}
The main claim follows from combining the bounds on the bias and variance terms, followed by straightforward simplifications. The reader is referred to \Cref{sec:Appendix-expectationbd} for the full proof.

\section{Proof of Theorem \ref{thm:expectation_bound}} \label{sec:Appendix-expectationbd}

\begin{proof}
    \ \\
\noindent\textbf{Step 1: Bias-variance decomposition}

Recall the updates in \Cref{alg:TD-Critic-main} can be rewritten as follows: 
\begin{align} \label{eq:wt-update}
    w_{t+1} = w_{t} + \lr (r_{t} \phi_{t}-\M_{t} w_{t}).
\end{align}
Defining the centered error as 
$z_{t+1} = w_{t+1}-\bar{w}$, we obtain
\begin{align*}
      z_{t+1}      &\phantom{\hskip 6pt}= w_{t}-\bar{w} + \lr (r_{t}\phi_{t}-\M_{t}w_{t})+\lr \M_{t} \bar{w} - \lr \M_{t} \bar{w} \\
            &\phantom{\hskip 6pt}= (\I-\lr \M_{t})(w_{t}-\bar{w})+\lr(r_{t}\phi_{t}-\M_{t}\bar{w}) \\
            &\phantom{\hskip 6pt}= (\I-\lr \M_{t})z_{t}+\lr(r_{t}\phi_{t}-\M_{t}\bar{w}).
\end{align*}
Letting $ h_{t}(w_{t}) = r_{t} \phi_{t} - \M_{t} w_{t} $, we have
\begin{align*}
    z_{t+1} &= (\I-\lr \M_{t})z_{t}+\lr h_{t}(\bar{w}).
\end{align*}
Unrolling the equation above, we obtain
\begin{align*}
        z_{t+1} &= (\I-\lr \M_{t})((\I-\lr \M_{t-1})z_{t-1}+\lr                              h_{t-1}(\bar{w}))+\lr h_{t}(\bar{w}) \\
                &= (\I-\lr \M_{t})(\I-\lr \M_{t-1}) \dots (\I-\lr \M_{0}) z_{0} \; + \lr h_{t}(\bar{w}) \\ 
                & \quad + \lr (\I-\lr \M_{t})  h_{t-1}(\bar{w}) \\ 
                & \quad + \lr  (\I-\lr \M_{t})(\I-\lr \M_{t-1}) h_{t-2}(\bar{w}) \\ 
                & \quad \vdots  \\
                & \quad+\lr (\I-\lr \M_{t})(\I-\lr \M_{t-1}) \dots (\I-\lr \M_{1}) h_{0}(\bar{w}). 
\end{align*}
Define 
\begin{align*}
    \mathbf{C}^{i:j} & = \begin{cases}
                (\I-\lr \M_{i})(\I-\lr \M_{i-1}) \dots (\I-\lr \M_{j}) & \text{if } i \geq j \\
                \; \I & \text{otherwise.} 
                \end{cases} 
\end{align*}
Using the  definition above, we obtain
\begin{align*}
   \|z_{t+1}\|^{2} & = \left \| \mathbf{C}^{t:0}z_{0}+\lr \sum_{k=0}^{t} \mathbf{C}^{t:k+1} h_{k}(\bar{w}) \right\|^{2} .
\end{align*}
Taking expectations and using $\|a+b\|^2  \leq 2 \|a\|^{2}+2\|b\|^{2}$, we obtain
\begin{align*}
    \E[\left\|z_{t+1}\right\|^{2}] 
    & \leq 2 z_{t}^{\text{bias}} + 2 \lr^{2} z_{t}^{\text{variance}}, \numberthis \label{eq:biasvardecomposition}
\end{align*}    
where $z_{t}^{\text{bias}} = \E\left[\left\|\mathbf{C}^{t:0}z_{0}\right\|^{2}\right]$ and
$z_{t}^{\text{variance}} = \E \psq{\norm{\sum_{k=0}^{t} \mathbf{C}^{t:k+1} h_{k}(\bar{w})}^{2}}$. 

\noindent\textbf{\newline Step 2: Bounding the bias term} 

Next, we state and prove a useful lemma that will assist in bounding the bias term in \eqref{eq:biasvardecomposition}.

\begin{lemma} \label{lem:I-gammaMt-support}
Consider a random vector $y \in \mathbb{R}^{2q}$ and let $\F_{t}$ be sigma-algebra generated by $\{w_{0} \dots w_{t}\}$, For $\lr \leq \lr_{\mx}$, we have
\begin{align} \label{appendix:supportingbias}
& \E \psq{y\tr  \p{\I-\lr \M_{t}}\tr \p{\I-\lr \M_{t}} y \; \middle| \; \F_{t}} \leq  \p{1-\lr \mu} \norm{y}_{2}^{2}, \\
& \E \psq{\norm{\p{\I-\lr \M_{t}} y} \; \middle| \; \F_{t}} 
\leq \p{1-\frac{\lr \mu}{2}} \norm{y}_{2},
\label{lem: norm-version}
\end{align} where $\lr \leq \lr_{\mx} = \frac{\mu}{k}, \mu = \lambda_{\mathsf{min}\p{\frac{\M\tr +\M}{2}}}$ is the minimum eigenvalue of the matrix $\frac{\M\tr +\M}{2}$ and 
$ k =\max \big \{ 4 (\fivmax)^{4}  \!+\!4 \disc^{2} R_{\max}^{2} {(\fiumax)}^{2}(\fivmax)^{2}  \bm{,} 
4(\fiumax)^{4}  \big \} \!+\!  2 \disc R_{\max} ((\fivmax)^{2}(\fiumax)^{2}+(\fiumax)^{4}).
$
\end{lemma}
\begin{proof}
To prove the desired result, we split \eqref{appendix:supportingbias} as follows:
\begin{align}
    &\E \psq{y\tr  \p{\I-\lr \M_{t}}\tr 
                                \p{\I-\lr \M_{t}} y \; \middle| \; \F_{t}
                        } \nonumber 
    = \E \psq{y\tr  \p{\I-\lr \p{\M_{t}\tr +\M_{t}}+                                      \lr^{2}\M_{t}\tr \M_{t}} y  \; \middle| \; \F_{t} }  \nonumber 
    \\  &=  \norm{y}_{2}^{2} -\lr \underbrace{y\tr \E \psq{\p{\M_{t}\tr +\M_{t}} \; \middle| \ \F_{t}}y}_{\colorcircle{blue!20}{T1}}+\lr^{2} \underbrace{y\tr \E \psq{\M_{t}\tr \M_{t} \; \middle| \;\F_{t}}y}_{\colorcircle{blue!20}{T2}}. \label{eq:ytrCy}
\end{align}
We lower bound the term T1 as follows:
\begin{align*}
    y\tr  \E \psq{\p{\M_{t}\tr +\M_{t}}\;\middle|\; \F_{t}} y &= y\tr  \p{\M\tr +\M} y  \geq 2 \mu \norm{y}_{2}^{2}. \numberthis \label{eq:T1bd}
\end{align*}
Next, we upper bound the term T2 as follows:
\begin{align*}
    \M_{t}\tr \M_{t}&={\begin{pmatrix}
                        \A_{t} & \zeromat \\
                        \C_{t} & \B_{t}
                        \end{pmatrix}}\tr 
                        \begin{pmatrix}
                                \A_{t} & \zeromat \\
                                \C_{t} & \B_{t}
                        \end{pmatrix} 
                     = \begin{pmatrix}
                                \A_{t}\tr \A_{t}+\C_{t}\tr \C_{t} & \C_{t}\tr \B_{t} \\
                                \B_{t}\tr \C_{t} & \B_{t}\tr \B_{t}
                        \end{pmatrix},
\end{align*}
Substituting the above into T2, we obtain:
\begin{align}
    y\tr  \E \psq{\M_{t}\tr \M_{t}\;\middle|\; \F_{t}} y 
    & = y\tr  \E \psq{
                        \begin{pmatrix}
                        \A_{t}\tr \A_{t}+\C_{t}\tr \C_{t} & \C_{t}\tr \B_{t} \\
                        \B_{t}\tr \C_{t} & \B_{t}\tr \B_{t}
                        \end{pmatrix}
                        \;\middle|\; \F_{t}} y \nonumber \\
    &=  \begin{pmatrix}
        v\tr  & u\tr 
        \end{pmatrix}
         \E \psq{
                        \begin{pmatrix}
                        \A_{t}\tr \A_{t}+\C_{t}\tr \C_{t} & \C_{t}\tr \B_{t} \\
                        \B_{t}\tr \C_{t} & \B_{t}\tr \B_{t}
                        \end{pmatrix}
                        \;\middle|\; \F_{t}}
        \begin{pmatrix}
        v \\ 
        u
        \end{pmatrix} \nonumber \\ 
    &=  \underbrace{v\tr  \E \psq{\A_{t}\tr \A_{t}+\C_{t}\tr \C_{t}\;\middle|\; \F_{t}}v}_{\colorcircle{gray!20}{S1}}
        + \underbrace{u\tr  \E \psq{\B_{t}\tr \B_{t}\;\middle|\; \F_{t}}u}_{\colorcircle{gray!20}{S2}}\nonumber\\
        &\qquad
        + \underbrace{v\tr  \E \psq{\C_{t}\tr \B_{t}\;\middle|\; \F_{t}} u}_{\colorcircle{gray!20}{S3}} 
        + \underbrace{u\tr  \E \psq{\B_{t}\tr \C_{t}\;\middle|\; \F_{t}} v}_{\colorcircle{gray!20}{S4}}. \label{appendix:T2}
\end{align}
To upper bound T2, we first derive upper bounds for S1, S2, S3, and S4.

First, we examine S1.  
\begin{align} \label{eq:S1ab-split}
    v\tr  \E \psq{\A_{t}\tr \A_{t}+\C_{t}\tr \C_{t}\;\middle|\; \F_{t}}v =& \underbrace{v\tr  \E \psq{\A_{t}\tr \A_{t}\;\middle|\; \F_{t}} v}_{\colorcircle{green!20}{(a)}} + \underbrace{v\tr  \E \psq{\C_{t}\tr \C_{t}\;\middle|\; \F_{t}}v}_{\colorcircle{green!20}{(b)}}.
\end{align}
We bound (a) in \eqref{eq:S1ab-split} as follows: 
\begin{align*}
&v\tr  \E \psq{\A_{t}\tr  \A_{t} \;\middle|\; \F_{t}} v \\
&= v\tr  \E [\p{\fiv (s_{t}) \fiv (s_{t})\tr  \!-\!\disc \fiv(s_{t})\fiv(s_{t+1})\tr }\tr (\fiv (s_{t}) \fiv (s_{t})\tr \!-\! \disc \fiv(s_{t})\fiv(s_{t+1})\tr )  \;|\;\F_{t}\;] v \\
& = v\tr  \E \big[\fiv (s_{t}) \fiv (s_{t})\tr  \fiv (s_{t}) \fiv (s_{t})\tr  \!-\! \disc \fiv (s_{t}) \fiv (s_{t})\tr  \fiv(s_{t})\fiv(s_{t+1})\tr  \\ & \phantom{\hskip 35pt} - \disc \fiv(s_{t+1})\fiv(s_{t})\tr  \fiv (s_{t}) \fiv (s_{t})\tr  \\ & \phantom{\hskip 35pt} + \disc^{2} \fiv(s_{t+1})\fiv(s_{t})\tr  \fiv(s_{t})\fiv(s_{t+1})\tr  \mid \F_{t} \big] v \\
&\stackrel{(i)}{=} v\tr  \E \big[\norm{\fiv(s_{t})}_{2}^{2} \big(\fiv(s_{t})\fiv(s_{t})\tr -\disc \underbrace{\p{\fiv(s_{t}) \fiv(s_{t+1})\tr +\fiv(s_{t+1}) \fiv(s_{t})\tr }}_{(I)} \\
& \quad+\disc^{2} \fiv(s_{t+1})\fiv(s_{t+1})\tr \big)\mid \F_{t}\big] v \\
&\stackrel{(ii)}{\leq} (\fivmax)^{2} v\tr  \E \big[\fiv(s_{t})\fiv(s_{t})\tr  +\disc \p{\fiv(s_{t})\fiv(s_{t})\tr +\fiv(s_{t+1})\fiv(s_{t+1})\tr } \\
&\quad+\disc^{2} \fiv(s_{t+1})\fiv(s_{t+1})\tr \mid \F_{t}\big] v \\
& \le (\fivmax)^{2} \p{1+2\disc+\disc^{2}} v\tr  \matB v, \label{eq:S1a} \numberthis 
\end{align*}
where $\matB  = \E \psq{\fiv(s_{t})\fiv(s_{t})\tr \mid \F_{t}}$.
In the above,
the inequality in (i) follows from the identity $\norm{\fiv(s_{t})}_{2}^{2} = \fiv(s_{t})\tr \fiv(s_{t})$; (ii) follows from applying the bound on the features from \Cref{asm:bddFeatures} and using the following inequality for term (I) in (i):
\begin{align}
    - v\tr  \p{\frac{aa\tr +bb\tr }{2}}v \leq v\tr  \p{ ab\tr } v \leq v\tr  \p{\frac{aa\tr +bb\tr }{2}} v. \label{eq:supporting-ineq}
\end{align} 
The final inequality in \eqref{eq:S1a} follows from using the following equivalent forms of $\matB$:
\begin{align*}
    \matB & = \E \psq{\fiv(s_{t})\fiv(s_{t})\tr \mid \F_{t}} 
    = \E \psq{\fiv(s_{t+1})\fiv(s_{t+1})\tr  \;\middle|\;\F_{t}}  = \E^{\chi,\TPM}\psq{\fiv(s_{t})\fiv(s_{t})\tr } \\ & = \E^{\chi,\TPM} \psq{\fiv(s_{t+1})\fiv(s_{t+1})\tr }. \numberthis \label{eq:defineB}
\end{align*}
The equivalences above hold due to the i.i.d. observation model (\Cref{asm:iidNoise}).

Next, We bound (b) in \eqref{eq:S1ab-split} as follows: 
\begin{align*}
    v\tr  \E \psq{\C_{t}{\tr }\C_{t} \;\middle|\; \F_{t}} v &= v \tr  \E \psq{\p{-2\disc r_{t} \fiu(s_{t}) \fiv(s_{t+1})\tr }\tr  \p{-2\disc r_{t} \fiu(s_{t}) \fiv(s_{t+1})\tr }\mid \F_{t}} v  \\
    &= 4 \disc^{2}  v \tr  \E \psq {r_{t}^{2} \fiv(s_{t+1}) \fiu(s_{t})\tr  \fiu(s_{t}) \fiv(s_{t+1})\tr \;\middle|\; \F_{t}} v  \\
    &\stackrel{(i)}{=} 4 \disc^{2} v \tr  \E \psq {r_{t}^{2} \norm{\fiu(s_{t})}_{2}^{2} \fiv(s_{t+1}) \fiv(s_{t+1})\tr \;\middle|\; \F_{t}} v \\
    & \stackrel{(ii)}{\leq} 4 \disc^{2} R_{\mx}^{2} (\fiumax)^{2} v\tr  \matB v,  \label{eq:S1b} \numberthis 
\end{align*}
where (i) follows from $\norm{\fiu(s_{t})}_{2}^{2} = \fiu(s_{t})\tr \fiu(s_{t})$ and (ii) follows from bound on rewards (\Cref{asm:bddRewards}) and the definition of $\matB$ in \eqref{eq:defineB}.

Combining \eqref{eq:S1a} and \eqref{eq:S1b}, we obtain the following upper bound for S1:
\begin{align}
v\tr  \E \psq{\A_{t}\tr \A_{t}+\C_{t}\tr \C_{t}\;\middle|\; \F_{t}} v & \leq \p{(\fivmax)^{2} \p{1+\disc}^{2}+ 4 \disc^{2} R_{\mx}^{2} {(\fiumax)}^{2}} v\tr  \matB v. \label{eq:S1}
\end{align} 

Next, we derive an upper bound for S2 in \eqref{appendix:T2} as follows:
\begin{align*}
& u\tr  \E \psq{\B_{t}\tr \B_{t}\;\middle|\; \F_{t}}u \\ & = u\tr  \E [\p{\fiu (s_{t}) \fiu (s_{t})\tr  - \disc^{2} \fiu(s_{t}) \fiu(s_{t+1})\tr }\tr (\fiu (s_{t}) \fiu (s_{t})\tr - \disc^{2} \fiu(s_{t}) \fiu(s_{t+1})\tr ) \mid  \F_{t}]u \\ 
&=u\tr  \E \big [ \fiu (s_{t}) \fiu (s_{t})\tr  \fiu (s_{t}) \fiu (s_{t})\tr -\disc^{2} \big(\fiu (s_{t}) \fiu (s_{t})\tr  \fiu (s_{t}) \fiu (s_{t+1})\tr  \\ 
& \phantom{\hskip 35pt} +\fiu (s_{t+1}) \fiu (s_{t})\tr  \fiu (s_{t}) \fiu (s_{t})\tr \big) \\ 
& \phantom{\hskip 35pt} +\disc^{4}(\fiu (s_{t+1}) \fiu (s_{t})\tr  \fiu (s_{t}) \fiu (s_{t+1})\tr )  \;|\; \F_{t} \big]u \\ 
&\stackrel{(i)}{=} u\tr  \E \big[\norm{\fiu(s_{t})}_{2}^{2} \big(\fiu(s_{t})\fiu(s_{t})\tr -\disc^{2} \underbrace{\p{\fiu(s_{t}) \fiu(s_{t+1})\tr +\fiu(s_{t+1}) \fiu(s_{t})\tr }}_{(II)} \\ 
& \phantom{\hskip 35pt} +\disc^{4} \fiu(s_{t+1})\fiu(s_{t+1})\tr \big) \mid \F_{t} \big] u \\
&\stackrel{(ii)}{\leq}  \p{{\fiumax}}^{2} u\tr  \E \big[\fiu(s_{t})\fiu(s_{t})\tr +\disc^{2} \p{\fiu(s_{t})\fiu(s_{t})\tr +\fiu(s_{t+1})\fiu(s_{t+1})\tr } \\ 
& \phantom{\hskip 70pt} +\disc^{4} \fiu(s_{t+1})\fiu(s_{t+1})\tr \; | \; \F_{t} \big] u \\
& \leq (\fiumax)^{2} \p{1+2\disc^{2}+\disc^{4}} u\tr  \matG u, \label{eq:S2} \numberthis 
\end{align*}
where $\matG = \E \psq{\fiu(s_{t})\fiu(s_{t})\tr \;\middle|\; \F_{t}}$. In the above, the inequality in (i) follows from $\norm{\fiu(s_{t})}_{2}^{2} = \fiu(s_{t})\tr \fiu(s_{t})$; (ii) follows from bound on features (\Cref{asm:bddFeatures}) and applying the inequality \eqref{eq:supporting-ineq} to (II); and \eqref{eq:S2} follows from bound on features (\Cref{asm:iidNoise}). 

The inequality in \eqref{eq:S2} follows from following equivalent forms of $\matG$:
\begin{align*} 
\matG &= \E \psq{\fiu(s_{t})\fiu(s_{t})\tr \;\middle|\; \F_{t}} = \E \psq{\fiu(s_{t+1})\fiu(s_{t+1})\tr \;\middle|\; \F_{t}} = \E^{\chi,\TPM}\psq{\fiu(s_{t})\fiu(s_{t})\tr } \\ 
& = \E^{\chi,\TPM} \psq{\fiu(s_{t+1})\fiu(s_{t+1})\tr }.  \numberthis \label{eq:defineG} 
\end{align*}
The equivalences above hold from the i.i.d observation model (\Cref{asm:iidNoise}).
   
We observe that the scalars S3 and S4 in \eqref{appendix:T2} are equal, i.e.,
\begin{align*}
v\tr  \E \psq{\C_{t}\tr  \B_{t}\;\middle|\; \F_{t}} u = u\tr  \E \psq{\B_{t}\tr  \C_{t}\;\middle|\; \F_{t}} v.
\end{align*}

We establish an upper bound for S3 in \eqref{appendix:T2} as follows:
\begin{align*}
    & v\tr  \E \psq{\C_{t}\tr  \B_{t}} u \\ & = v\tr  \E \big[-2 \disc r_{t} \fiv(s_{t+1}) \fiu(s_{t})\tr  \fiu(s_{t}) \fiu(s_{t})\tr \\ & \phantom{\hskip 35pt} + 2 \disc^{3} r_{t} \fiv(s_{t+1}) \fiu(s_{t})\tr  \fiu(s_{t}) \fiu(s_{t+1})\tr \;|\; \F_{t} \big] u \\
    & \stackrel{(i)}{=} \norm{\fiu(s_{t})}^{2}_{2} v\tr  \E [-2 r_{t} \disc  \underbrace{\fiv(s_{t+1}) \fiu(s_{t})\tr}_{(III)}  + 2 r_{t} \disc^{3} \underbrace{\fiv(s_{t+1}) \fiu(s_{t+1})\tr}_{(IV)} \;|\; \F_{t}] u \\
    & \stackrel{(ii)}{\leq}  (\fiumax)^{2} R_{\mx} v\tr  \E \bigg[ \disc \p{ \fiv(s_{t+1}) \fiv(s_{t+1})\tr +\fiu(s_{t}) \fiu(s_{t})\tr } 
    \\ & \phantom{\hskip 90pt}+ \disc^{3}  (\fiv(s_{t+1}) \fiv(s_{t+1})\tr +\fiu(s_{t+1}) \fiu(s_{t+1})\tr )\;|\; \F_{t}\bigg] u \\
    & {\leq}  (\fiumax)^{2} R_{\mx} v\tr \p{\disc (\matB+\matG)+\disc^{3} (\matB+\matG) } u, \label{eq:S3} \numberthis 
\end{align*}
where (i) follows from $\norm{\fiu(s_{t})}_{2}^{2} = \fiu(s_{t})\tr \fiu(s_{t})$ ; (ii) follows from bounds on features and rewards (\Cref{asm:bddFeatures,asm:bddRewards}) and applying the inequality below to the coefficients of $\disc$ (III) with ($a=\fiv(s_{t+1})$, $b=\fiu(s_{t})$) and $\disc^3$ (IV) with  ($a=\fiv(s_{t+1})$, $b=\fiu(s_{t+1})$) respectively.
\begin{align*}
    - v\tr  \p{\frac{aa\tr +bb\tr }{2}}u \leq v\tr  \p{ ab\tr } u \leq v\tr  \p{\frac{aa\tr +bb\tr }{2}} u. 
\end{align*}
\eqref{eq:S3} follows from using values of matrices $\matB$ \eqref{eq:defineB} and $\matG$ \eqref{eq:defineG}.

Substituting \eqref{eq:S1}--\eqref{eq:S3} in \eqref{appendix:T2}, we determine the upper bound for T2 as follows:
\begin{align}
    y\tr  \E \psq{\M_{t}\tr \M_{t}\;\middle|\; \F_{t}} y
    & \leq \p{(\fivmax)^{2} \p{1+\disc}^{2} + 4 \disc^{2} R_{\mx}^{2} {(\fiumax)}^{2}} v\tr  \matB v \label{eq:T2Bd} \\ 
    & \phantom{\hskip 10pt} + (\fiumax)^{2} \p{1+\disc^{2}}^{2} u\tr  \matG u  \nonumber \\   
    & \phantom{\hskip 10pt} + 2 (\fiumax)^{2} R_{\mx} (\disc (1+\disc^2)) v\tr \p{\matB+\matG} u. \nonumber  
\end{align}
Next, we state and prove a useful result to simplify \eqref{eq:T2Bd} further.
\begin{lemma} For any $y = (v,u)\tr  \in \R^{2|\S|}$ and matrix $\matB+\matG$  defined in \eqref{eq:S3}, we have \label{lem:UB for vtr(B+G)u}
\begin{align*}
v \tr  (\matB+\matG) u & \leq \frac{\lambda_{\mathsf{max}}(\matB+\matG)}{2} \norm{y}_{2}^{2}. 
\end{align*}
\end{lemma}
\begin{proof}
We have
\begin{align*}
      v \tr  (\matB+\matG) u 
                    &\stackrel{(a)}{\leq} \norm{v}_{\matB+\matG} \norm{u}_{\matB+\matG} \\ 
                    &\stackrel{(b)}{\leq} \sqrt{v \tr  (\matB+\matG) v} \sqrt{u \tr  (\matB+\matG) u} \\ 
                    &\stackrel{(c)}{\leq} \lambda_{\mathsf{max}}(\matB+\matG) \sqrt{\norm{v}_{2}^2 \norm{u}_{2}^2} \\
                    & \stackrel{(d)}{\leq} \lambda_{\mathsf{max}}(\matB+\matG) \frac{\norm{v}_{2}^{2}+\norm{u}_{2}^{2}}{2} \\ 
                    & \stackrel{(e)}{\leq} \frac{\lambda_{\mathsf{max}}(\matB+\matG)}{2} \norm{y}_{2}^{2},
\end{align*}
where (a) follows from Cauchy-Schwarz inequality; (b) follows from definition of the weighted norm; (c) follows from Rayleigh quotient theorem for a symmetric real matrix $\mathbf{Q}$, i.e.,
$x \tr  \mathbf{Q} x \leq \lambda_{\mathsf{max}}(\mathbf{Q}) \norm{x}^{2}_{2}$; (d) follows from AM-GM inequality; and (e) follows from definition of $\norm{y}_{2}^2 = \norm{v}_{2}^2+\norm{u}_{2}^2 $. 
\end{proof}
Substituting the upper bounds obtained for T1 \eqref{eq:T1bd} and T2 \eqref{eq:T2Bd} in \eqref{eq:ytrCy}, we get
\begin{align*}
    & \E \psq{y\tr  \p{\I-\lr \M_{t}}\tr 
                                \p{\I-\lr \M_{t}} y \; \middle| \; \F_{t}
                        }  
    =  \norm{y}_{2}^{2} -\lr \underbrace{y\tr \E \psq{\p{\M_{t}\tr +\M_{t}} | \F_{t}}y}_{\colorcircle{blue!20}{T1}}\\
    & \phantom{\hskip 190pt} +\lr^{2} \underbrace{y\tr \E \psq{\M_{t}\tr \M_{t} \; \middle| \;\F_{t}}y}_{\colorcircle{blue!20}{T2}}  \\
    &\leq  \norm{y}^{2}_{2}- 2\lr\mu \norm{y}_{2}^{2}  +  \lr^{2}  \bigg ( \p{(\fivmax)^{2} \p{1+\disc}^{2} + 4 \disc^{2} R_{\mx}^{2} {(\fiumax)}^{2}} v\tr  \matB v \\ 
    & \phantom{\hskip 10pt} + (\fiumax)^{2} \p{1+\disc^{2}}^{2} u\tr  \matG u  + 2 (\fiumax)^{2} R_{\mx} (\disc (1+\disc^2)) v\tr \p{\matB+\matG} u \bigg ) \\
    &\stackrel{(i)}{\leq} \norm{y}_{2}^{2}- 2 \lr\mu \norm{y}_{2}^{2} +  \lr^{2}  \bigg ( \p{(\fivmax)^{2} \p{1+\disc}^{2} + 4 \disc^{2} R_{\mx}^{2} {(\fiumax)}^{2}} \lambda_{\mx}(\matB) \norm{v}^{2}_{2} \\ & \phantom{\hskip 10pt} + (\fiumax)^{2} \p{1+\disc^{2}}^{2} \lambda_{\mx}(\matG)\norm{u}^{2}_{2} + (\fiumax)^{2} R_{\mx} (\disc (1+\disc^2))\lambda_{\mathsf{max}}(\matB+\matG) \norm{y}_{2}^{2} \bigg ) \\ 
    &\leq \norm{y}_{2}^{2}- 2 \lr \mu \norm{y}_{2}^{2} + \lr^{2}  \bigg ( \mx \Big \{\p{(\fivmax)^{2} \p{1+\disc}^{2} + 4 \disc^{2} R_{\mx}^{2} {(\fiumax)}^{2}}\lambda_{\mx}(\matB) \bm{,} \\ 
    & \phantom{\hskip 10pt}  (\fiumax)^{2} \p{1+\disc^{2}}^{2} \lambda_{\mx}(\matG) \Big \} \norm{y}_{2}^{2}+(\fiumax)^{2} R_{\mx} (\disc (1+\disc^2))\lambda_{\mathsf{max}}(\matB+\matG) \norm{y}_{2}^{2} \bigg ) \\
    &\leq \norm{y}_{2}^{2}-\lr \bigg( 2 \mu - \lr \Big(\mx \Big \{ \p{(\fivmax)^{2} \p{1+\disc}^{2} + 4 \disc^{2} R_{\mx}^{2} {(\fiumax)}^{2}} \lambda_{\mx}(\matB)\bm{,} \\ 
    & \phantom{\hskip 10pt} (\fiumax)^{2} \p{1+\disc^{2}}^{2} \lambda_{\mx}(\matG)\Big \} + (\fiumax)^{2} R_{\mx} (\disc (1+\disc^2))\lambda_{\mathsf{max}}(\matB+\matG) \Big) \bigg) \norm{y}_{2}^{2} \\
    & \stackrel{(ii)}{\leq} \norm{y}_{2}^{2}-\lr \bigg( 2 \mu - \lr \Big(\max \left \{ 4 (\fivmax)^{4}  \!+\!4 \disc^{2} R_{\max}^{2} {(\fiumax)}^{2}(\fivmax)^{2}  \bm{,} 
    4(\fiumax)^{4}  \right \} \\ 
    & \phantom{\hskip 10pt}+  2 \disc R_{\max} \left((\fivmax)^{2}(\fiumax)^{2}+(\fiumax)^{4}\right) \Big) \bigg) \norm{y}_{2}^{2} \\
    & \leq \; (1-\lr \mu) \norm{y}_{2}^{2},\label{BoundingBiasUsefulResult} \numberthis
\end{align*} 
where (i) follows from \Cref{lem:UB for vtr(B+G)u} and the inequality \( x \tr  \mathbf{Q} x \leq \lambda_{\mathsf{max}}(\mathbf{Q}) \norm{x}^{2}_{2} \);  
(ii) follows from the bounds \( \lambda_{\mathsf{max}}(\mathbf{B}) \leq (\fivmax)^{2} \), \( \lambda_{\mathsf{max}}(\mathbf{G}) \leq (\fiumax)^{2} \), and \( \lambda_{\mathsf{max}}(\mathbf{B+G}) \leq (\fivmax)^{2} + (\fiumax)^{2} \), given that \( \mathbf{B} \) and \( \mathbf{G} \) are outer products of the vectors \( \phi_{v}(s_{t}) \) and \( \phi_{u}(s_{t}) \), respectively;  
\eqref{BoundingBiasUsefulResult} follows from choosing \( \lr \leq \lr_{\mx} \).

Rewriting \eqref{BoundingBiasUsefulResult} in norm form gives:
\begin{align*} 
    & \E \psq{y\tr  \p{\I-\lr \M_{t}}\tr 
                                \p{\I-\lr \M_{t}} y \; \middle| \; \F_{t}
                        }  = 
    \E \psq{\norm{(\I-\lr \M_{t}) y}^2 \; \middle| \; \F_{t}} \leq (1-\lr \mu) \norm{y}_{2}^2. \numberthis \label{eq:norm-ver}
\end{align*} 
Taking the square root on both sides of \eqref{eq:norm-ver} and applying Jensen's inequality yields the second claim.
\begin{align*} 
    & \E \psq{\norm{(\I-\lr \M_{t}) y} \; \middle| \; \F_{t}} \leq \sqrt{\E \psq{\norm{(\I-\lr \M_{t}) y}^2 \; \middle| \; \F_{t}}} \leq (1-\lr \mu)^{\frac{1}{2}} \norm{y}_{2} {\leq} \p{1-\frac{\lr \mu}{2}} \norm{y}_{2}, \label{eq:secondclaim} \numberthis 
\end{align*}
where \eqref{eq:secondclaim} follows from applying the inequality $(1-x)^{\frac{1}{2}} \leq 1-\frac{x}{2}$, for $x \geq 0$ with $x = \lr \mu$. 
\end{proof}
Now, we bound the bias term as follows:
\begin{samepage}
\begin{align*}
    z_{t}^{\text{bias}} & = \E \psq{\norm{\mathbf{C}^{t:0}z_{0}}^{2}} \\
    & = \E \psq {\E \psq{
    \p{\mathbf{C}^{t-1:0}z_{t-1}^{\text{bias}}}\tr  \p{\I-\lr \M_{t}}\tr 
    \p{\I-\lr \M_{t}}
     \p{\mathbf{C}^{t-1:0}z_{t-1}^{\text{bias}}}
      | \F_{t} } }\\
    & \stackrel{(i)}{\leq} \p{1-\lr \mu} \E \psq{\norm{\mathbf{C}^{t-1:0}z_{t-1}^{\text{bias}}}^{2}} \\ 
    & {\leq} \p{1-\lr \mu}^{t} \E \psq{\norm{z_{0}}^{2}} \numberthis \label{eq:bias-recursion}\\
    & {\leq} \exp\p{-\lr \mu t} \E \psq{\norm{z_{0}}^{2}}  \numberthis \label{eq:biasbd},
\end{align*}
where (i) follows from \Cref{lem:I-gammaMt-support}; \eqref{eq:bias-recursion} follows from unrolling the recursion and applying \Cref{lem:I-gammaMt-support} repeatedly; and \eqref{eq:biasbd} follows from the inequality below:  
\[
(1-\lr\mu)^{t} = \exp(t\log(1-\lr\mu)) \leq \exp(-\lr\mu t).
\]
\end{samepage}
\noindent\textbf{\newline Step 3: Bounding the variance term} 
For the variance bound, we require an upper bound for $\norm{h_{t}(\bar{w})}^{2}$, which we derive below.
\begin{align*}
& \norm{h_{t}(\bar{w})}^{2} = \norm{r_{t} \phi(s_{t}) - \M_{t} \bar{w}}^{2} \\
& \stackrel{(a)}{\leq} 2\norm{r_{t} \phi(s_{t})}^{2}+ 2\norm{\M_{t} \bar{w}}^{2}_{2} \\ 
& \stackrel{(b)}{\leq} 2  R_{\mathsf{max}}^{2} \p{(\fivmax)^{2} + R_{\mx}^{2} (\fiumax)^{2}}+ 2 \norm{\M_{t}}^{2} \norm{\bar{w}}_{2}^{2} \\
& \stackrel{(c)}{\leq} 2  R_{\mathsf{max}}^{2} \p{(\fivmax)^{2} + R_{\mx}^{2} (\fiumax)^{2}} + 2 ( (\fivmax)^{4} \p{1+\disc}^{2}+  (\fiumax)^{4} \p{1+\disc^{2}}^{2}\\ 
    & \phantom{\hskip 10pt} +4 \disc^{2} R_{\mx}^{2}(\fivmax)^{2}(\fiumax)^{2}) \norm{\bar{w}}_{2}^{2}  \\
& = \sigma^{2}, \numberthis \label{eq:upperboundforsigmasq}
\end{align*} where (a) follows using $\norm{a+b}^{2} \leq 2\norm{a}^{2}+2\norm{b}^{2} $; (b) follows from bounds on features and rewards (\Cref{asm:bddFeatures,asm:bddRewards}); and (c) follows from expanding the upper bound on  $\norm{\M_{t}}^{2}$.

Next, we bound the variance term in \eqref{eq:biasvardecomposition} as follows:
\begin{align*}
z_{t}^{\text{variance}} & =\E \psq{\norm{\sum_{k=0}^{t} \Cij^{t:k+1} h_{k}(\bar{w})}_{2}^{2}}\\
& \stackrel{(a)}{\leq} \sum_{k=0}^{t} \E \psq{ \norm{ \Cij^{t:k+1} h_{k}(\bar{w})}^{2}_{2}} \\
& \stackrel{(b)}{\leq} \sum_{k=0}^{t} \E \psq{ \norm{\Cij^{t:k+1}}^{2} \norm{h_{k}(\bar{w})}^{2}} \\ 
& \stackrel{(c)}{\leq} \sigma^{2} \sum_{k=0}^{t} \E \psq{\norm{\Cij^{t:k+1}}_{2}^{2}} \\ 
& \stackrel{(d)}{\leq} \sigma^{2} \sum_{k=0}^{t} \E \psq{ \E \psq{\norm{\Cij^{t:k+1}}_{2}^{2} \;\middle|\; \F_{t}}} \\
& \stackrel{(e)}{\leq} \sigma^{2} \sum_{k=0}^{t} \E \psq{ \E \psq{\norm{(\I-\lr \M_{t})\Cij^{t-1:k+1}}_{2}^{2} \;\middle|\; \F_{t}}} \\
& \stackrel{(f)}{\leq} \sigma^{2} \sum_{k=0}^{t} \E \psq{ \E \psq{\norm{(\I-\lr \M_{t})}^{2} \;\middle|\; \F_{t}} \norm{\Cij^{t-1:k+1}}_{2}^{2}} \\
& \stackrel{(g)}{\leq} \sigma^{2} \sum_{k=0}^{t} \p{1-\lr \mu} \E \psq{ \norm{\Cij^{t-1:k+1}}_{2}^{2}} \\
& \stackrel{(h)}{\leq} \sigma^{2} \sum_{k=0}^{t} \p{1-\lr \mu}^{t-k} \\
& \stackrel{(i)}{\leq} \frac{\sigma^{2}} {\lr \mu}, \numberthis \label{eq:variancebd}
\end{align*}
where (a) follows from triangle inequality and linearity of expectations; (b) follows from the inequality $\norm{\mathbf{A}x} \leq \norm{\mathbf{A}} \norm{x} $; (c) follows from a bound on $\norm{h_{k}(\bar{w})}^2$ in \eqref{eq:upperboundforsigmasq}; (d) follows from the tower property of conditional expectations; (e) follows from unrolling the product of matrices $\Cij^{t:k+1}$ by one step; (f) follows from the inequality $\norm{\mathbf{A}\mathbf{B}} \leq \norm{\mathbf{A}} \norm{\mathbf{B}}$; (g) follows from \Cref{lem:I-gammaMt-support}; (h) follows from unrolling the product of matrices; and (i) follows from computing the upper bound for the finite geometric series. 

\noindent\textbf{Step 4: Clinching argument} 

The main claim follows from combining the bounds on the bias term \eqref{eq:biasbd} and the variance term \eqref{eq:variancebd} in \eqref{eq:biasvardecomposition} as follows:
\begin{align*}
    \E[\left\|z_{t+1}\right\|^{2}] 
    & \leq 2 z_{t}^{\text{bias}} + 2 \lr^{2} z_{t}^{\text{variance}} \\ 
    &  \leq  2 \exp\p{-\lr \mu t} \E \psq{\norm{z_{0}}^{2}} + \frac{2\lr \sigma^2}{\mu}.
\end{align*}
\end{proof}
\section{Proof of Theorem \ref{thm:expectation_bound-tailavg}}
\label{sec:Appendix-tail-avg-expectationbd}
\begin{proof} \noindent\textbf{\newline Step 1: Bias-variance decomposition for tail averaging} 

The tail averaged error when starting at $k+1$, at time t is given by
\begin{align*}
    z_{k+1:t} &= \frac{1}{N}\sum_{i = k+1}^{k+N}z_{i}= \frac{1}{t-k}\sum_{i = k+1}^{t}z_{i}.
\end{align*}

By taking expectations, $\norm{z_{k+1:t}}^2$ can be expressed as:
\begin{align}
\E\left[\norm{z_{k+1:t}}_{2}^{2}\right] &= \frac{1}{N^{2}}\sum_{i,j = k+1}^{k+N}\E\left[z_i\tr  z_j\right]\nonumber\\
&\stackrel{(a)}{\leq} \frac{1}{N^{2}}\bigg(\sum_{i= k+1}^{k+N}\E\left[\norm{z_i}_{2}^{2}\right] + 2 \sum_{i=k+1}^{k+N-1}\sum_{j=i+1}^{k+N} \E\left[z_{i}\tr  z_{j}\right] \bigg)\numberthis\label{eq:cross-term-decomp},
\end{align}
where $(a)$ follows from isolating the diagonal and off-diagonal terms.

Next, we state and prove a result that bounds the second term in \eqref{eq:cross-term-decomp}.
\begin{lemma} \label{lem:crosstermbound}
For all $i\ge 1$, we have
\begin{align}
    \sum_{i=k+1}^{k+N-1}\sum_{j=i+1}^{k+N} \E\psq{z_{i}\tr  z_{j}} 
    &\leq \frac{2}{\lr \mu}\sum_{i=k+1}^{k+N}\E\psq{\norm{z_{i}}^{2}_{2}}.
\end{align}
\end{lemma}
\begin{proof}
    \begin{align*}
        \sum_{i=k+1}^{k+N-1}\sum_{j=i+1}^{k+N} \E \left[z_{i}\tr  z_{j}\right]  &\stackrel{(a)}{=}  
        \sum_{i=k+1}^{k+N-1}\sum_{j=i+1}^{k+N} \E \left[z_{i}\tr  (\Cij^{j:i+1}z_{i} + \lr\sum_{l=i+1}^{j-i-1} \Cij^{j:l+1}h_{l}(\bar{w}))\right]\\
        &\stackrel{(b)}{=}\sum_{i=k+1}^{k+N-1}\sum_{j=i+1}^{k+N} \E  \left[z_{i}\tr  \Cij^{j:i+1}z_{i}\right] \\
        &\stackrel{(c)}{\leq}\sum_{i=k+1}^{k+N-1}\sum_{j=i+1}^{k+N} \E  \left[ \norm{z_{i}} \E \left[ \lVert \Cij^{j:i+1}  z_{i} \rVert \;\middle|\;\F_{j}\right]\right] \\
        &\stackrel{(d)}{\leq}\sum_{i=k+1}^{k+N-1}\sum_{j=i+1}^{k+N} \left(1 - \frac{\lr \mu}{2}\right)^{j-i}\E\left[\norm{z_{i}}^{2}_{2}\right]\\
        &\leq \sum_{i=k+1}^{k+N}\E\left[\norm{z_{i}}^{2}_{2}\right]\sum_{j=i+1}^{\infty}\left(1- \frac{\lr \mu}{2}\right)^{j-i}\\
        &\stackrel{(e)}{\leq}\frac{2}{\lr \mu}\sum_{i=k+1}^{k+N}\E\left[\norm{z_{i}}^{2}_{2}\right],
    \end{align*}
where (a) follows from expanding \( z_{j} \) using \eqref{eq:biasvardecomposition}; (b) follows from the observation that \[\E[h_t(\bar{w}) \mid \F_{t}] = \E [r_{t}\phi_{t}-\M_{t} \bar{w} \mid \F_{t}] = \xi-\M\bar{w}= 0;\] (c) follows from applying Cauchy-Schwarz inequality 
and tower property of expectations; (d) follows from a repetitive application of \Cref{lem:I-gammaMt-support}; and (e) follows from computing the limit of the infinite geometric series.
\end{proof} 
Substituting the result of \Cref{lem:crosstermbound} in \eqref{eq:cross-term-decomp}, we obtain
\begin{align*}
\E\left[\norm{z_{k+1:t}}_{2}^{2}\right] &\leq \frac{1}{N^{2}}\left(\sum_{i = k+1}^{k+N}\E\left[\norm{z_i}_{2}^{2}\right] + \frac{4}{\lr \mu} \sum_{i = k+1}^{k+N}\E\left[\norm{z_i}_{2}^{2}\right] \right)\\
&= \frac{1}{N^{2}}\left(1+\frac{4}{\lr \mu}\right) \sum_{i = k+1}^{k+N}\E\left[\norm{z_i}_{2}^{2}\right]\\
&\stackrel{(a)}{\leq} \underbrace{\frac{2}{N^{2}}\left(1+\frac{4}{\lr \mu}\right) \sum_{i = k+1}^{k+N}   z_{i}^{\mathsf{bias}}}_{z_{k+1, N}^{\mathsf{bias}}} + \underbrace{\frac{2}{N^{2}}\bigg(1+\frac{4}{\lr \mu}\bigg)\lr^{2}  \sum_{i = k+1}^{k+N}   z_{i}^{\mathsf{variance}}}_{z_{k+1:t}^{\mathsf{variance}}}\numberthis\label{eq:tail-av-bias-var-2},
\end{align*}
where $(a)$ follows from the bias-variance decomposition of $\E[\norm{z_{i}}_{2}^2]$ in \eqref{eq:biasvardecomposition}.

\noindent\textbf{\newline Step 2: Bounding the bias} 

First term, $z_{k+1:t}^{\mathsf{bias}}$ in \eqref{eq:tail-av-bias-var-2}  is bounded as follows:
\begin{align*}
z_{k+1:t}^{\mathsf{bias}} &\leq \frac{2}{N^{2}}\left(1 + \frac{4}{\lr \mu}\right) \sum_{i = k+1}^{\infty}   z_{i}^{\mathsf{bias}} \\
&\stackrel{(a)}{\leq} \frac{2}{N^{2}}\left(1 + \frac{4}{\lr \mu}\right) \sum_{i = k+1}^{\infty} (1 - \lr \mu)^{i}\E\left[\norm{z_{0}}^{2}_{2}\right] \\
&\stackrel{(b)}{=}\frac{2\E\left[\norm{z_{0}}^{2}_{2}\right]}{\lr \mu N^{2}}\left(1-\lr\mu\right)^{k+1}\left(1 + \frac{4}{\lr \mu}\right),
\end{align*}
where $(a)$ follows from \eqref{eq:bias-recursion}, which provides a bound on $z_{i}^{\mathsf{bias}}$;  $(b)$ follows from the bound on the summation of a geometric series. 

\noindent\textbf{Step 4: Bounding the variance} 

Next, the second term $z_{k+1:t}^{\mathsf{variance}}$ in \eqref{eq:tail-av-bias-var-2} is bounded as follows:
\begin{align*}
z_{k+1:t}^{\mathsf{variance}} &\stackrel{(a)}{\leq}\frac{2\lr^2}{N^{2}}\left(1+\frac{4}{\lr \mu}\right)\sum_{i=k+1}^{k+N}\frac{\sigma^2}{\lr\mu}\\
&\leq \frac{2\lr^2}{N^{2}}\left(1+\frac{4}{\lr \mu}\right)\sum_{i=0}^{N}\frac{\sigma^2}{\lr\mu}\\
&= \bigg(1+\frac{4}{\lr \mu}\bigg)\frac{2\lr\sigma^2}{\mu N},
\end{align*}
where  $(a)$ follows from \eqref{eq:variancebd}, which provides a bound on $z_{i}^{\mathsf{variance}}$.

\noindent\textbf{Step 5: Clinching argument}

Finally substituting the bounds on $z_{k+1:t}^{\mathsf{bias}}$ and $z_{k+1:t}^{\mathsf{variance}}$ in \eqref{eq:tail-av-bias-var-2}, we get
\begin{align*}
\E[\norm{z_{k+1:t}}^{2}_{2}] &\leq \bigg(1+ \frac{4}{\lr\mu}\bigg)\bigg(\frac{2}{\lr \mu N^{2}}(1-\lr\mu)^{k+1}\E[\norm{z_{0}}^{2}_{2}] + \frac{2\lr\sigma^2}{\mu N}\bigg),\\
&\stackrel{(a)}{\leq}\bigg(1+ \frac{4}{\lr\mu}\bigg)\bigg(\frac{2 \exp(-k \lr \mu)}{\lr \mu N^{2}}\E[\norm{z_{0}}^{2}_{2}] + \frac{2\lr\sigma^2}{\mu N}\bigg)\\
&\stackrel{(b)}{\leq}\frac{10 \exp(-k \lr \mu)}{\lr^2 \mu^{2} N^{2}}\E\left[\norm{z_{0}}^{2}_{2}\right] + 
\frac{10\sigma^2}{\mu^{2} N},
\end{align*}
where $(a)$ follows from $(1+x)^{y} = \exp(y\log(1+x)) \leq \exp(xy)$;  
$(b)$ uses the fact that $\lr\mu < 1$, since $\lr \leq \lr_\mathsf{max}$ as defined in \Cref{thm:expectation_bound}, which implies  
$
1+ \frac{4}{\lr\mu} \leq \frac{5}{\lr\mu}.
$
\end{proof} 


\section{Proof of Theorem \ref{thm:regTDlamdasq}}
\label{appendix:reg-td-expbound}
To prove \Cref{thm:regTDlamdasq}, we first establish an upper bound on the mean squared error (MSE) between the tail-averaged TD iterate and the regularized TD fixed point. The following result provides this bound, which we subsequently use to complete the proof of \Cref{thm:regTDlamdasq}.

\begin{theorem} \label{thm:regtd-expectation_bound} 
          Suppose  \Crefrange{asm:stationary}{asm:iidNoise} hold. Let $\check w_{k+1:t} = \frac{1}{N} \sum_{i=k+1}^{k+N} \check w_{i}$ denote the tail-averaged regularized iterate with $N=t-k$. Suppose the step size $\lreg$ satisfies
        \begin{align*}
         \lreg \leq& \check \lr_{\mathsf{max}} =  \frac{\zeta}{\check{c}}, \textrm{ where } \\
         \check c =&\, \zeta^2+ 2 \zeta \big((\fivmax)^{4} (1+\disc)^{2} + (\fiumax)^{4} (1+\disc^{2})^{2} +4 \disc^{2} R_{\max}^{2}(\fivmax)^{2}(\fiumax)^{2}\big)^{\frac{1}{2}} \\ &+ \max \big \{ 4 (\fivmax)^{4}  \!+\!4 \disc^{2} R_{\max}^{2} {(\fiumax)}^{2}(\fivmax)^{2}  \bm{,} 
         4(\fiumax)^{4}  \big \} \\ & +  2 \disc R_{\max} ((\fivmax)^{2}(\fiumax)^{2}+(\fiumax)^{4}).
    \end{align*}
            Then, 
        \begin{align*}
              & \E \psq{\norm{\check w_{k+1:t} - \bar{w}_{\mathsf{reg}}}^{2}_{2}} \leq \frac{  10\exp\p{-k \lreg (2\mu+\zeta)}}{\lreg^2  \p{2\mu+\zeta}^{2} N^{2}}\E\psq{\norm{\check w_{0} - \bar{w}_{\mathsf{reg}}}^{2}_2} + \frac{10\check\sigma^{2}}{(2\mu+\zeta)^{2} N},\numberthis \label{eq:regtd-expec-bd} 
        \end{align*}
           where $N=t - k$, $\mu=\eigmin $,  and
    \begin{align*}
     \check \sigma^2\!&=\! 2  R_{\mathsf{max}}^{2} \p{(\fivmax)^{2} \!+\! R_{\max}^{2} (\fiumax)^{2}} 
    \!+\! 4 \big( \zeta^2 + (\fivmax)^{4} \p{1\!+\!\disc}^{2}  \!+\!(\fiumax)^{4} \p{1\!+\!\disc^{2}}^{2}  \\ 
    & \phantom{\hskip 40pt} \!+\!4 \disc^{2} R_{\max}^{2}(\fivmax)^{2}(\fiumax)^{2} \big) \norm{\wreg}_{2}^{2}  \numberthis \label{eq:sigmacheck}
    \end{align*}
    
\end{theorem}

\begin{proof}
Our proof incorporates techniques from \citet{patilFiniteTimeAnalysis2024}. However, as described earlier, the analysis of mean-variance TD involves additional cross-terms, which necessitate significant deviations in the proof.

\noindent\textbf{\newline Step 1: Bias-variance decomposition with regularization}

For regularized TD, we solve the following linear system:
\begin{align}
    -(\M+\zeta \I) \bar w_{\reg} + \xi =0,\label{appendix:Mwxi-reg} 
\end{align}

The corresponding TD updates in \Cref{alg:TD-Critic-main} to solve \eqref{appendix:Mwxi-reg} would be:
\begin{align}
v_{t+1} & = (\I-\lreg \zeta)v_{t}+\lreg \; \check{\delta}_{t} \; \fiv(s_{t}), \label{appendix:vu-td-update-reg}\\ 
u_{t+1} & = (\I-\lreg \zeta) u_{t}+\lreg \; \check\epsilon_{t}\; \fiu(s_{t}), \nonumber
\end{align} 
where $\check\delta_{t},\check\epsilon_{t}$ are defined as
\begin{align} \label{eq:delta-epsilon-def-reg}
    \check\delta_{t} =& r(s_{t},a_{t}) + \disc \check v_{t}\tr  \fiv(s_{t+1}) - \check v_{t}\tr  \fiv(s_{t}) \\
    \check\epsilon_{t} =& r(s_{t},a_{t})^{2} + 2 \disc r(s_{t},a_{t}) \; \check v_{t}\tr  \fiv(s_{t+1}) + \disc^{2} \check u_{t}\tr  \fiu(s_{t+1})- \check u_{t}\tr  \fiu(s_{t}).  \nonumber
\end{align}

We rewrite the updates in an alternative form as follows:
\begin{align} \label{appendix:Mt-update-reg}
  \check  w_{t+1} =  \check w_{t} + \lreg (r_{t} \phi_{t}-(\zeta \I+\M_{t}) \check w_{t}),
\end{align}
where $\M_t,r_t,\phi_t$ are defined in \eqref{eq:Mt-update}. 

Letting $ \check h_{t}(w_{t}) = r_{t} \phi_{t} - (\zeta \I +\M_{t}) \check w_{t} $, we have
\begin{align} \label{appendix:wt+gammaht}
    \check w_{t+1} = \check w_{t} + \lreg \check h_{t}(\check w_{t}).
\end{align}

As in the case of the `vanilla' mean-variance TD, we derive a one-step recursion for the centered error, $\check z_{t+1} = \check w_{t+1}-\bar{w}_{\mathsf{reg}}$, as follows:
\begin{align*}
\check z_{t+1} &= \check w_{t}-\bar{w}_{\mathsf{reg}} + \lreg (r_{t}\phi_{t}-\M_{t}\check w_{t})+\lreg (\zeta \I+\M_{t}) \bar{w}_{\mathsf{reg}} - \lreg (\zeta \I+\M_{t}) \bar{w}_{\mathsf{reg}} \\
&= (\I-\lreg (\zeta \I+\M_{t}))(w_{t}-\bar{w}_{\mathsf{reg}})+\lreg(r_{t}\phi_{t}-(\zeta \I+\M_{t})\bar{w}_{\mathsf{reg}}) \\
&= (\I-\lreg (\zeta \I+\M_{t}))z_{t}+\lreg \check h_{t}(\wreg). \numberthis \label{eq:tdupdate-reg}
\end{align*}
Unrolling the equation above, we obtain
\begin{align*}
\check{z}_{t+1} = \mathbf{ \check C}^{t:0}\check{z}_{0}+\lreg \sum_{k=0}^{t} \Creg^{t:k+1} \check{h}_{k}(\bar{w}_{\mathsf{reg}}), \numberthis \label{eq:reg-centered-error}
\end{align*}

where 
\begin{align*}
   \Creg^{i:j} & = \begin{cases}
                (\I-\lreg (\zeta \I+\M_{i}))(\I-\lreg (\zeta \I+\M_{i-1})) \dots (\I-\lreg(\zeta \I+\M_{j})) & \text{if } i \geq j \\
                \; \I & \text{otherwise.} 
                \end{cases} 
\end{align*}

Taking expectations and using $\|a+b\|^2  \leq 2 \|a\|^{2}+2\|b\|^{2}$, we obtain, 
\begin{align} \label{eq:biasvardecomposition-reg1} \E\psq{\norm{\check{z}_{t+1}}^{2}} & \leq 2\E\p{\norm{ \mathbf{ \check C}^{t:0}\check{z}_{0}}^{2}}+2\lreg^{2} \E \psq{\norm{\sum_{k=0}^{t} \Creg^{t:k+1} \check{h}_{k}(\bar{w}_{\mathsf{reg}})}^{2}},  \\
& \leq 2 \check{z}_{t}^{\text{bias}} + 2 \lreg^{2} \check{z}_{t}^{\text{variance}} \nonumber,
\end{align}
where $\check z_{t}^{\text{bias}} = \E\left[\left\|\Creg^{t:0} \check z_{0}\right\|^{2}\right]$ and
$\check z_{t}^{\text{variance}} = \E \psq{\norm{\sum_{k=0}^{t} \Creg^{t:k+1} \check h_{k}(\wreg)}^{2}}$. 

\noindent\textbf{\newline Step 2: Bounding the bias term}

Before bounding the bias term, we first present and prove some useful lemmas.
\begin{lemma} \label{lem:bounddonM}
    \[\norm{\M} \leq \p{ (\fivmax)^{4} (1+\disc)^{2}+ (\fiumax)^{4} (1+\disc^{2})^{2}+4 \disc^{2} R_{\mx}^{2}(\fivmax)^{2}(\fiumax)^{2} }^{\frac{1}{2}}.\]
\end{lemma}
\begin{proof}
Recall that $\M=\E[\M_{t}\mid \F_{t}]$ where
\begin{align*} 
\M_t \define \begin{pmatrix}
    \A_{t} &  \zeromat \\
    \C_{t} & \B_{t} \\ 
\end{pmatrix} \textrm{with } 
\A_{t} &\define \fiv (s_{t}) \fiv (s_{t})\tr   - \disc \fiv(s_{t}) \fiv(s_{t+1})\tr  , \\ 
\B_{t} &\define \fiu (s_{t}) \fiu (s_{t})\tr   - \disc^{2} \fiu(s_{t}) \fiu(s_{t+1})\tr , \\ 
\C_{t} &\define -2\disc r_{t} \fiu (s_{t}) \fiv (s_{t+1})\tr .
\end{align*} 
We bound the norms of the matrices $\A_{t},\B_{t},\C_{t}$ using the boundedness assumptions on features and rewards (\Cref{asm:bddFeatures,asm:bddRewards}) as follows:
\begin{align*}
    \norm{\A_t} \leq (1+\disc) (\fivmax)^{2},  
    \norm{\B_t} \leq (1+\disc^2) (\fiumax)^2, 
    \norm{\C_t} \leq  2\disc R_{\mx} \fivmax \fiumax \numberthis \label{bound-on-at-bt-ct}.
\end{align*}
Next, we derive the following result:
\begin{align*}
\norm{\M} &= \norm{\E[\M_{t}\mid \F_{t}]}  \stackrel{(i)}{\leq} \E[\norm{\M_{t}}\mid \F_{t}] \\ 
& \stackrel{(ii)}{\leq} \norm
    {\begin{pmatrix}
    (1+\disc) (\fivmax)^{2} & 0 \\
    2\disc R_{\mx} \fivmax \fiumax & (1+\disc^2) (\fiumax)^{2}
    \end{pmatrix}}_{F}  \\
& \stackrel{(iii)}{\leq} \p{ (\fivmax)^{4} (1+\disc)^{2}+ (\fiumax)^{4} (1+\disc^{2})^{2}+4 \disc^{2} R_{\mx}^{2}(\fivmax)^{2}(\fiumax)^{2} }^{\frac{1}{2}},
\end{align*}
where (i) follows from Jensen's inequality, (ii) follows from \eqref{bound-on-at-bt-ct}, and (iii) follows from expanding the Frobenius norm.
 
\end{proof}

\begin{lemma}\label{lem:regtd-I-gammasupport}
    For any $\check y \in \mathbb{R}^{2q}$ measurable w.r.t $\F_{t}$ and $\lreg \leq \check \lr_{\mx}$ as in \Cref{thm:regtd-expectation_bound}. The following holds:
    \begin{align*}
    \E\left[\check y (\I - \lreg (\zeta \I + \M_t))\tr  (\I - \lreg (\zeta \I + \M_t))\check y \; \middle| \;\F_{t}\right] &\leq \left(1-\lreg(2\mu+\zeta)\right)\norm{\check y}_{2}^{2}, \\
    \E\left[\norm{(\I - \lreg(\zeta \I+\M_t)) \check y}_{2}\; \middle| \; \F_{t}\right] &\leq \left(1-\frac{\lreg(2\mu+\zeta)}{2}\right)\norm{\check y}_{2}.
    \end{align*}
\end{lemma}
\begin{proof} Notice that
\begin{align*} 
& \E\left[{\check y}\tr  (\I - \lreg (\zeta \I + \M_t))\tr  (\I - \lreg (\zeta \I + \M_t)){\check y} \; \middle| \; \F_{t}\right] \\
&= \E\left[{\check y}\tr  (\I - 2\lreg\zeta\I - \lreg(\M_t + \M_{t}\tr )) + \lreg^2(\zeta^2 \I + \zeta  (\M_t + \M_{t}\tr ) + \M_{t}\tr  \M_t)  {\check y} \; \middle| \; \F_{t}\right]\\
&= \E\left[{\check y}\tr {\check y}\; \middle| \; \F_{t}\right] -  \lreg \E\left[{\check y}\tr  2\zeta \I {\check y} \;\middle|\;\F_{t}\right] -\lreg \underbrace{{\check y}\tr  \E\left[ \M\tr _t + \M_t \; \middle|\;\F_{t}\right] {\check y}}_{\text{Term 1}} \\
&\quad +\lreg^2 \underbrace{{\check y}\tr  \E\left[ \M_{t}\tr \M_t \;\middle|\; \F_{t}\right]{\check y}}_{\text{Term 2}}  + \lreg^2 \zeta \underbrace{{\check y}\tr \E\left[\M_{t} + \M_{t}\tr \; \middle|\; \F_{t}\right]{\check y} }_{\text{Term 3}}+ \lreg^2 \E\left[{\check y}\tr  \zeta^2\I {\check y} \; \middle|\; \F_{t}\right]. \label{eq:T1to3split}
\numberthis
\end{align*}
We bound Term 1 in \eqref{eq:T1to3split} as follows:  
\begin{align*}
{\check y}\tr  \E\left[ \M\tr _t + \M_t \;\middle|\; \F_{t}\right] {\check y} &= {\check y}\tr  (\M\tr  + \M){\check y} \stackrel{(i)}{\geq}2\mu\norm{{\check y}}_{2}^{2},
\numberthis\label{eq:bound-M1-regtd}
\end{align*}
where (i) follows from \Cref{asm:phiFullRank}, which implies that $\M + \M\tr$ has a minimum positive eigenvalue $\mu = \eigmin$.

We bound Term 2 in \eqref{eq:T1to3split} using the bound for T2 in \eqref{eq:T2Bd} as follows: 
\begin{samepage}
\begin{align*}
     {\check y}\tr  \E[ \M_{t}\tr \M_t\mid\F_{t}]{\check y}  & \leq \p{(\fivmax)^{2} \p{1+\disc}^{2} + 4 \disc^{2} R_{\mx}^{2} {(\fiumax)}^{2}} \check v\tr  \matB \check v   \\ 
    & \quad + (\fiumax)^{2} \p{1+\disc^{2}}^{2} \check u\tr  \matG \check u \nonumber + 2 (\fiumax)^{2} R_{\mx} (\disc (1+\disc^2)) \check v\tr \p{\matB+\matG} \check u. 
\end{align*}
\end{samepage}
We bound Term 3 in \eqref{eq:T1to3split} as follows: 
\begin{align*}
    &\check y\tr\E [\M_{t}+\M_{t}\tr\mid \F_{t}] \check y \leq  \norm{\E[\M_{t}+\M_{t}\tr\mid \F_{t}]} \norm{\check y}^{2} \leq \norm{\M+\M\tr} \norm{\check y}^2 \\ 
    & \stackrel{(i)}{\leq} 2\p{ (\fivmax)^{4} (1+\disc)^{2}+ (\fiumax)^{4} (1+\disc^{2})^{2}+4 \disc^{2} R_{\mx}^{2}(\fivmax)^{2}(\fiumax)^{2} }^{\frac{1}{2}} \norm{\check y}^2,
\end{align*}
where (i) follows from \Cref{lem:bounddonM}.

Substituting  the bounds for Terms 1--3 in \eqref{eq:T1to3split}, we obtain
\begin{align*}
& \E[{\check y}\tr  (\I - \lreg (\zeta \I + \M_t))\tr  (\I - \lreg (\zeta \I + \M_t)){\check y} \mid \F_{t}]  \\ 
& \leq \E[{\check y}\tr {\check y}\vert\F_{t}] -  \lreg \E[{\check y}\tr  2\zeta \I {\check y} \mid\F_{t}] -\lreg 2\mu\norm{\check{y}}^2  \\
&\quad +\lreg^2 \Big(\big((\fivmax)^{2} \p{1+\disc}^{2} + 4 \disc^{2} R_{\mx}^{2} {(\fiumax)}^{2}\big) \check v\tr  \matB \check v  \\
&\quad + (\fiumax)^{2} \p{1+\disc^{2}}^{2} \check u\tr  \matG \check u + 2 (\fiumax)^{2} R_{\mx} (\disc (1+\disc^2)) \check v\tr \p{\matB+\matG} \check u \Big) \\ 
& \quad + \lreg^2 \Big( 2 \big((\fivmax)^{4} (1+\disc)^{2}+ (\fiumax)^{4} (1+\disc^{2})^{2}+4 \disc^{2} R_{\mx}^{2}(\fivmax)^{2}(\fiumax)^{2}\big)^{\frac{1}{2}} \norm{\check y}^2 \Big)\\ & \quad +\lreg^2 \E[{\check y}\tr  \zeta^2\I {\check y} \mid \F_{t}]. \\
&\stackrel{(i)}{\leq} \norm{\check y}_{2}^{2}(1-2\lreg ( \mu+\zeta)) +  \lreg^{2}  \Big (\big((\fivmax)^{2} \p{1+\disc}^{2} + 4 \disc^{2} R_{\mx}^{2} {(\fiumax)}^{2}\big) \lambda_{\mx}(\matB) \norm{\check v}^{2}_{2} \\ & \phantom{\hskip 10pt} + (\fiumax)^{2} \p{1+\disc^{2}}^{2} \lambda_{\mx}(\matG)\norm{\check u}^{2}_{2} + (\fiumax)^{2} R_{\mx} (\disc (1+\disc^2))\lambda_{\mathsf{max}}(\matB+\matG) \norm{\check y}_{2}^{2} \\ 
&\phantom{\hskip 10pt}+ 2 \zeta \big((\fivmax)^{4} (1+\disc)^{2} + (\fiumax)^{4} (1+\disc^{2})^{2}\\ 
&\phantom{\hskip 10pt} +4 \disc^{2} R_{\mx}^{2}(\fivmax)^{2}(\fiumax)^{2}\big)^{\frac{1}{2}} \norm{\check y}^2  + \zeta^2 \norm{\check y}_{2}^{2} \Big ) \\ 
&\leq \bigg(1-\lreg \Big(2\mu+  2\zeta - \lreg \big(\mx \big \{ \big((\fivmax)^{2} \p{1+\disc}^{2} + 4 \disc^{2} R_{\mx}^{2} {(\fiumax)}^{2}\big) \lambda_{\mx}(\matB)\bm{,} \\ & \phantom{\hskip 10pt}  (\fiumax)^{2} \p{1+\disc^{2}}^{2} \lambda_{\mx}(\matG)\big \} + (\fiumax)^{2} R_{\mx} (\disc (1+\disc^2))\lambda_{\mathsf{max}}(\matB+\matG) \\ 
 & \phantom{\hskip 10pt} + \zeta^2  + 2 \zeta \big((\fivmax)^{4} (1+\disc)^{2} + (\fiumax)^{4} (1+\disc^{2})^{2}\\ 
&\phantom{\hskip 10pt} +4 \disc^{2} R_{\mx}^{2}(\fivmax)^{2}(\fiumax)^{2}\big)^{\frac{1}{2}}  \big) \Big) \bigg) \norm{\check y}_{2}^{2} \\
&  \leq \bigg(1\!-\!\lreg \Big(2\mu+  2\zeta \!-\! \lreg \big(\max \left \{ 4 (\fivmax)^{4}  \!+\!4 \disc^{2} R_{\max}^{2} {(\fiumax)}^{2}(\fivmax)^{2}  \bm{,} 
4(\fiumax)^{4}  \right \}  \\ 
& \phantom{\hskip 10pt} + 2 \disc R_{\max} \left((\fivmax)^{2}(\fiumax)^{2}\!+\!(\fiumax)^{4}\right)  \zeta^2 + 2 \zeta \big((\fivmax)^{4} (1+\disc)^{2} + (\fiumax)^{4} (1+\disc^{2})^{2} \\ 
& \phantom{\hskip 10pt} +4 \disc^{2} R_{\mx}^{2}(\fivmax)^{2}(\fiumax)^{2}\big)^{\frac{1}{2}}  \big) \Big) \bigg) \norm{\check y}_{2}^{2}  \\
& \stackrel{(ii)}{\leq} (1-\lreg (2\mu+\zeta)) \norm{\check y}_{2}^{2},\numberthis\label{eq:s123}
\end{align*}
where (i) follows from \Cref{lem:UB for vtr(B+G)u} and using  $x \tr  \mathbf{Q} x \leq \lambda_{\mathsf{max}}(\mathbf{Q}) \norm{x}^{2}_{2}$, and (ii) follows from choosing $\lreg \leq \lreg_{\mx}$.

Taking the square root on both sides of \eqref{eq:s123} and applying Jensen's inequality, we obtain
\begin{samepage}
\begin{align*} 
    & \E \psq{\norm{(\I-\lreg (\zeta\I+\M_{t})) \check y} \; \middle| \; \F_{t}}  \leq (1-\lreg (2\mu+\zeta))^{\frac{1}{2}} \norm{\check y}_{2} \stackrel{(i)}{\leq} \p{1-\frac{\lreg (2\mu+\zeta)}{2}} \norm{\check y}_{2}, \numberthis 
\end{align*}
where (i) follows from applying the inequality $(1-x)^{\frac{1}{2}} \leq 1-\frac{x}{2}$, for $x \geq 0$ with $x = \lreg (2\mu+\zeta)$.
\end{samepage}
\end{proof}
Now, we bound the bias term in \eqref{eq:biasvardecomposition-reg1} as follows:
\begin{align*}
    \check z_{t}^{\text{bias}}  &= \E \psq{\norm{\Creg^{t:0}\check z_{0}}^{2}} \\
    & = \E \psq {\E \psq{
    \p{\Creg^{t-1:0}\check z_{t-1}^{\text{bias}}}\tr  (\I-\lreg (\zeta \I+\M_{t}))\tr 
    (\I-\lreg (\zeta \I+\M_{t}))
     (\Creg^{t-1:0}\check z_{t-1}^{\text{bias}})}
      \;\middle|\; \F_{t} } \\
    & \stackrel{(i)}{\leq} \p{1-\lreg (2\mu+\zeta)} \E \psq{\norm{\Creg^{t-1:0}\check z_{t-1}^{\text{bias}}}^{2}} \\ 
    & \stackrel{(ii)}{\leq} \p{1-\lreg (2\mu+\zeta)}^{t} \E \psq{\norm{\check z_{0}}^{2}} \numberthis \label{eq:bias-recursion-reg}\\
    & \stackrel{(iii)}{\leq} \exp\p{-\lreg (2\mu+\zeta) t} \E \psq{\norm{\check z_{0}}^{2}}  \numberthis \label{eq:biasbd-reg},
\end{align*}
where (i) follows from \Cref{lem:regtd-I-gammasupport}, (ii) follows from unrolling the recursion and applying \Cref{lem:regtd-I-gammasupport} repeatedly, and (iii) follows from applying the inequality  
\[
(1-\lr(2\mu+\zeta))^{t} = \exp(t\log(1-\lr(2\mu+\zeta))) \leq \exp(-\lr(2\mu+\zeta) t).
\]

\noindent\textbf{Step 3: Bounding the variance term}

Before we find an upper bound for the variance term, we upper bound on $\norm{h_{t}(\wreg)}^{2}$ as follows: 
\begin{align*}
\norm{\check h_{t}(\wreg)}^{2} &= \norm{r_{t} \phi(s_{t}) - (\zeta \I +\M_{t}) \wreg}^{2} \\
& \stackrel{(a)}{\leq} 2\norm{r_{t} \phi(s_{t})}^{2}+ 2\norm{(\zeta \I +\M_{t}) \wreg}^{2}_{2} \\ 
& \stackrel{(b)}{\leq} 2  R_{\mathsf{max}}^{2} \p{(\fivmax)^{2} + R_{\mx}^{2} (\fiumax)^{2}}+ 2 \norm{\zeta \I + \M_{t}}^{2} \norm{\wreg}_{2}^{2} \\
& \stackrel{(c)}{\leq} 2  R_{\mathsf{max}}^{2} \p{(\fivmax)^{2} + R_{\mx}^{2} (\fiumax)^{2}} +  4 \big( \zeta^2+ (\fivmax)^{4} \p{1+\disc}^{2}\\ & \quad + (\fiumax)^{4} \p{1+\disc^{2}}^{2}+ 4 \disc^{2} R_{\mx}^{2}(\fivmax)^{2}(\fiumax)^{2} \big) \norm{\wreg}_{2}^{2}  \numberthis \label{eq:reg-upperboundforMt}\\
& = \check \sigma^{2}, \numberthis \label{eq:reg-upperboundforsigmasq}
\end{align*} 
where (a) follows from the inequality \( \norm{a+b}^{2} \leq 2\norm{a}^{2} + 2\norm{b}^{2} \);  
(b) follows from the bounds on features and rewards (\Cref{asm:bddFeatures,asm:bddRewards});  
and (c) follows from the bound on \( \M \) (\Cref{lem:bounddonM}) and the inequality \( \norm{a+b}^{2} \leq 2\norm{a}^{2} + 2\norm{b}^{2} \).

Next, we bound the variance term in \eqref{eq:biasvardecomposition-reg1} as follows:
\begin{align*}
\check z_{t}^{\text{variance}} & =\E \psq{\norm{\sum_{k=0}^{t} \Creg^{t:k+1} \check h_{k}(\wreg)}_{2}^{2}}\\
& \stackrel{(a)}{\leq} \sum_{k=0}^{t} \E \psq{ \norm{ \Creg^{t:k+1} \check h_{k}(\wreg)}^{2}_{2}} \\
& \stackrel{(b)}{\leq} \sum_{k=0}^{t} \E \psq{ \norm{\Creg^{t:k+1}}^{2} \norm{\check h_{k}(\wreg)}^{2}} \\ 
& \stackrel{(c)}{\leq} \check \sigma^{2} \sum_{k=0}^{t} \E \psq{\norm{\Creg^{t:k+1}}_{2}^{2}} \\ 
& \stackrel{(d)}{\leq} \check \sigma^{2} \sum_{k=0}^{t} \E \psq{ \E \psq{\norm{\Creg^{t:k+1}}_{2}^{2}| \F_{t}}} \\
& \stackrel{(e)}{\leq} \check \sigma^{2} \sum_{k=0}^{t} \E \psq{ \E \psq{\norm{(\I-\lreg (\zeta \I +\M_{t}))\Creg^{t-1:k+1}}_{2}^{2}| \F_{t}}} \\
& \stackrel{(f)}{\leq} \check \sigma^{2} \sum_{k=0}^{t} \E \psq{ \E \psq{\norm{\I - \lreg(\zeta \I +\M_{t})}^{2}| \F_{t}} \norm{\Creg^{t-1:k+1}}_{2}^{2}} \\
& \stackrel{(g)}{\leq} \check \sigma^{2} \sum_{k=0}^{t} (1-\lreg (2\mu+\zeta)) \E \psq{ \norm{\Creg^{t-1:k+1}}_{2}^{2}} \\
& \stackrel{(h)}{\leq} \check \sigma^{2} \sum_{k=0}^{t}  (1-\lreg (2\mu+\zeta))^{t-k} \\
& \stackrel{(i)}{\leq} \frac{\check \sigma^{2}} {\lreg (2\mu+\zeta)}, \numberthis \label{eq:reg-variancebd}
\end{align*}
where (a) follows from the triangle inequality and linearity of expectations;
(b) follows from applying the inequality \( \norm{\mathbf{A}x} \leq \norm{\mathbf{A}} \norm{x} \);  
(c) follows from a bound on \( \norm{\check{h}_{k}(\wreg)}^{2} \);  
(d) follows from the tower property of conditional expectations;  
(e) follows from unrolling the product of matrices \( \Creg^{t:k+1} \) by one time step;  
(f) follows from applying the inequality \( \norm{\mathbf{A}\mathbf{B}} \leq \norm{\mathbf{A}} \norm{\mathbf{B}} \);  
(g) follows from \Cref{lem:regtd-I-gammasupport};  
(h) follows from unrolling the product of matrices;  
and (i) follows from computing the upper bound for the finite geometric series.

\noindent \textbf{\newline Step 4: Tail Averaging}

Using the parallel arguments from \Cref{sec:Appendix-tail-avg-expectationbd}, we derive the error bounds for the regularized tail-averaged iterate, focusing on its bias and variance terms, as follows:

\noindent\textbf{\newline 4 (a) Bias-variance decomposition for tail averaging} 

The tail averaged error, starting from time $k+1$, with $N=t-k$ is given by:
\begin{align*}
    \check z_{k+1:t} &= \frac{1}{N}\sum_{i = k+1}^{k+N}\check z_{i}.
\end{align*}
By taking expectations, $\norm{\check z_{k+1:t}}^2$ can be expressed as:
\begin{align}
\E\left[\norm{\check z_{k+1:t}}_{2}^{2}\right] &= \frac{1}{N^{2}}\sum_{i,j = k+1}^{k+N}\E\left[\check z_i\tr  \check z_j\right]\nonumber\\
&\stackrel{(a)}{\leq} \frac{1}{N^{2}}\bigg(\sum_{i= k+1}^{k+N}\E\left[\norm{\check z_i}_{2}^{2}\right] + 2 \sum_{i=k+1}^{k+N-1}\sum_{j=i+1}^{k+N} \E\left[\check z_{i}\tr  \check z_{j}\right] \bigg)\numberthis\label{eq:cross-term-decomp-reg},
\end{align}
where $(a)$ follows from isolating the diagonal and off-diagonal terms.

Next, we state and prove \Cref{lem:crosstermbound-reg} to bound the second term in terms of the first term in \eqref{eq:cross-term-decomp-reg}.
\begin{lemma} \label{lem:crosstermbound-reg}
For all $i\ge 1$, we have
\begin{align}
    \sum_{i=k+1}^{k+N-1}\sum_{j=i+1}^{k+N} \E\psq{\check z_{i}\tr  \check z_{j}} 
    &\leq \frac{2}{\lreg (2\mu+\zeta)}\sum_{i=k+1}^{k+N}\E\psq{\norm{\check z_{i}}^{2}_{2}}.
\end{align}
\end{lemma}
\begin{proof}
    \begin{align*}
        \sum_{i=k+1}^{k+N-1}\sum_{j=i+1}^{k+N} \E \left[\check z_{i}\tr  \check z_{j}\right]  &\stackrel{(a)}{=}  
        \sum_{i=k+1}^{k+N-1}\sum_{j=i+1}^{k+N} \E \left[\check z_{i}\tr  (\Creg^{j:i+1}\check z_{i} + \lreg \sum_{l=i+1}^{j-i-1} \Creg^{j:l+1}\check h_{l}(\wreg))\right]\\
        &\stackrel{(b)}{=}\sum_{i=k+1}^{k+N-1}\sum_{j=i+1}^{k+N} \E  \left[\check z_{i}\tr  \Creg^{j:i+1}z_{i}\right] \\
        &\stackrel{(c)}{\leq}\sum_{i=k+1}^{k+N-1}\sum_{j=i+1}^{k+N} \E  \left[ \norm{\check z_{i}} \E [ \norm{\Creg^{j:i+1}  \check z_{i}}|\F_{j}]\right] \\
        &\stackrel{(d)}{\leq}\sum_{i=k+1}^{k+N-1}\sum_{j=i+1}^{k+N} \left(1 - \frac{\lreg (2\mu+\zeta)}{2}\right)^{j-i}\E\left[\norm{\check z_{i}}^{2}_{2}\right]\\
        &\leq \sum_{i=k+1}^{k+N}\E\left[\norm{\check z_{i}}^{2}_{2}\right]\sum_{j=i+1}^{\infty}\left(1- \frac{\lreg (2\mu+\zeta)}{2}\right)^{j-i}\\
        &\stackrel{(e)}{\leq}\frac{2}{\lreg (2\mu+\zeta)}\sum_{i=k+1}^{k+N}\E\left[\norm{\check z_{i}}^{2}_{2}\right],
    \end{align*}
where (a) follows from expanding \( z_{j} \) using \eqref{eq:reg-centered-error};  
(b) follows from the observation that  
\[
\E[\check h_t(\wreg)\mid\F_{t}] = \E [r_{t}\phi_{t}-(\zeta \I+\M_{t}) \wreg \mid \F_{t}] = \xi-(\M+\zeta \I)\wreg= 0;
\]
(c) follows from applying the Cauchy–Schwarz inequality  
and the tower property of expectations;  
(d) follows from a repetitive application of \Cref{lem:regtd-I-gammasupport};  
and (e) follows by computing the limit of the infinite geometric series.
\end{proof} 
By substituting the result of \Cref{lem:crosstermbound-reg} into \eqref{eq:cross-term-decomp-reg}, we obtain
\begin{align*}
& \E\left[\norm{\check z_{k+1:t}}_{2}^{2}\right] \leq \frac{1}{N^{2}}\left(\sum_{i = k+1}^{k+N}\E\left[\norm{\check z_i}_{2}^{2}\right] + \frac{4}{\lreg (2\mu+\zeta)} \sum_{i = k+1}^{k+N}\E\left[\norm{\check z_i}_{2}^{2}\right] \right)\\
&= \frac{1}{N^{2}}\left(1+\frac{4}{\lreg (2\mu+\zeta)}\right) \sum_{i = k+1}^{k+N}\E\left[\norm{\check z_i}_{2}^{2}\right]\\
&\stackrel{(a)}{\leq} \underbrace{\frac{2}{N^{2}}\left(1+\frac{4}{\lr (2\mu+\zeta)}\right) \sum_{i = k+1}^{k+N}   \check z_{i}^{\mathsf{bias}}}_{\check z_{k+1, N}^{\mathsf{bias}}} + \underbrace{\frac{2}{N^{2}}\bigg(1+\frac{4}{\lr (2\mu+\zeta)}\bigg)\lreg^{2}  \sum_{i = k+1}^{k+N}  \check z_{i}^{\mathsf{variance}}}_{\check z_{k+1:t}^{\mathsf{variance}}}\numberthis\label{eq:tail-avg-bias-var-reg},
\end{align*}
where $(a)$ follows from \eqref{eq:biasvardecomposition-reg1}.

\noindent\textbf{\newline 4 (b) Bounding the bias term} 

First term, $\check z_{k+1:t}^{\mathsf{bias}}$in \eqref{eq:tail-avg-bias-var-reg}  is bounded as follows:
\begin{align*}
\check z_{k+1:t}^{\mathsf{bias}} &\leq \frac{2}{N^{2}}\left(1 + \frac{4}{\lreg (2\mu+\zeta)}\right) \sum_{i = k+1}^{\infty}   \check z_{i}^{\mathsf{bias}} \\
&\stackrel{(a)}{\leq} \frac{2}{N^{2}}\left(1 + \frac{4}{\lreg (2\mu+\zeta)}\right) \sum_{i = k+1}^{\infty} (1 - \lreg (2\mu+\zeta))^{i}\E\left[\norm{\check z_{0}}^{2}_{2}\right] \\
&\stackrel{(b)}{=}\frac{2\E\left[\norm{\check z_{0}}^{2}_{2}\right]}{\lreg (2\mu+\zeta) N^{2}}\left(1-\lreg(2\mu+\zeta)\right)^{k+1}\left(1 + \frac{4}{\lreg (2\mu+\zeta)}\right),
\end{align*}
where $(a)$ follows from \eqref{eq:bias-recursion-reg}, which provides a bound on $\check z_{i}^{\mathsf{bias}}$ and $(b)$ follows from the bound on the summation of a geometric series. 

\noindent\textbf{4 (c) Bounding the variance term} 

Next, the second term $z_{k+1:t}^{\mathsf{variance}}$ in \eqref{eq:tail-avg-bias-var-reg} is bounded as follows:
\begin{align*}
\check z_{k+1:t}^{\mathsf{variance}} &\stackrel{(a)}{\leq}\frac{2\lreg^2}{N^{2}}\left(1+\frac{4}{\lreg (2\mu+\zeta)}\right)\sum_{i=k+1}^{k+N}\frac{\check \sigma^2}{\lreg (2\mu+\zeta)}\\
&\leq \frac{2\lreg^2}{N^{2}}\left(1+\frac{4}{\lreg (2\mu+\zeta)}\right)\sum_{i=0}^{N}\frac{\check \sigma^2}{\lreg (2\mu+\zeta)}\\
&= \bigg(1+\frac{4}{\lreg (2\mu+\zeta)}\bigg)\frac{2\lreg \check \sigma^2}{(2\mu+\zeta) N},
\end{align*}
where  $(a)$ follows from \eqref{eq:reg-variancebd}, which provides a bound on $\check z_{i}^{\mathsf{variance}}$.

\noindent\textbf{Step 5: Clinching argument}

Finally substituting the bounds on $\check z_{k+1:t}^{\mathsf{bias}}$ and $\check z_{k+1:t}^{\mathsf{variance}}$ in \eqref{eq:tail-avg-bias-var-reg}, we get
\begin{align*}
& \E[\norm{\check z_{k+1:t}}^{2}_{2}] \\ &\leq \bigg(1+ \frac{4}{\lreg (2\mu+\zeta)}\bigg)\bigg(\frac{2}{\lreg (2\mu+\zeta) N^{2}}(1-\lreg (2\mu+\zeta))^{k+1}\E[\norm{\check z_{0}}^{2}_{2}] + \frac{2\lreg \check \sigma^2}{(2\mu+\zeta) N}\bigg),\\
&\stackrel{(a)}{\leq}\bigg(1+ \frac{4}{\lreg (2\mu+\zeta)}\bigg)\bigg(\frac{2 \exp(-k \lreg (2\mu+\zeta))}{\lreg (2\mu+\zeta) N^{2}}\E[\norm{z_{0}}^{2}_{2}] + \frac{2\lreg \check \sigma^2}{(2\mu+\zeta) N}\bigg)\\
&\stackrel{(b)}{\leq}\frac{10 \exp(-k \lreg (2\mu+\zeta))}{\lreg^2 (2\mu+\zeta)^2 N^{2}}\E\left[\norm{\check z_{0}}^{2}_{2}\right] + 
\frac{10\check \sigma^2}{(2\mu+\zeta)^{2} N}, \numberthis \label{eq:reg-td-expectation-bd}
\end{align*}
where $(a)$ follows from $(1+x)^{y} = \exp(y\log(1+x)) \leq \exp(xy)$, and $(b)$ uses $\lreg(2\mu+\zeta) < 1$ as $\lreg \leq \lreg_\mathsf{max}$ defined in \Cref{thm:regtd-expectation_bound}, which implies that 
\[1+ \frac{4}{\lreg (2\mu+\zeta)} \leq \frac{5}{\lreg (2\mu+\zeta)}.\]

\end{proof} 
\subsection*{Proof of \Cref{thm:regTDlamdasq}} 
\label{sec:Appendix-regTDlamdasq}

The proof of \Cref{thm:regTDlamdasq} builds on \Cref{thm:regtd-expectation_bound} and a bound on $\norm{\check w_{k+1:t} - \wreg}_{2}^{2}$, incorporating techniques from \citep[Corollary 1,2]{patilFiniteTimeAnalysis2024}.

\begin{proof}\label{proof:cor1} 
Notice that
    \begin{align*}
                 \E\left[\norm{\check w_{k+1:t} - \bar{w}}_{2}^{2}\right] 
                    &\stackrel{(i)}{\leq}   \underbrace{2\norm{\wreg - \bar{w}}_{2}^{2}}_{\text{Term 1}} + \underbrace{2\E\left[\norm{\check w_{k+1:t} - \wreg}_{2}^{2}\right]}_{\text{Term 2}},\numberthis\label{eq:regtd-fixedpt-tailavg}
    \end{align*}
where (i) follows from applying $\norm{a+b}^2 \le 2\norm{a}^2+2\norm{b}^2$.

We bound Term 1 below.
\begin{align*}
    \norm{\bar{w}-\wreg}_{2}^{2} &= \norm{\M^{-1}\xi - (\M + \zeta \I)^{-1}\xi}_{2}^{2}\\
                &\stackrel{(a)}{\leq} \norm{\M^{-1} - (\M + \zeta \I)^{-1}}_{2}^{2}\norm{\xi}_{2}^{2}\\
                &= \norm{\M^{-1} (\M + \zeta \I - \M) (\M + \zeta \I)^{-1}}_{2}^{2}\norm{\xi}_{2}^{2}\\
                & \leq \norm{\M^{-1}}_{2}^{2} \zeta^2 \norm{(\M + \zeta \I)^{-1}}_{2}^{2}\norm{\xi}_{2}^{2}\\
                &\stackrel{(b)}{\leq}\frac{\zeta^2 (R_{\mathsf{max}}^{2} \p{(\fivmax)^{2} + R_{\mx}^{2} (\fiumax)^{2}})}{\iota^2(\zeta +\iota)^2},\numberthis\label{eq:regtd-drift}
            \end{align*}
where $(a)$  follows from $\norm{\mathbf{AB}} \leq \norm{\mathbf{A}}\norm{\mathbf{B}}$, and $(b)$ follows from the fact that \\ $\norm{\M^{-1}} = 1/\iota_{\min}(\M)$, where $\iota=\iota_{\min}(\M)$ is the minimum singular value of $\M$. 

We observe that \eqref{eq:reg-td-expectation-bd} provides a bound for Term 2. Applying this bound along with \eqref{eq:regtd-drift} in \eqref{eq:regtd-fixedpt-tailavg}, we obtain
\begin{align*}
 \E\left[\norm{\check w_{k+1:t} - \bar{w}}_{2}^{2}\right] & 
\!\leq\!  \frac{20 \exp(-k \lreg (2\mu+\zeta))}{\lreg^2 (2\mu+\zeta)^2 N^{2}}\E\left[\norm{\check z_{0}}^{2}_{2}\right] \!+\! 
\frac{20\check \sigma^2}{(2\mu+\zeta)^{2} N} \\ & \quad \!+\! \frac{2\zeta^2 (R_{\mathsf{max}}^{2} \p{(\fivmax)^{2} + R_{\mx}^{2} (\fiumax)^{2}})}{\iota^2(\zeta +\iota)^2}.\numberthis \label{eq:FinalBoundRegularized}
    \end{align*}
For $\zeta = \frac{1}{\sqrt{N}}$, we obtain the following upper bounds, where we use a coarse bound on $2\mu + \zeta$ and similar simplifications in the exponent and denominator.
\begin{align*}
   \E\left[\norm{\check w_{k+1:t} - \bar{w}}_{2}^{2}\right]  & \leq \frac{  5\exp{(-k \lreg \mu)}}{\lreg^2 \mu^2 N^{2}}\E\left[\norm{{\check w}_{0} - \bar{w}_{\mathsf{reg}}}^{2}_2\right] \!+\! \frac{5\check \sigma^{2}}{\mu^2 N} \\ & \quad \!+\! \frac{2 (R_{\mathsf{max}}^{2} \p{(\fivmax)^{2} + R_{\mx}^{2} (\fiumax)^{2}})}{\iota^4 N}.\numberthis \label{eq:FinalBoundZetaSpecificVal}
\end{align*}
\end{proof}
\section{High Probability Bounds for Mean-Variance TD}\label{appendix:high-probability-bd} 

For the high probability bound, we consider the following update rule:
\begin{align} \label{eq:TDupdatewithprojection}
w_{t+1} = \Gamma(w_{t}+\lr h_{t}(w_{t})),
\end{align}
where $\Gamma$ projects on to the set $\mathcal{C} \define \{w \in \mathbb{R}^{2q} \mid \norm{w}_2 \leq H \}$.

Under \Cref{asm:projection}, we first state and prove a high-probability bound for the tail-averaged iterate in the next subsection. Then, we derive the high-probability bound for the regularized tail-averaged iterate.
\subsection{Bounds for vanilla (un-regularized) mean-variance TD} \label{sec:Appendix-highprob}
\begin{theorem}[Restatement of \Cref{thm:high-prob-bound}]\label{thm:app:high-prob-bound}
    	Suppose \Crefrange{asm:stationary}{asm:projection} hold. Run \Cref{alg:TD-Critic-main} for $t$ iterations  with step size $\lr$ as defined in \Cref{thm:expectation_bound-tailavg}. Then, for any  $\delta\in (0,1]$, we have the following bound for the projected tail-averaged iterate $w_{k+1:t}$ with $N=t-k$:
    \begin{align*}
        &\mathbb{P}\bigg(\norm{w_{k+1:t} \!-\! \bar{w}}_{2} \leq \frac{2\tau}{\mu \sqrt{N}}\sqrt{\log\left(\frac{1}{\delta}\right)} \!+\! \frac{4 \exp \p{-k \lr \mu}}{\lr \mu N}\E\left[\norm{w_{0} - \bar{w}}_{2}\right] \!+\! \frac{4\tau}{\mu \sqrt{N}} \bigg) \!\geq\! 1 \!-\! \delta,
    \end{align*}
    where $w_0, \bar{w},\lr$ are defined as in \Cref{thm:expectation_bound}, and 
    \begin{align*}
    \tau =& \big(2  R_{\mathsf{max}}^{2} \p{(\fivmax)^{2} + R_{\max}^{2} (\fiumax)^{2}} + 2\big( (\fivmax)^{4} \p{1+\disc}^{2}+ (\fiumax)^{4} \p{1+\disc^{2}}^{2} \\ & + 4 \disc^{2} R_{\max}^{2}(\fivmax)^{2}(\fiumax)^{2} \big) H^{2} \big)^{\frac{1}{2}}.
    \end{align*}
\end{theorem}

The proof follows a similar structure to \citep[Theorem~2]{patilFiniteTimeAnalysis2024} and \citep[Proposition~8.3]{prashanthConcentrationBoundsTemporal2021}, with necessary adaptations to account for our setting.

\begin{proof} 
A martingale difference decomposition of $\norm{z_{k+1, N}}_{2}  - \E[\norm{z_{k+1:t}}_{2}]$ is as follows:
\begin{align*} \numberthis\label{eq:D}
\norm{z_{k+1, N}}_{2} - \E\left[\norm{z_{k+1:t}}_{2}\right] = 
\sum_{i=k+1}^{k+N}(g_{i}-g_{i-1})=\sum_{i=k+1}^{k+N}D_{i},
\end{align*}
where $z_{k+1:t}$ denotes tail-averaged iterate error, \[D_i \define g_i -\E\left[ g_i\mid \mathcal{G}_{i-1}\right], \; g_i \define \E[ \norm{z_{k+1:t}}_{2} \mid \mathcal{G}_{i}], \textrm{ and}\]  $\mathcal{G}_{i}$ denotes the sigma-field generated by random variables $\{w_{t}, \; t\leq i\}$ for $t, \, i \; \in \mathbb{Z}^{+}.$ 

Let $h_{i}(w) \define r_{i}\phi_{i}-\M_{i}w$ denote random innovation at time $i$ for $w_{i} = w$. If we show that functions $g_{i}$ are $L_{i}$ Lipschitz continuous in the random innovation $h_{i}$ at time $i$,
then we can see that the martingale difference $D_{i}$ is a $L_{i}$ Lipschitz function of the $i$th random innovation. 

Let $\Omega_{j}^{i}(w)$ represent the iterate value at time $j$, evolving according to \eqref{eq:TDupdatewithprojection}, starting from the value of $w$ at time $i$. Let $w$ and $w'$ be two different iterate values at time $i$, dependent on $h$ and $h'$, respectively, as $w = w_{i-1} + \lr h$ and $w' = w_{i-1} + \lr h'$. We compute the difference between the iterate values at time $j$ when the initial values at time $i$ are $w$ and $w'$ as follows:
\begin{align*}
\Omega_{j}^{i}(w) - \Omega_{j}^{i}(w') &=\Omega_{j-1}^{i}(w) - \Omega_{j-1}^{i}(w') - \lr [h_{j}(\Omega_{j-1}^{i}(w)) - h_{j}(\Omega_{j-1}^{i}(w'))]\\
&= \Omega_{j-1}^{i}(w) - \Omega^{i}_{j-1}(w') - \lr \M_{j} (\Omega^{i}_{j-1}(w) - \Omega^{i}_{j-1}(w'))\\
&= (\I - \lr \M_{j})(\Omega_{j-1}^{i}(w) - \Omega_{j-1}^{i}(w')). \label{appendix:EvolDiff} \numberthis
\end{align*}
Taking expectation and since the projection $\Gamma$ is non-expansive, we have the following
\begin{align*}
\E\psq{\norm{\Omega^{i}_{j}(w) - \Omega^{i}_{j}(w')}_{2}} &= \E\left[\E\left[ \norm{ \Omega^{i}_{j}(w) - \Omega^{i}_{j}(w')}_{2}  \;\middle|\;  \mathcal{G}_{j-1}\right]\right]\\
&= \E\left[ \E\left[ \norm{ (\I - \lr M_j)(\Omega_{j-1}^{i}(w) - \Omega_{j-1}^{i}(w'))}_{2} \;\middle|\; \mathcal{G}_{j-1}\right]\right]\\
&\stackrel{(i)}{\leq} \left(1-\frac{ \lr\mu}{2}\right) \E[\norm{\Omega^{i}_{j-1}(w) - \Omega^{i}_{j-1}(w')}_{2}]\\
&\stackrel{(ii)}{=} \left(1 - \frac{\lr\mu}{2}\right)^{j-i+1} \norm{w - w'}_{2}, \\
&\stackrel{(iii)}{\leq} \lr \left(1 - \frac{\lr \mu}{2}\right)^{j-i+1} \norm{h - h'}_{2}.\numberthis \label{eq:EvolUsefulResult}
\end{align*}
where (i) follows from \Cref{lem:I-gammaMt-support}; (ii) follows from repeated application of (i); and (iii) follows from substituting $w$ and $w'$. 

Let $\Omega^{i}_{t}(w)$ to be the value of the iterate at time $t$, where $t$ ranges from the tail index $k+1$ to $k+N$. The iterate evolves according to \eqref{eq:Mt-update} beginning from $w$ at time $i=k+1$. 
Next, we define 
\begin{align}
\tilde{\Omega}^{i}_{k+1:t}(\tilde{w}, w ) &\define\frac{(i-k)\tilde{w}}{N} + \frac{1}{N}\sum_{j = i+1}^{i+N}\Omega^{i}_{j}(w),\label{eq:s1}
\end{align}
where $\tilde w$ is the value of the tail averaged iterate at time $i$. In the above, $\tilde{\Omega}^{i}_{k+1:t}(\tilde{w}, w )$ denotes the value of tail-averaged iterate at time $t$. 

From \eqref{eq:s1} and using the triangle inequality, we have
\begin{align*}
\E\left[\norm{\tilde{\Omega}^{i}_{k+1:t}(\tilde{w}, w) -  \tilde{\Omega}^{i}_{k+1:t}(\tilde{w}, w') }_{2}\right] & \leq
\E\left[ \frac{1}{N}\sum_{j = i+1}^{i+N}\norm{ (\Omega^{i}_{j}(w) - \Omega^{i}_{j}(w'))}_{2}\right].\numberthis\label{eq:interm}
\end{align*}
Using \eqref{eq:EvolUsefulResult}, we bound the term $\Omega^{i}_{j}(w) - \Omega^{i}_{j}(w')$ inside the summation of \eqref{eq:interm}.
\begin{align*}
\E\left[\norm{\tilde{\Omega}^{i}_{k+1:t}(\tilde{w}, w) -  \tilde{\Omega}^{i}_{k+1:t}(\tilde{w}, w') }_{2}\right] 
&\leq \frac{\lr}{N}  \sum_{j=i+1}^{i+N}\left(1 - \frac{\lr \mu}{2}\right)^{j-i+1} \norm{h - h'}_{2}. \numberthis\label{eq:lip-up}
\end{align*}
Taking into account the bounds on features, rewards, and the projection assumption (\Crefrange{asm:bddFeatures}{asm:projection}), along with the upper bound on $\sigma$ from \eqref{eq:upperboundforsigmasq}, we derive a uniform upper bound $\tau$ on $\norm{h_{i}(w)}$ for all $i$ as:
\begin{samepage}
\begin{align*}
 \tau & =\bigg( 2  R_{\mathsf{max}}^{2} \p{(\fivmax)^{2} + R_{\mx}^{2} (\fiumax)^{2}} \\ & \qquad + 2 \left( (\fivmax)^{4} \p{1+\disc}^{2}+  (\fiumax)^{4} \p{1+\disc^{2}}^{2}+4 \disc^{2} R_{\mx}^{2}(\fivmax)^{2}(\fiumax)^{2} \right)H^{2} \bigg)^{\frac 1 2}
\end{align*}
\end{samepage}
Now, we use a martingale difference concentration, following \citep[Step 3, Theorem 2]{patilFiniteTimeAnalysis2024} to obtain
\begin{align*}
\mathbb{P}\left(\norm{z_{k+1, N}}_{2} - \E\left[\norm{z_{k+1, N}}_{2}\right] > \epsilon\right) &\leq \exp(-\eta \epsilon) \exp\left( \frac{ \eta^{2} \tau^{2} \sum_{i=k+1}^{k+N}L_{i}^{2}}{2}\right).
\end{align*}
Optimising over $\eta$ in the above inequality yields:
\begin{align} \label{appendix:Step3-highprob}
\mathbb{P}\left(\norm{z_{k+1:t}}_{2} - \E\left[\norm{z_{k+1:t}}_{2}\right] > \epsilon \right) &\leq \exp\left(-\frac{\epsilon^{2}}{ \tau^{2} \sum_{i = k+1}^{k+N} L^{2}_{i}}\right).
\end{align}

Using \citep[Lemma 13]{patilFiniteTimeAnalysis2024}, we obtain the following bound on the Lipschitz constant,   
\begin{align} \label{appendix:high-prob-lipschitzbound}
\sum_{i=k+1}^{k+N} L_{i}^{2} &\leq \frac{4}{N\mu^2}. 
\end{align}
By applying \eqref{appendix:high-prob-lipschitzbound} in \eqref{appendix:Step3-highprob}, we obtain
\begin{align*}
\mathbb{P}\left(\norm{z_{k+1:t}}_{2} - \E[\norm{z_{k+1:t}}_{2}]> \epsilon \right) 
&{\leq} \exp\left(-\frac{N\mu^2 \epsilon^2}{4\tau^2}\right)\numberthis\label{eq:crude_hprob},
\end{align*}
For any $\delta \in (0,1]$ the inequality \eqref{eq:crude_hprob} can be expressed in high-confidence form as:
\begin{align} \label{appendix:high-confd-version}
\mathbb{P}\left(\norm{z_{k+1:t}}_{2} - \E[\norm{z_{k+1:t}}_{2}]\leq \frac{2\tau}{\mu\sqrt{N}}\sqrt{\log\left(\frac{1}{\delta}\right)} \right) \geq  1-\delta.
\end{align}
The final bound follows by substituting the bound on $\E\left[\norm{z_{k+1:t}}_{2}\right]$, obtained via Jensen's inequality from \Cref{thm:expectation_bound-tailavg}, into \eqref{appendix:high-confd-version}.
\end{proof}
\subsection{Bounds for mean-variance  TD with regularization}

\begin{theorem}[Restatement of \Cref{thm:main-regtd-high-prob}]\label{thm:regtd-high-prob} 
Suppose \Crefrange{asm:stationary}{asm:projection} hold. Run the regularized version of \Cref{alg:TD-Critic-main}, specified by \eqref{eq:Mt-update-reg}, for $t$ iterations with a step size $\lreg$ as specified in \Cref{thm:regtd-expectation_bound}. Then, for any  $\delta \in (0,1]$, we have the following bound for the projected tail-averaged regularized TD iterate:
\begin{align*}
\mathbb{P}\Bigg(\norm{\check w_{k+1:t} \!-\! \bar{w}_{\mathsf{reg}}}_{2} 
 & \leq \frac{2\check \tau}{\p{2\mu+\zeta}\sqrt{N}}\sqrt{\log\p{\frac{1}{\delta}}} \!+\! \frac{4 \exp \p{-k \lreg \p{2\mu \!+\! \zeta}}}{\check \lr \p{2\mu+\zeta} N} \E \norm{w_{0}-\bar{w}_{\mathsf{reg}}}_{2} \\ & \quad \!+\! \frac{4 \check \tau}{\p{2\mu \!+\! \zeta}\sqrt{N}}\Bigg) \geq 1 - \delta,
 \end{align*}
 where $N, \check w_0,\bar{w}_{\mathsf{reg}},\mu.\textrm{ are} $ as specified in Theorem \ref{thm:regtd-expectation_bound} and
 \begin{align*}
     \check \tau &\!=\! \Big(2  R_{\mathsf{max}}^{2} \p{(\fivmax)^{2} \!+\! R_{\mx}^{2} (\fiumax)^{2}} \\ & \quad \!+\! 4\big(\zeta^2+ (\fivmax)^{4} \p{1+\disc}^{2} \!+\! (\fiumax)^{4} \p{1+\disc^{2}}^{2} \!+\! 4 \beta^{2} R_{\mx}^{2}(\fivmax)^{2}(\fiumax)^{2} \big) H^{2} \Big)^{\frac1{2}}. 
\end{align*} 
\end{theorem}

The proof for the regularized case follows from arguments similar to those in the proof of \Cref{thm:high-prob-bound}, with the modifications outlined below.
\begin{proof}
Let $\check \Omega_{j}^{i}(\check w)$ represent the iterate value at time $j$, evolving following \eqref{eq:TDupdatewithprojection}, starting from the value of $\check w$ at time $i$. 
We compute the difference between the iterate values at time $j$ when the initial values at time $i$ are $\check w$ and $\check w'$, respectively. Let $\check w$ and $\check w'$ be two different parameter values at time $i$ which depend on $\check h$ and $\check h'$ as $\check w = \check w_{i-1} + \lreg \check h$, and $\check w' = \check w_{i-1} + \lreg h'$. We obtain the difference as:
\begin{align*}
\check \Omega_{j}^{i}(\check w) - \check \Omega_{j}^{i}(\check w') &= \check \Omega_{j-1}^{i}(\check w) - \check \Omega_{j-1}^{i}(\check w') - \lreg [\check h_{j}(\check \Omega_{j-1}^{i}(\check w)) - \check h_{j}(\check \Omega_{j-1}^{i}(\check w'))]\\
&= (\I - \lreg (\zeta \I+\M_{j}))(\check \Omega_{j-1}^{i}(\check w) - \check \Omega_{j-1}^{i}(\check w')). \label{appendix:EvolDiff-reg} \numberthis
\end{align*}
Taking expectation and since the projection $\Gamma$ is non-expansive, we have the following
\begin{align*}
\E\psq{\norm{\check \Omega^{i}_{j}(\check w) - \check \Omega^{i}_{j}(\check w')}_{2}} &= \E\left[\E\left[ \norm{ \check \Omega^{i}_{j}(\check w) - \check \Omega^{i}_{j}(\check w')}_{2}  \;\middle|\;  \mathcal{\check G}_{j-1}\right]\right]\\
&= \E\left[ \E\left[ \norm{ (\I - \lreg \M_j)(\check \Omega_{j-1}^{i}(\check w) - \check \Omega_{j-1}^{i}(\check w'))}_{2} \;\middle|\;   \mathcal{\check G}_{j-1}\right]\right]\\
&\stackrel{(i)}{\leq} \left(1-\frac{ \lreg(2\mu+\zeta)}{2}\right) \E\left[\norm{\check \Omega^{i}_{j-1}(w) - \check \Omega^{i}_{j-1}(\check w')}_{2}\right]\\
&\stackrel{(ii)}{=} \left(1 - \frac{\lreg(2\mu+\zeta)}{2}\right)^{j-i+1} \norm{\check w - \check w'}_{2}, \\
&\leq \lreg \left(1 - \frac{\lreg(2\mu+\zeta)}{2}\right)^{j-i+1} \norm{\check h - \check h'}_{2}.\numberthis \label{eq:EvolUsefulResult-reg}
\end{align*}
where (i) follows by \Cref{lem:regtd-I-gammasupport}; (ii) follows by repeated application of (i); and \eqref{eq:EvolUsefulResult-reg} follows from substituting the values of $w$ and $w'$. 

Let $\check \Omega^{i}_{t}(\check w)$ be the value of the iterate at time t where t ranges from the tail index $k+1$ to $k+N$. The iterate evolves according to \eqref{eq:Mt-update-reg} starting at the value $\check w$ at time $i=k+1$. 
Next, we define
\begin{align}
\bar{\Omega}^{i}_{k+1:t}(\hat{w}, \check w ) &\define\frac{(i-k)\tilde{\check w}}{N} + \frac{1}{N}\sum_{j = i+1}^{i+N}\check \Omega^{i}_{j}(\check w),\label{eq:s1-reg}
\end{align}
where $\hat{w}$ is the value of the tail-averaged iterate at time $i$.

Now, we prove that Lipschitz continuity in the random innovation $\check h_i$ at time $i$ with constant $\check L_{i}$. 
\begin{align*}
\E\left[\norm{\tilde{\check \Omega}^{i}_{k+1:t}(\tilde{\check w}, \check w) -  \tilde{\check \Omega}^{i}_{k+1:t}(\tilde{\check w}, \check w') }_{2}\right] &=
\E\left[ \frac{1}{N}\sum_{j = i+1}^{i+N}\norm{ (\check \Omega^{i}_{j}(\check w) - \check \Omega^{i}_{j}(\check w'))}_{2}\right].\numberthis\label{eq:interm-reg}
\end{align*}
Using \eqref{eq:EvolUsefulResult-reg}, we bound the difference $\norm{\check \Omega^{i}_{j}(\check w) - \check \Omega^{i}_{j}(\check w')}$ in \eqref{eq:interm-reg}.
\begin{align*}
\E\left[\norm{\tilde{\Omega}^{i}_{k+1:t}(\tilde{w}, w) -  \tilde{\Omega}^{i}_{k+1:t}(\tilde{w}, w') }_{2}\right] 
&\leq \frac{\lr}{N}  \sum_{j=i+1}^{i+N}\left(1 - \frac{\lreg(2\mu+\zeta)}{2}\right)^{j-i+1} \norm{\check h - \check h'}_{2}. \numberthis\label{eq:lip-up-reg}
\end{align*}
Considering the bounds on features, rewards, and the projection assumption (\Crefrange{asm:bddFeatures}{asm:projection}), along with a bound on $\check{\sigma}$ in \eqref{eq:sigmacheck}, we  find an upper bound $\check \tau$ on $\norm{\check h_{i}(\check w_{i})}$ as follows:
\begin{align*}
 \check \tau &\!=\! \Big(2  R_{\mathsf{max}}^{2} \p{(\fivmax)^{2} \!+\! R_{\mx}^{2} (\fiumax)^{2}} \\ & \quad + 4\big(\zeta^2+ (\fivmax)^{4} \p{1+\disc}^{2} \!+\! (\fiumax)^{4} \p{1+\disc^{2}}^{2} \!+\! 4 \beta^{2} R_{\mx}^{2}(\fivmax)^{2}(\fiumax)^{2} \big) H^{2} \Big)^{\frac1{2}}.
\end{align*} 
Using \citep[Lemma 20]{patilFiniteTimeAnalysis2024}, we obtain the following bound on the Lipschitz constant,   
\begin{align} \label{appendix:high-prob-lipschitzbound-reg}
\sum_{i=k+1}^{k+N} \check {L}_{i}^{2} &\leq \frac{4}{N(2\mu+\zeta)^2}. 
\end{align}
The rest of the proof follows by making parallel arguments to those in \Cref{sec:Appendix-highprob}.
\end{proof}
\section{Outline of Actor Analysis} \label{actor:proof-sketch}
\begin{proof}\textit{(Sketch)}
\begin{figure}[H]
    \centering
    \begin{tikzpicture}[
        scale=0.78,  
        transform shape,
        node distance=1cm,
        existing/.style={rectangle, draw=blue!80!black, thick, minimum width=1.5cm, text width=3.5cm,
        align=center, rounded corners=0.25cm},
        novel/.style={ellipse, draw=green!85!black, thick, text width=2.25cm, align=center},
        final/.style={double, circle, draw=green!85!black, thick, text width=2.5cm,  align=center},
        annotation/.style={font=\footnotesize, text width=2cm, align=left,
        thick,
        minimum width=2.5cm,
        minimum height=2cm,
        font=\small,
        align=center
        }
    ]
   
        \node (gradU) [existing] {Policy Gradient for the square-value Function {\footnotesize \citep[Lemma 2]{l.a.VarianceconstrainedActorcriticAlgorithms2016}}};
        \node (lemma4b) [novel, right=of gradU] {Smoothness of square-value Function {\footnotesize(\Cref{lemma:Lipschitz})}};
        \node (lemma4a) [existing, below=of lemma4b] {Smoothness of value function \\ {\footnotesize \citep[Proposition 1]{xuImprovingSampleComplexity2021}}};
        \node (lemmaG1) [existing, left= of lemma4a] { Smoothness of state-action visitation distribution \\ {\footnotesize \citep[Lemma 3]{xuImprovingSampleComplexity2021}}};
        \node (lagrange) [novel, right=of lemma4b] {Smoothness of Lagrangian {\footnotesize(\Cref{lemma:LipschitzGradL})}};
        \node (SPSA) [existing, right= of lemma4a] {Gradient Estimation using Perturbation technique (SPSA) \\ \footnotesize{\citep{spallMultivariateStochasticApproximation1992}}};
        \coordinate (midpoint) at ($(lagrange)!0.5!(SPSA)$);
        \node (theorem) [final, right=2.5cm of midpoint] {Convergence to $\epsilon$-stationary point  {\footnotesize({\Cref{thm:actor}})}};
        \draw[-{Stealth[length=2mm]}, thick, dashed] (gradU) -- (lemma4b);
        \draw[-{Stealth[length=2mm]}, thick] (lemmaG1) -- (lemma4a);
        \draw[-{Stealth[length=2mm]}, thick, dashed] (lemma4a) -- (lagrange);
        \draw[-{Stealth[length=2mm]}, thick, dashed] (lemma4b) -- (lagrange);
        \draw[-{Stealth[length=2mm]}, thick, dashed] (lagrange) -- (theorem);
        \draw[-{Stealth[length=2mm]}, thick, dashed] (SPSA) -- (theorem);
        
    \end{tikzpicture}
    \caption{Logical dependency graph for proving \Cref{thm:actor}. Rectangular nodes (blue) represent established results from prior work, elliptical nodes (green) denote our novel contributions, and dashed lines illustrate the logical dependencies we establish to derive the final result (green circle).}
    \label{fig:proof-sketch}
\end{figure}
As visualized in \Cref{fig:proof-sketch}, the proof begins by establishing the smoothness of the policy gradient for the square-value function:  
\begin{equation}  
    \nabla U(\theta) \!=\! \tfrac{1}{1 - \disc^2} \big( \underbrace{ \!\textstyle \sum_{s,a}\! \tilde{\nu}_{\theta}(s, a) \nabla \log \pi_\theta(a|s) W_\theta(s, a)}_{T_1(\theta)} + 2\disc  \underbrace{\textstyle \!\sum_{s,a,s'}\! \tilde{\nu}_{\theta}(s, a) P(s' | s, a) \nabla V_{\theta}(s')}_{T_2(\theta)} \big).\label{eq:GradUmain}  
\end{equation}  

We decompose the expression in \eqref{eq:GradUmain} into \( T_1(\theta) \) and \( T_2(\theta) \).  \( T_1(\theta) \) consists of three terms: the state-action visitation distribution, the score function, and the square-value function.  
To derive a smoothness constant for \( T_1(\theta) \), we leverage the following:  
(i) the smoothness result for the state-action visitation distribution (\Cref{lemma:xu2021}), as stated in \citep[Lemma 3]{xuImprovingSampleComplexity2021};  
(ii) the boundedness and smoothness of the policy (\Cref{asm:scorefnandpolicy}). 

\( T_2(\theta) \) is the product of the state-action visitation distribution and the policy gradient of the value function.  
To establish the smoothness constant for \( T_2(\theta) \), we utilize the smoothness result for the value function from \citep[Proposition 1]{xuImprovingSampleComplexity2021}.

Combining the results for \( T_1(\theta) \) and \( T_2(\theta) \), we derive the smoothness constants for the square-value function.  
By decomposing the terms in the Lagrangian into the gradients of the value function and the square-value function and carefully bounding the gradient norms, we obtain the smoothness constant \( L \) in \eqref{eq:LforLagrangian} for the Lagrangian.

After establishing the smoothness of the Lagrangian, the remainder of the proof largely follows a standard SGD analysis framework \citep{ghadimiStochasticFirstZerothorder2013, kumarSampleComplexityActorcritic2023}. However, key modifications are necessary to accommodate SPSA-based gradient estimates, particularly in handling and optimizing the perturbation parameter \( p_t \) and the critic batch size \( m \).

As $\nabla L(\theta_{t})$ is $L$-Lipschitz (\Cref{lemma:LipschitzGradL}), we have $$L(\theta_{t+1}) \geq L(\theta_{t}) + \langle \nabla L(\theta_{t}), \theta_{t+1} - \theta_{t} \rangle - \frac{L \alpha_{t}^{2}}{2} \| \nabla \hat{L}(\theta_{t}) \|^{2}$$
In the above, $\nabla \hat{L}(\theta_{t})$ is an SPSA gradient estimate. 

Taking the expectation with respect to the sigma field \(\mathcal{F}_t = \sigma(\theta_k, k \leq t)\), denoted by \(\mathbb{E}_t\), we have
\begin{align*}
 \mathbb{E}_{t} [L(\theta_{t+1})] & \geq \mathbb{E}_{t} [L(\theta_{t})] + \alpha_{t} \mathbb{E}_{t} \left[ \| \nabla L(\theta_{t}) \|^{2} \right] \\ & \quad - \alpha_{t} K_{1} \left( 1 + \frac{2 \lambda R_{\text{max}}}{1 - \disc} \right) \underbrace{\left\| \mathbb{E}_{t} \left[ \nabla \hat{J}(\theta_{t}) - \nabla J(\theta_{t}) \right] \right\|}_{\text{(A)}} \\
& \quad - \lambda \alpha_{t} K_{1} \underbrace{\left\| \mathbb{E}_{t} \left[ \nabla \hat{U}(\theta_{t}) - \nabla U(\theta_{t}) \right] \right\| }_{\text{(B)}} - \frac{L}{2} \alpha_{t}^{2} \underbrace{\mathbb{E}_{t} \left[ \| \nabla \hat{L}(\theta_{t}) \|^{2} \right]}_{\text{(C)}}.
\end{align*}
Now, substituting the bounds for the biased SPSA gradient estimates---(A) from \eqref{eq:Bound(A)}, (B) from \eqref{eq:Bound(B)}, and (C) from \eqref{eq:Bound(C)}---into the above equation, we obtain
\begin{align*}
 \mathbb{E}_{t} [L(\theta_{t+1})] & \geq \mathbb{E}_{t} [L(\theta_t)] + \alpha_t \mathbb{E}_{t} \left[ \|\nabla L(\theta_t) \|^2 \right] \\ & \quad - \alpha_t K_1 \left( 1 + \frac{2 \lambda R_{\text{max}}}{1 - \disc} \right) \left( \frac{d^{\frac 32} L_J p_t}{2} + \frac{d^{\frac 12} \fivmax K_2}{p_t \sqrt{m}} \right) \\
& \quad - \lambda \alpha_t K_1 \left( \frac{d^{\frac 32} L_U p_t}{2} + \frac{d^{\frac 12} \fiumax K_2}{p_t \sqrt{m}} \right) - \frac{L \alpha_t^2}{2} \left( \frac{K_3}{p_t^2} \right).
\end{align*}
\noindent Summing from \( t = 1 \) to \( n \) and dividing both sides by \( n \), and setting \( \alpha_t = \alpha \) and \( p_t = p \), we get
\begin{align*}
\frac{1}{n} \sum_{t=1}^{n} \mathbb{E} \left[ \|\nabla L(\theta_t)\|^2 \right] 
& \leq \frac{C_1}{n \alpha} + C_2 p + \frac{C_3}{\sqrt{m} p} + \frac{C_4 \alpha}{p^2}.
\end{align*}
Setting \(\alpha = n^a, \, p = n^b, \, m = n^c\), we have
\begin{align*}
\mathbb{E} \left[ \|\nabla L(\theta_R)\|^2 \right] 
& \leq C_1 n^{-1-a} + C_2 n^b + C_3 n^{-b-c/2} + C_4 n^{a-2b}.
\end{align*}
Optimizing for \(a, b, c\), we find their values to be \(a = -\frac{3}{4}, \, b = -\frac{1}{4}, \, c = 1\). Substituting these values, we get
\begin{align*}
\mathbb{E} \left[ \|\nabla L(\theta_R)\|^2 \right] 
& \leq C_1 n^{-1/4} + C_2 n^{-1/4} + C_3 n^{-1/4} + C_4 n^{-1/4} \\
& = O(n^{-1/4}).
\end{align*}

\end{proof}
\section{Proofs for the claims in Section \ref{sec:actor}} \label{appendix: spsa-based-actor}

Before we prove the claims, we state a few useful supporting lemmas in the analysis. 

\begin{lemma}[Restatement of Lemma 3 \citep{xuImprovingSampleComplexity2021}]
\label{lemma:xu2021}
    Consider the initialization distribution $\eta(\cdot)$ and the transition kernel $\mathsf{P}(\cdot|s,a)$. Let $\eta(\cdot) = \zeta(\cdot)$ or $\eta(\cdot) = \mathsf{P}(\cdot|\hat{s},\hat{a})$ for any given $(\hat{s},\hat{a}) \in \mathcal{S} \times \mathcal{A}$. Denote $\nu_{\pi_\theta,\eta}(\cdot,\cdot)$ as the state-action visitation distribution of the MDP with policy $\pi_\theta$ and initialization distribution $\eta(\cdot)$. Suppose the Assumption holds. Then, we have
    \begin{flalign*}
        \norm{\nu_{\pi_{\theta_1},\eta}(\cdot,\cdot) - \nu_{\pi_{\theta_2},\eta}(\cdot,\cdot)}_{\text{TV}} 
        &\leq C_\nu \norm{\theta_1 - \theta_2}_{2},
    \end{flalign*}
    for all $\theta_1, \theta_2 \in \mathbb{R}^d$, where $C_\nu = C_\pi \left( 1 + \lceil \log_\rho \kappa^{-1} \rceil + \frac{1}{1-\rho} \right)$.
\end{lemma}

\subsection{Proof of Lemma \ref{lemma:Lipschitz}}
\begin{proof}
The first claim concerning the smoothness of $J(\cdot)$
can be inferred from \citep[Proposition 1]{xuImprovingSampleComplexity2021}.

We prove the smoothness of the square-value function below.

From \citep[Lemma 1]{l.a.VarianceconstrainedActorcriticAlgorithms2016}, we have 
\begin{align}
    \nabla U(\theta) = \frac{1}{1 - \disc^2} \Bigg( & \underbrace{ \sum_{s,a} \tilde{\nu}_{\theta}(s, a) \nabla \log \pi_\theta(a|s) W_\theta(s, a)}_{T_1(\theta)} + 2\disc  \underbrace{\sum_{s,a,s'} \tilde{\nu}_{\theta}(s, a) P(s' | s, a) \nabla V_{\theta}(s')}_{T_2(\theta)} \Bigg), \label{eq:GradU}
\end{align}
where
$$W_{\theta}(s,a) = \mathbb{E} \left[ \left( \sum_{k=0}^{\infty} \disc^{k} r_{t+k} \right)^2 \; \middle| \; s_{t} = s, a_{t} = a \right]$$ 
and $\tilde{\nu}_{\theta}(s, a) = (1 - \disc^2) \sum_{t=0}^{\infty} \disc^{2t} \mathbb{P}(s_t = s, a_t = a)$ is the $\disc^2$-discounted state-action visitation distribution, with
$\mathbb{P}(s_t = s, a_t = a) = \mathbb{P}(s_t = s | s_0 = s) \pi_{\theta}(a | s) $. \begin{align*}
    \|\nabla U(\theta_1)-\nabla U(\theta_2)\|_2 \leq \frac{1}{1-\disc^2} \p{\pnorm{T_1(\theta_1)-T_1(\theta_2)}+ 2\disc \pnorm{T_2(\theta_1)-T_2(\theta_2)}} \numberthis \label{eq:GradU-T1T2}
\end{align*}
We now show that $T_1(\theta)$, defined in \eqref{eq:GradU} is Lipschitz in $\theta$.
\begin{align*}
& \|T_1(\theta_1) - T_1(\theta_2)\|_2 \\ & = \Bigg\| \sum_{s,a} \underbrace{\tilde{\nu}_{\theta_1}(s, a)}_{a_{1}} \underbrace{\nabla \log \pi_{\theta_1}(a|s)}_{b_{1}} \underbrace{W_{\pi_{\theta_1}}(s, a)}_{c_{1}}  - \sum_{s,a} \underbrace{\tilde{\nu}_{\theta_2}(s, a)}_{a_{2}} \underbrace{\nabla \log \pi_{\theta_2}(a|s)}_{b_{2}}\underbrace{ W_{\pi_{\theta_2}}(s, a)}_{c_{2}} \Bigg\|_2 \\
&= \left\|\sum_{s,a} (a_1 b_1 c_1 - a_2 b_2 c_2)\right\| \\
&= \left\|\sum_{s,a} a_1 b_1 c_1 - a_2 b_2 c_2 + a_2 b_2 c_1 - a_2 b_2 c_1\right\|\\
&= \left\|\sum_{s,a} c_1(a_1 b_1 - a_2 b_2) + a_2 b_2(c_1 - c_2)\right\| \\
&= \left\|\sum_{s,a} c_1(a_1 b_1 - a_2 b_2 + a_1 b_2 - a_1 b_2) + a_2 b_2(c_1 - c_2)\right\|\\
&= \left\|\sum_{s,a} c_1(a_1(b_1 - b_2) + b_2(a_1 - a_2)) + a_2 b_2(c_1 - c_2)\right\| \\
&\leq \sum_{s,a} \Big| W_{\theta_1}(s,a) \Big| \Big| \tilde{\nu}_{\theta_1}(s,a) \Big| \|\nabla \log \pi_{\theta_1}(a|s) - \nabla \log \pi_{\theta_2}(a|s)\|_2 \\
&\quad + \sum_{s,a} \Big| W_{\theta_1}(s,a) \Big| \|\nabla \log \pi_{\theta_1}(a|s)\|_2 \Big| \tilde{\nu}_{\theta_1}(s,a) - \tilde{\nu}_{\theta_2}(s,a) \Big| \\
&\quad + \sum_{s,a} \tilde{\nu}_{\theta_2}(s,a) \|\nabla \log \pi_{\theta_2}(a|s)\|_2 \Big| W_{\theta_1}(s,a) - W_{\theta_2}(s,a) \Big| \\ 
& \stackrel{(a)}{\leq} \frac{R_{\text{max}}}{(1 - \disc)^{2}} \sum_{s,a} \| \nabla \log \pi_{\theta_{1}}(a|s) - \nabla \log \pi_{\theta_{2}}(a|s) \|_{2} + \frac{C_{\psi} R_{\text{max}}}{(1-\disc)^2} \sum_{s,a} | \tilde{\nu}_{\theta_{1}}(s,a) - \tilde{\nu}_{\theta_{2}}(s,a) | \\
& \quad + C_{\psi} \sum_{s,a} \Big |W_{\theta_{1}}(s,a) - W_{\theta_{2}}(s,a) \Big| \tilde{\nu}_{\theta_{2}}(s,a) \\
& \stackrel{(b)}{\leq} \frac{R_{\text{max}} L_{\psi}}{(1 - \disc)^{2}} \| \theta_{1} - \theta_{2} \|_{2}+ \frac{2 R_{\text{max}}  C_{\psi}  C_{\nu}}{(1-\disc)^2} \| \theta_{1} - \theta_{2} \|_{2} + C_{\psi} \sum_{s,a} | W_{\theta_{1}}(s,a) - W_{\theta_{2}}(s,a) | \tilde{\nu}_{\theta_{2}}(s,a) \\ 
& \stackrel{(c)}{\leq}  \frac{R_{\text{max}} L_{\psi}}{(1 - \disc)^{2}} \| \theta_{1} - \theta_{2} \|_{2}  +  \frac{2 R_{\text{max}} C_{\psi} C_{v}}{(1 - \disc)^{2}} \| \theta_{1} - \theta_{2} \|_{2} + \frac{2R_{\text{max}} C_{\psi} C_{v}}{(1 - \disc)^{2}} \| \theta_{1} - \theta_{2} \|_{2} \\
& \leq  \frac{R_{\text{max}} L_{\psi}}{(1 - \disc)^{2}} \| \theta_{1} - \theta_{2} \|_{2}  +  \frac{4R_{\text{max}} C_{\psi} C_{v}}{(1 - \disc)^{2}} \| \theta_{1} - \theta_{2} \|_{2}, \numberthis \label{eq:T1-lipschitz}
\end{align*}
where (a) follows by $| W_{\theta}(s,a)||\tilde{\nu}_{\theta_{1}}(s,a)| \leq \frac{R_{\text{max}}}{(1 - \disc)^{2}} \; \text{for any } \theta \in \R^d$ and by the upper bound $C_\psi$ on the score function, see \Cref{asm:scorefnandpolicy}; (b) follows by smoothness of the policy (\Cref{asm:scorefnandpolicy}) and $C_{\nu}$- Lipschitzness of $\tilde{\nu}(s,a)$  \citep[see][Lemma 3]{xuImprovingSampleComplexity2021}; (c) follows by employing similar arguments for the square-value function, in place of the value function in \citep[Lemma 4]{xuImprovingSampleComplexity2021}, as outlined below:
\begin{align*}
C_{\psi} \sum_{s,a} | W^{\pi}_{\theta_{1}}(s,a) - W^{\pi}_{\theta_{2}}(s,a) | \tilde{\nu}_{\theta_{2}}(s,a) & \leq C_{\psi} \frac{R_{\text{max}}}{(1 - \disc)^{2}} \| P^{\pi}_{\theta_{1}}(\cdot, \cdot) - P^{\pi}_{\theta_{2}}(\cdot, \cdot) \|_{TV}  \\
&  \leq \frac{2R_{\text{max}} C_{\psi} C_{v}}{(1 - \disc)^{2}} \| \theta_{1} - \theta_{2} \|_{2}. 
\end{align*}
Next, we obtain the Lipschitz constant for 
$T_{2}(\theta) =\sum_{s,a,s'} \tilde{\nu}_{\theta}(s,a) P(s'|s,a) \nabla V_{\theta}(s')$ below. The Lipschitzness of $T_2(\theta)$ together with that of $T_1(\theta)$ would lead to smoothness of $U(\cdot)$, from  \eqref{eq:GradU}.
\begin{align*}
&  \| T_{2}(\theta_{1}) - T_{2}(\theta_{2}) \|_{2}  \\ & \leq \Bigg\|\sum_{s,a,s'} \tilde{\nu}_{\theta_{1}}(s,a) P(s'|s,a) \nabla V_{\theta_{1}}(s') - \sum_{s,a,s'} \tilde{\nu}_{\theta_{2}}(s,a) P(s'|s,a) \nabla V_{\theta_{2}}(s')\Bigg\| \\
& \leq \Bigg\|\sum_{s,a,s'} \tilde{\nu}_{\theta_{1}}(s,a) P(s'|s,a) \nabla V_{\theta_{1}}(s') - \sum_{s,a,s'} \tilde{\nu}_{\theta_{2}}(s,a) P(s'|s,a) \nabla V_{\theta_{2}}(s') \\
& \qquad + \sum_{s,a,s'} \tilde{\nu}_{\theta_{2}}(s,a) P(s'|s,a) \nabla V_{\theta_{1}}(s') - \sum_{s,a,s'} \tilde{\nu}_{\theta_{2}}(s,a) P(s'|s,a) \nabla V_{\theta_{2}}(s') \Bigg\|  \\
& {\leq} \sum_{s,a,s'} P(s'|s,a) \| \nabla V_{\theta_{1}}(s') \|_{2} \| \tilde{\nu}_{\theta_{1}}(s,a) - \tilde{\nu}_{\theta_{2}}(s,a) \| \\
& \quad + \sum_{s,a,s'} P(s'|s,a) \tilde{\nu}_{\theta_{2}}(s,a) \| \nabla V_{\theta_{1}}(s') - \nabla V_{\theta_{2}}(s') \|_{2} \\
& \stackrel{(a)}{\leq} \frac{2 R_{\text{max}} C_{\psi}}{(1 - \disc)^{2}} \sum_{s,a} \| \tilde{\nu}_{\theta_{1}}(s,a) - \tilde{\nu}_{\theta_{2}}(s,a) \| + \sum_{s,a,s'} P(s'|s,a) \tilde{\nu}_{\theta_{2}}(s,a) \| \nabla V_{\theta_{1}}(s') - \nabla V_{\theta_{2}}(s') \|_{2} \\
& \stackrel{(b)}{\leq} \frac{2 R_{\text{max}} C_{\psi} C_{\nu}}{(1 - \disc)^{2}} \| \theta_{1} - \theta_{2} \|_{2} + 2 L_{J} \| \theta_{1} - \theta_{2} \|_{2} \numberthis \label{eq:T2-lipschitz}
\end{align*}
where (a) follows by $ P(s'|s,a) \| \nabla V_{\theta}(s') \|_{2} \leq \frac{R_{\text{max}} C_{\psi}}{(1 - \disc)^{2}}$; (b) follows from applying \citep[Lemma 3]{xuImprovingSampleComplexity2021}, where $C_\nu=(1/2)C_\pi\left( 1 +  \lceil \log_\rho \kappa^{-1} \rceil + (1-\rho)^{-1} \right)$.

Combining  $T_{1}$ and $T_{2}$ into \eqref{eq:GradU-T1T2},
\begin{align*}
& \| \nabla U(\theta_{1}) - \nabla U(\theta_{2}) \|  \leq L_{U} \| \theta_{1} - \theta_{2} \|_{2}, \text{ where} \\
 & L_{U}  = \frac{1}{1 - \disc^{2}} \left( \frac{R_{\text{max}} L_{\psi}}{(1 - \disc)^{2}} + \frac{4 R_{\text{max}} C_{\psi} C_{v}}{(1 - \disc)^{2}} + \frac{4 \disc R_{\text{max}} C_{\psi} C_{v} + 4 \disc L_{J}}{(1 - \disc)^{2}} \right).
\end{align*}
\end{proof}
\subsection{Proof of Lemma \ref{lemma:LipschitzGradL}}
\begin{proof}
Notice that
\begin{samepage}
    \begin{align*}
    & \pnorm{\nabla L(\theta_1) - \nabla L(\theta_2)} \leq \pnorm{\nabla J(\theta_1) - \nabla J(\theta_2)} + \lambda \pnorm{\nabla U(\theta_1) - \nabla U(\theta_2)} \\
    & \phantom{\hskip 100pt} + 2 \lambda \pnorm{J(\theta_1) \nabla J(\theta_1) - J(\theta_2) \nabla J(\theta_2)} \\
    & \stackrel{(a)}{\leq} L_J\pnorm{\theta_1 - \theta_2} + \lambda L_U\pnorm{\theta_1 -\theta_2} + 2 \lambda \underbrace{\pnorm{J(\theta_1) \nabla J(\theta_1) - J(\theta_2) \nabla J(\theta_2)}}_{\text{(I)}}, \numberthis
    \label{eq:LagragianLipschitz}
    \end{align*}
where (a) follows from \Cref{lemma:Lipschitz}. 
\end{samepage}
We bound (I) as follows: 
\begin{samepage}
\begin{align*}
& \| J(\theta_1) \nabla J(\theta_1) - J(\theta_2) \nabla J(\theta_2)\|_2 \\ & = \| J(\theta_1) \nabla J(\theta_1) - J(\theta_1) \nabla J(\theta_2) + J(\theta_1) \nabla J(\theta_2) - J(\theta_2) \nabla J(\theta_2) \|_2 \\
&\leq |J(\theta_1)| \cdot \| \nabla J(\theta_1) - \nabla J(\theta_2)\|_2 + \|\nabla J(\theta_2)\|_2 \cdot |J(\theta_1) - J(\theta_2)| \\
&\stackrel{(i)}{\leq} \frac{R_{\text{max}} L_J}{1-\disc} \|\theta_1 - \theta_2\|_{2} + \|\nabla J(\theta_2)\|_2 \cdot |J(\theta_1) - J(\theta_2)| \\
&\stackrel{(ii)}{\leq} \frac{R_{\text{max}} L_J}{1-\disc} \|\theta_1 - \theta_2\|_{2} + \frac{R_{\text{max}} C_\psi}{(1-\disc)^2} |J(\theta_1) - J(\theta_2)| \\
&\leq \frac{R_{\text{max}} L_J}{1-\disc} \|\theta_1 - \theta_2\|_{2} +\frac{R_{\text{max}} C_\psi}{(1-\disc)^2} \|\theta_1 - \theta_2\|_2, \label{eq:Term(I)} \numberthis 
\end{align*}
where (i) follows from 
$ \left| J(\theta) \right| \leq \frac{R_{\text{max}}}{1 - \disc}$; (ii) follows from $\|\nabla J(\theta)\|_2 \leq \frac{R_{\text{max}} C_\psi}{(1-\disc)^2} \; \text{ for any } \theta \in~\R^d$, we arrive at this by Policy Gradient Theorem \citep{suttonPolicyGradientMethods1999}, \Cref{asm:scorefnandpolicy} and $ \left| Q_{\pi_\theta}(s,a) \right| \leq~\frac{R_{\text{max}}}{1 - \disc}$;
\eqref{eq:Term(I)} follows from the first order Taylor expansion at $\theta_1$,  mean-value theorem $\exists \; \tilde{\theta} = \lambda \theta_1 \!+\! (1 - \lambda) \theta_2, \ \text{for some } \lambda \in [0, 1]$. \end{samepage}
\[J(\theta_1) = J(\theta_2) + \nabla J(\tilde{\theta})^\top (\theta_1 - \theta_2) \implies |J(\theta_1) - J(\theta_2)| \leq \frac{R_{\text{max}} C_\psi}{(1-\disc)^2} \|\theta_1 - \theta_2\|_2.\]
Now, substituting  \eqref{eq:Term(I)} in \eqref{eq:LagragianLipschitz}, we obtain
\begin{align*}
& \| \nabla L(\theta_{1}) - \nabla L(\theta_{2}) \| \leq \| \nabla J(\theta_{1}) - \nabla J(\theta_{2}) \|+ 2 \lambda \| J(\theta_{1}) \nabla J(\theta_{1}) - J(\theta_{2}) \nabla J(\theta_{2}) \| \\ & \phantom{\hskip 100pt}  + \lambda \| \nabla U(\theta_{1}) - \nabla U(\theta_{2}) \| \\
& \leq L_{J} \| \theta_{1} - \theta_{2} \|_{2} + 2 \lambda \left( \frac{R_{\text{max}} L_{J}}{1 - \disc} + \frac{R_{\text{max}} C_{\psi}}{(1 - \disc)^{2}} \right) \| \theta_{1} - \theta_{2} \|_{2}+ \lambda L_{U} \| \theta_{1} - \theta_{2} \|_{2} \\
& \leq \left( L_{J} + 2 \lambda \left( \frac{R_{\text{max}} L_{J}}{1 - \disc} + \frac{R_{\text{max}} C_{\psi}}{(1 - \disc)^{2}} \right) + \lambda L_{U} \right) \| \theta_{1} - \theta_{2} \|_{2} \\
& \leq L_{o} \| \theta_{1} - \theta_{2} \|_{2}
\end{align*}
Hence, the gradient of the Lagrangian is \( L_{o} \)-Lipschitz with the Lipschitz constant given by  
\begin{equation*}
    L_{o}= L_{J} + 2 \lambda \left( \frac{R_{\text{max}} L_{J}}{1 - \gamma} + \frac{R_{\text{max}} C_{\psi}}{(1 - \gamma)^{2}}\right) + \lambda L_{U}.
\end{equation*}
\end{proof} 
\subsection{Proof of Theorem \ref{thm:actor}}
\begin{proof}
    Notice that as $\nabla L(\theta_{t})$ is $L$-Lipschitz (\Cref{lemma:LipschitzGradL}), we have $$L(\theta_{t+1}) \geq L(\theta_{t}) + \langle \nabla L(\theta_{t}), \theta_{t+1} - \theta_{t} \rangle - \frac{L \alpha_{t}^{2}}{2} \| \nabla \hat{L}(\theta_{t}) \|^{2}$$
Taking expectation w.r.t the sigma field $\F_t= \sigma(\theta_{k},k\le t)$, denoted by $\E_{t}$

\begin{align*}
& \mathbb{E}_{t} [L(\theta_{t+1})] \geq \mathbb{E}_{t} [L(\theta_{t})] + \mathbb{E}_{t} \left[ \left \langle \nabla L(\theta_{t}), \alpha_{t} \nabla L(\theta_{t}) + \alpha_{t} \left( \nabla \hat{L}(\theta_{t}) -\nabla L(\theta_{t}) \right) \right \rangle \right] \\ & \phantom{\hskip 60pt} - \mathbb{E}_{t} \left[ \frac{L}{2} \alpha_{t}^{2} \| \nabla \hat{L}(\theta_{t}) \|^{2} \right] \\
& = \mathbb{E}_{t} [L(\theta_{t})] + \alpha_{t} \mathbb{E}_{t} \left[ \| \nabla L(\theta_{t}) \|^{2} \right] + \alpha_{t} \mathbb{E}_{t} \left[ \nabla L(\theta_{t})^{\top} \left( \nabla \hat{L}(\theta_{t})-\nabla L(\theta_{t}) \right) \right]\\ & \quad - \mathbb{E}_{t} \left[ \frac{L}{2} \alpha_{t}^{2} \| \nabla \hat{L}(\theta_{t}) \|^{2} \right] \\
&\geq \mathbb{E}_{t} [L(\theta_{t})] + \alpha_{t} \mathbb{E}_{t} \left[ \| \nabla L(\theta_{t}) \|^{2} \right] - \alpha_{t} \left| \mathbb{E}_{t} \left[ \nabla L(\theta_{t})^{\top} \left( \nabla \hat{L}(\theta_{t})-\nabla L(\theta_{t}) \right) \right] \right| \\ & \quad - \mathbb{E}_{t} \left[ \frac{L}{2} \alpha_{t}^{2} \| \nabla \hat{L}(\theta_{t}) \|^{2} \right] \\
& \stackrel{(i)}{\geq} \mathbb{E}_{t} [L(\theta_{t})] + \alpha_{t} \mathbb{E}_{t} \left[ \| \nabla L(\theta_{t}) \|^{2} \right] - \alpha_{t}  \norm{\nabla L(\theta_{t})} \norm{ \mathbb{E}_{t} \left[ \nabla \hat{L}(\theta_{t}) - \nabla L(\theta_{t}) \right]}  \\ & \quad - \mathbb{E}_{t} \left[ \frac{L}{2} \alpha_{t}^{2} \| \nabla \hat{L}(\theta_{t}) \|^{2} \right] \\
& \stackrel{(ii)}{\geq} \mathbb{E}_{t} [L(\theta_{t})] + \alpha_{t} \mathbb{E}_{t} \left[ \| \nabla L(\theta_{t}) \|^{2} \right] - \alpha_{t} K_{1} \norm{ \mathbb{E}_{t} \left[ \nabla \hat{L}(\theta_{t}) - \nabla L(\theta_{t}) \right]} \\ & \quad - \frac{L}{2} \alpha_{t}^{2} \mathbb{E}_{t} \left[ \| \nabla \hat{L}(\theta_{t}) \|^{2} \right] \\
& \stackrel{(iii)}{\geq} \mathbb{E}_{t} [L(\theta_{t})] + \alpha_{t} \mathbb{E}_{t} \left[ \| \nabla L(\theta_{t}) \|^{2} \right] - \alpha_{t} K_{1} \norm{\mathbb{E}_{t} \left[  \nabla \hat{J}(\theta_{t}) - \nabla J(\theta_{t})  \right]} \\
& \quad -  \lambda \alpha_{t} K_{1}   \norm{\mathbb{E}_{t} \left[ \nabla \hat{U}(\theta_{t}) - \nabla U(\theta_{t})  \right]}  - 2 \lambda \alpha_{t} K_{1} \norm{\mathbb{E}_{t} \left[  J(\theta_{t}) \nabla J(\theta_{t}) - \hat{{J}}(\theta_{t}) \nabla \hat{J}(\theta_{t}) \right]} \\
& \quad  - \frac{L}{2} \alpha_{t}^{2} \mathbb{E}_{t} \left[ \| \nabla \hat{L}(\theta_{t}) \|^{2} \right] \\ 
& \stackrel{(iv)}{\geq} \mathbb{E}_{t} [L(\theta_{t})] + \alpha_{t} \mathbb{E}_{t} \left[ \| \nabla L(\theta_{t}) \|^{2} \right]  - \alpha_{t} K_{1}  \norm{\mathbb{E}_{t} \left[ \nabla \hat{J}(\theta_{t}) - \nabla J(\theta_{t})  \right]}  \\
& \quad - \lambda \alpha_{t} K_{1}   \left\| \mathbb{E}_{t}\left[ \nabla \hat{U}(\theta_{t}) - \nabla U(\theta_{t})  \right] \right\|  \\ & \quad - 2 \alpha_{t} K_{1} \lambda  \left\| \mathbb{E}_{t} \left[ J(\theta_t) \nabla J(\theta_t)-J(\theta_t) \nabla \hat{J}(\theta_{t})+J(\theta_t) \nabla \hat{J}(\theta_{t})-\hat{J}(\theta_t) \nabla \hat{J}(\theta_{t}) \right] \right\|  \\
& \quad  - \frac{L}{2} \alpha_{t}^{2} \mathbb{E}_{t} \left[ \| \nabla \hat{L}(\theta_{t}) \|^{2} \right] \\
& \stackrel{(v)}{\geq} \mathbb{E}_{t} [L(\theta_{t})] + \alpha_{t} \mathbb{E}_{t} \left[ \| \nabla L(\theta_{t}) \|^{2} \right]  - \alpha_{t} K_{1}   \left\|\mathbb{E}_{t} \left[\nabla \hat{J}(\theta_{t}) - \nabla J(\theta_{t})  \right] \right\| \\
& \quad - \lambda \alpha_{t} K_{1}   \left\| \mathbb{E}_{t} \left[ \nabla \hat{U}(\theta_{t}) - \nabla U(\theta_{t})\right]  \right\|  \\ & \quad - 2 \alpha_{t} K_{1} \lambda  \left\| \mathbb{E}_{t} \left[ J(\theta_t) \p{\nabla \hat{J}(\theta_t)-\nabla J(\theta_t)} \right] \right\| - 2 \alpha_{t} K_{1} \lambda  \left\| \mathbb{E}_{t} \left[ \nabla \hat{J}(\theta_t) \p{J(\theta_t)-\hat{J}(\theta_t)} \right] \right\| \\
& \quad  - \frac{L}{2} \alpha_{t}^{2} \mathbb{E}_{t} \left[ \| \nabla \hat{L}(\theta_{t}) \|^{2} \right] \\
& \stackrel{(vi)}{\geq} \mathbb{E}_{t} [L(\theta_{t})] + \alpha_{t} \mathbb{E}_{t} \left[ \| \nabla L(\theta_{t}) \|^{2} \right] - \alpha_{t} K_{1} \left( 1 + \frac{2 \lambda R_{\text{max}}}{1 - \disc} \right) \underbrace{ \left\|  \mathbb{E} _{t}\left[ \nabla \hat{J}(\theta_{t}) - \nabla J(\theta_{t}) \right] \right\|}_{\text{(A)}} \\
& \quad - \lambda \alpha_{t} K_{1} \underbrace{ \left\|\mathbb{E}_{t}  \left[\nabla \hat{U}(\theta_{t}) - \nabla U(\theta_{t}) \right] \right\| }_{\text{(B)}} - \frac{L}{2} \alpha_{t}^{2} \underbrace{\mathbb{E}_{t} \left[ \| \nabla \hat{L}(\theta_{t}) \|^{2} \right]}_{\text{(C)}} \\ & \quad - \alpha_{t} K_{1} \left( \frac{2 \lambda \sqrt{d} R_{\text{max}}}{(1 - \disc)p_t} \right) \underbrace{  \left| \mathbb{E}_{t} \left[\hat{J}(\theta_{t}) - J(\theta_{t}) \right] \right| }_{\text{(D)}}, \numberthis \label{eq:SGD-ABC}
\end{align*}
where (i) follows by applying the Cauchy–Schwarz inequality to the modulus of the inner product; 
(ii) follows from the uniform upper bound \(\norm{\nabla L(\theta_t)} \leq K_{1}\), which we establish below; 
(iii) follows from substituting  
\[
\nabla L(\theta) = -\nabla J(\theta) + \lambda(\nabla U(\theta) - 2 J(\theta) \nabla J(\theta));
\]  
(iv) follows from adding and subtracting the cross term \(J(\theta_t)\nabla \hat{J}(\theta_t)\);  
(v) follows from the triangle inequality; and  
(vi) follows from the bound \(|J(\theta_t)|\leq \frac{R_{\max}}{1-\gamma}\) and   
\(
\lVert \nabla \hat{J}(\theta_t)\rVert \leq \frac{2 \sqrt{d} R_{\max}}{1-\gamma},
\)
which holds as a consequence of the definition of the SPSA gradient estimate,  
\[
\nabla \hat{J}(\theta) = \frac{\hat{J}(\theta_t + p_t \Delta_t) - \hat{J}(\theta_t)}{p_t \Delta_t}.
\]

Before we derive upper bounds for (A), (B), (C), and (D) in \eqref{eq:SGD-ABC}, 
we first establish the bound \(\norm{\nabla L(\theta_{t})}_{2} \leq K_{1}\), 
which is used in (ii), as follows: 

By Policy Gradient Theorem \citep{suttonPolicyGradientMethods1999}, we have
$$\nabla J(\theta) = \frac{1}{1 - \disc} \mathbb{E}_{(s,a) \sim \chi_\theta(\cdot,\cdot)} \left[ \nabla \log \pi_\theta(a|s) Q_{\pi_\theta}(s,a) \right],
$$ where \[
Q_{\pi_\theta}(s,a) = \mathbb{E} \left[ \sum_{t=0}^{\infty} \disc^t r(s_t,a_t) \mid s_0 = s, a_0 = a \right].
\]
We upper bound the action-value function as  
\(
|Q_{\pi_\theta}(s,a)| \leq \frac{R_{\text{max}}}{1 - \gamma}.
\) 
Furthermore, by \Cref{asm:scorefnandpolicy}, the score function satisfies  
\(
\|\nabla \log \pi_\theta(a | s)\|_2 \leq C_\psi.
\)
Thus, we obtain  
\begin{equation} \label{eq:boundonJ}  
  \| \nabla J(\theta)\|_2 \leq \frac{R_{\text{max}} C_\psi}{(1 - \gamma)^2}, \quad \forall \theta \in \mathbb{R}^d.  
\end{equation}
In the same manner, we use \eqref{eq:GradU}, which is a policy gradient-style theorem for the square-value function from \citep[Lemma~1]{l.a.VarianceconstrainedActorcriticAlgorithms2016}, to upper bound the norm of the square-value function below. \(W_{\pi_{\theta}}(s,a)\) is the action-value function corresponding to the square-value function, 
i.e.,  
\(
U(\theta) = \E_{a\sim\pi_{\theta}}[W_{\pi_{\theta}}(s,a)],
\)  
similar to \(Q_{\pi_{\theta}}(s,a)\). 
\begin{align*}
& \|\nabla U(\theta)\|_2 
\\ & = \frac{1}{1 - \disc^2} \left\| \sum_{s,a} \tilde{\nu}_{\pi_\theta}(s,a) \nabla \log \pi_\theta(a|s) W_{\pi_\theta}(s,a) + 2 \disc \sum_{s,a,s'} \tilde{\nu}_{\pi_\theta}(s,a) P(s'|s,a) \nabla V_{\pi_\theta}(s') \right\| \\
&\leq \frac{1}{1 - \disc^2} \sum_{s,a} \|\tilde{\nu}_{\pi_\theta}(s,a) \nabla \log \pi_\theta(a|s)\| \, |W_{\pi_\theta}(s,a)| \\ & \quad + \frac{2 \disc}{1 - \disc^2} \sum_{s,a,s'} \|\tilde{\nu}_{\pi_\theta}(s,a)\| |P(s'|s,a)| \|\nabla V_{\pi_\theta}(s')\| \\
&\leq \frac{1}{1 - \disc^2} \|\nabla \log \pi_\theta(a|s)\| \sum_{s,a} \tilde{\nu}_{\pi_\theta}(s,a) W_{\pi_\theta}(s,a) \\ & \quad + \frac{2 \disc}{1 - \disc^2} \sum_{s,a,s'} \tilde{\nu}_{\pi_\theta}(s,a) P(s'|s,a) \|\nabla V_{\pi_\theta}(s')\| \\
&\leq \frac{C_\psi}{1 - \disc^2} \sum_{s,a} \tilde{\nu}_{\pi_\theta}(s,a) W_{\pi_\theta}(s,a) + \frac{2 \disc}{1 - \disc^2} \sum_{s,a,s'} \tilde{\nu}_{\pi_\theta}(s,a) P(s'|s,a) \|\nabla V_{\pi_\theta}(s')\| \\
&\leq \frac{C_\psi R_{\text{max}}}{(1 - \disc^2)(1 - \disc)^2} + \frac{2 \disc R_{\text{max}} C_\psi}{(1 - \disc^2)(1 - \disc)^2}  \numberthis \label{eq:NormGradU}
\end{align*}
Combining \eqref{eq:boundonJ} and \eqref{eq:NormGradU}, we obtain $K_1$:
\begin{samepage}
\begin{align*}
\|\nabla L(\theta_t)\| \leq & \|\nabla J(\theta_t)\| + \lambda \|\nabla U(\theta_t)\| + 2 \lambda |J(\theta_t)| \|\nabla J(\theta_t)\| \\
\leq & \frac{R_{\text{max}} C_\psi}{(1 - \disc)^2} + 2 \lambda \frac{R_{\text{max}} C_\psi}{(1 - \disc)^3} + \lambda \|\nabla U(\theta_t)\| \\
\leq & \frac{R_{\text{max}} C_\psi}{(1 - \disc)^2} + 2 \lambda \frac{R_{\text{max}} C_\psi}{(1 - \disc)^3} + \lambda \left( \frac{C_\psi R_{\text{max}}}{(1 - \disc^2)(1 - \disc)^2} + \frac{2 \disc R_{\text{max}} C_\psi}{(1 - \disc^2)(1 - \disc)^2} \right) \\ 
=& K_1
\end{align*}
\end{samepage}
Next, we bound (A) in \eqref{eq:SGD-ABC} as follows:
\begin{align*}
& \left\| \mathbb{E}_{t} \left[   \nabla \hat{J}(\theta_{t})  - \nabla J(\theta_{t}) \right] \right\|  \leq d^{\frac 12}  \left| \mathbb{E}_{t} \left[ \nabla_{i} \hat{J}(\theta_{t})  - \nabla_{i} J(\theta_{t}) \right] \right|  \\
&   \left| \mathbb{E}_{t} \left[ \nabla_{i} \hat{J}(\theta_{t})  - \nabla_{i} J(\theta_{t}) \right] \right|  \stackrel{(a)}{=}  \left| \mathbb{E}_{t} \left[  \frac{\phi_{v}(s_0)^{\top} v_{m}^{+} - \phi_{v}(s_0)^{\top} v_{m}}{p_{t} \Delta_{i}(t)} - \nabla_{i} J(\theta_{t}) \right]\right| \\
& \stackrel{(b)}{=}   \left|\mathbb{E}_{t} \left[\tfrac{\phi_{v}(s_0)^{\top} v_{m}^{+} - \phi_{v}(s_0)^{\top} v_{m} + \phi_{v}(s_0)^{\top} \bar{v}^{+} - \phi_{v}(s_0)^{\top} \bar{v}^{+} -\phi_{v}(s_0)^{\top} \bar{v}+\phi_{v}(s_0)^{\top} \bar{v}}{p_{t} \Delta_{i}(t)} - \nabla_{i} J(\theta_{t}) \right]  \right| \\
& \stackrel{(c)}{=} \left| \mathbb{E}_{t} \left[ \frac{\phi_{v}(s_0)^{\top} (\bar{v}^{+} - \bar{v})}{p_{t} \Delta_{i}(t)} + \frac{\phi_{v}(s_0)^{\top} (v_{m}^{+} - \bar{v}^{+})+\phi_{v}(s_0)^{\top} (\bar{v} - v_{m})}{p_{t} \Delta_{i}(t)}  - \nabla_{i} J(\theta_{t}) \right] \right|  \\
& \leq \underbrace{\left| \mathbb{E}_{t} \left[ \frac{J(\theta_{t} + p_{t} \Delta(t)) - J(\theta_{t})}{p_{t} \Delta_{i}(t)} - \nabla_{i} J(\theta_{t}) \right]\right|}_{\text{(I)}}+ \underbrace{\left| \mathbb{E}_{t} \left[ \frac{\phi_{v}(s_0)^{\top} (v_{m}^{+} - \bar{v}^{+}) + \phi_{v}(s_0)^{\top} (\bar{v} - v_{m})}{p_{t} \Delta_{i}(t)} \right]\right| }_{\text{(II)}}, \label{eq:Bound-A}\numberthis 
\end{align*}
where (a) follows from substituting the value of the SPSA gradient estimate $\nabla_i \hat{J}(\theta_t)$;  
(b) follows from adding and subtracting $\fiv(s_{0})\tr\bar{v}^{+}$ and $\fiv(s_{0})\tr\bar{v}$,  
where $\bar{v}$ and $\bar{v}^{+}$ denote the fixed points for the unperturbed and perturbed policies, respectively;  
(c) follows from rearranging the terms;  
\eqref{eq:Bound-A} follows from \Cref{asm:criticapproxerror} (which states that the critic approximation error at the fixed point is zero).  
Consequently, the first term in (I) equals the actual value function.

We bound (I) in \eqref{eq:Bound-A} as follows: 
\begin{align*}
& \left| \mathbb{E}_{t} \left[\frac{J(\theta_t + p_t \Delta_i(t)) - J(\theta_t)}{p_t \Delta_i(t)} - \nabla_i J(\theta_t) \right]\right|  \\
& \stackrel{(a)}{\leq} \left| \mathbb{E}_{t} \left[  \frac{p_t (\nabla J(\theta_t))^\top \Delta(t) + \frac{L_J}{2} p_t^2 \| \Delta(t) \|^2}{\Delta_i(t) p_t} - \nabla_i J(\theta_t) \right]\right| \\
& \stackrel{(b)}{\leq} \left|\E_{t} \psq{ \sum_{j \neq i} \left( \frac{\Delta_j(t)}{\Delta_i(t)} \right) \nabla_j J(\theta_t)} \right| + \left| \E_t\psq{\frac{L_J p_t \| \Delta(t) \|^2}{2}}\right| \\
& \stackrel{(c)}{\leq} \frac{d L_J p_t}{2}, \label{eq:A1}\numberthis
\end{align*}
where (a) follows from the second-order Taylor expansion of \( J(\theta_t + p_t \Delta_i(t)) \) around \( \theta_t \), leveraging the fact that \( J(\theta) \) has a Lipschitz gradient (with constant \( L_J \)) to bound the quadratic term; 
(b) follows from the triangle inequality and expanding the inner product into a summation over components. Here, the first term has an expectation of zero because \(\Delta(t)\) is a Rademacher vector. Specifically, each component \(\Delta_j(t)\) satisfies \(\mathbb{E}_t[\Delta_j(t)] = 0\), and the independence of \(\Delta_j(t)\) and \(\Delta_i(t)\) ensures that the expectation of the ratio \(\frac{\Delta_j(t)}{\Delta_i(t)}\) is also zero. By the linearity of expectation, the entire summation contributes zero in expectation; 
(c) follows from bounding \( \| \Delta(t) \| \leq \sqrt{d} \).

We bound (II) in \eqref{eq:Bound-A} as follows: 
\begin{align*}
& \bigg{|} \mathbb{E}_{t} \left[\frac{ \phi_v(s_0)^{\top} (v_m^{+} - \bar{v}^{+}) + \phi_v(s_0)^{\top} (\bar{v} - v_m)}{p_t \Delta_i(t)} \right] \bigg{|} \\
& \stackrel{(a)}{\leq} \bigg{|} \mathbb{E}_{t} \left[ \frac{\|\phi_v(s_0)\| \|v_m^{+} - \bar{v}^{+}\| + \|\phi_v(s_0)\| \|\bar{v} - v_m\|}{p_t \Delta_i(t)} \right] \bigg{|} \\
& \stackrel{(b)}{\leq} \frac{\fivmax}{p_t} \left( \mathbb{E}_{t} \left[ \|v_m^{+} - \bar{v}^{+}\| \right] + \mathbb{E}_{t} \left[ \|v_m - \bar{v}\| \right] \right) \\
& \stackrel{(c)}{\leq} \frac{\fivmax}{p_t \sqrt{m}} 
\underbrace{\left( \frac{10^{\frac{1}{2}} e^{\frac{-k \lr \mu}{2}}}{\disc^2 \mu} \left( \max_{\theta_{i=1,\dots,n}} \mathbb{E}[\|w_0 - \bar{w}\|] \right)^{\frac{1}{2}} + \frac{10^{\frac{1}{2}} \sigma}{\mu} \right)}_{K_2} \\
& \stackrel{(d)}{\leq} \frac{\fivmax K_2}{p_t \sqrt{m}},
\label{eq:A2}\numberthis
\end{align*} 
where (a) follows from the Cauchy-Schwarz inequality; 
(b) follows from the upper bound on the norm of the features (\Cref{asm:bddFeatures}) and linearity of expectation; 
(c) follows from bounding the terms using the tail-averaged critic error bound in \eqref{eq:tailavg-expt-bd}; 
(d) follows from defining \( K_2 \) in step (c).

Combining \eqref{eq:A1} and \eqref{eq:A2} in \eqref{eq:Bound-A}, we obtain an upper bound for (A) in \eqref{eq:SGD-ABC} as: 
\begin{align*}
\left\| \mathbb{E}_{t} \left[ \nabla \hat{J}(\theta_t) - \nabla J(\theta_t) \right] \right\|
& \leq \frac{d^{\frac 32} L_J p_t}{2} + \frac{d^{\frac 12} \fivmax K_2}{p_t \sqrt{m}}. \numberthis \label{eq:Bound(A)} 
\end{align*}

We obtain the upper bound for (B) in \eqref{eq:SGD-ABC} using arguments parallel to those used to derive the upper bound for (A). The only difference lies in the feature vector, where \(\fiumax\) replaces \(\fivmax\). 
\begin{align*}
\left\| \mathbb{E}_{t} \left[ \nabla \hat{U}(\theta_t) - \nabla U(\theta_t) \right] \right\|
& \leq \frac{d^{\frac{3}{2}} L_U p_t}{2} + \frac{d^{\frac{1}{2}} \fiumax K_2}{p_t \sqrt{m}}. \numberthis \label{eq:Bound(B)}
\end{align*}

Next, we bound (C) in \eqref{eq:SGD-ABC} as follows:

The SPSA gradient estimate of the Lagrangian is denoted as 
\[
\nabla \hat{L}(\theta_t) = \nabla \hat{J}(\theta_t) - \lambda \left( \nabla \hat{U}(\theta_t) - 2 {\hat{J}}(\theta_t) \nabla \hat{J}(\theta_t) \right).
\]
Taking the expectation with respect to the sigma field \(\mathcal{F}_t = \sigma(\theta_k, k \leq t)\), denoted by \(\mathbb{E}_t\), we have
\begin{align*}
 & \mathbb{E}_{t} [\| \nabla \hat{L}(\theta_t) \|_2^2]  \stackrel{(a)}{\leq} 3 \mathbb{E}_{t}[\| \nabla \hat{J}(\theta_t) \|_2^2] + 3 \lambda^2 \mathbb{E}_{t}[ \| \nabla \hat{U}(\theta_t) \|_2^2] + 12 \lambda^2 \left( \frac{R_{\text{max}}}{1 - \disc} \right)^2 \mathbb{E}_{t}[\| \nabla \hat{J}(\theta_t) \|_2^2] \\
& \stackrel{(b)}{\leq} \max \left\{ 3 + 3 \left( \frac{2 \lambda R_{\text{max}}}{1 - \disc} \right)^2, 3 \lambda^2 \right\} \left( \| \nabla \hat{J}(\theta_t) \|_2^2 + \| \nabla \hat{U}(\theta_t) \|_2^2 \right) \\
& \stackrel{(c)}{\leq} \max \left\{ 3 + 3 \left( \frac{2 \lambda R_{\text{max}}}{1 - \disc} \right)^2, 3 \lambda^2 \right\} \left( d \left( \frac{2 R_{\text{max}}}{1 - \disc} \right)^2 \frac{1}{p_t^2} + d \left( \frac{2 R_{\text{max}}^2}{(1 - \disc)^2} \right)^2 \frac{1}{p_t^2} \right) \\
& \stackrel{(d)}{=} \frac{K_{3}}{p_{t}^2}, \numberthis \label{eq:Bound(C)}
\end{align*}
where (a) follows from \(\|a + b + c\|^2 \leq 3\|a\|^2 + 3\|b\|^2 + 3\|c\|^2\); (b) follows from taking the maximum of all coefficients; (c) follows from bounding the SPSA gradient estimate \(\left\|\frac{J(\theta_t + p_t \Delta_i(t)) - J(\theta_t)}{p_t \Delta_i(t)}\right\|^2 \leq~\left(\frac{2R_{\text{max}}}{(1 - \disc)p_t}\right)^2\) for the first term and similarly bounding the SPSA gradient estimate of the square-value function for the second term; and (d) follows by defining \( K_3 \) as a constant, which is the coefficient of \(\frac{1}{p_t^2}\) in (c).

Now, substituting the bounds obtained for (A) from \eqref{eq:Bound(A)}, (B) from \eqref{eq:Bound(B)}, and (C) from \eqref{eq:Bound(C)} into \eqref{eq:SGD-ABC}, we get
\begin{samepage}
\begin{align*}
& \mathbb{E}_{t} [L(\theta_{t+1})] \geq \mathbb{E}_{t} [L(\theta_{t})] \!+\! \alpha_t \mathbb{E}_{t} \left[ \|\nabla L(\theta_{t}) \|^{2} \right]  \!-\! \alpha_t K_1 \left( 1 + \frac{2 \lambda R_{\text{max}}}{1 - \disc} \right) \underbrace{\left\| \mathbb{E}_{t} \left[ \nabla \hat{J}(\theta_t) - \nabla J(\theta_t) \right] \right\|}_{\text{(A)}} \\
& \phantom{\hskip 60pt} - \lambda \alpha_t K_1 \underbrace{\left\| \mathbb{E}_{t} \left[ \nabla \hat{U}(\theta_t) - \nabla U(\theta_t) \right] \right\|}_{\text{(B)}} - \underbrace{\frac{L}{2} \alpha_t^2 \mathbb{E}_{t} \left[ \|\nabla \hat{L}(\theta_t) \|^2 \right]}_{\text{(C)}} \\
& \geq \mathbb{E}_{t} [L(\theta_t)] + \alpha_t \mathbb{E}_{t} \left[ \|\nabla L(\theta_t) \| \right] - \alpha_t K_1 \left( 1 + \frac{2 \lambda R_{\text{max}}}{1 - \disc} \right) \left( \frac{d^{\frac 32} L_J p_t}{2} + \frac{d^{\frac 12} \fivmax K_2}{p_t \sqrt{m}} \right) \\
& \quad - \lambda \alpha_t K_1 \left( \frac{d^{\frac 32} L_U p_t}{2} + \frac{d^{\frac 12} \fiumax K_2}{p_t \sqrt{m}} \right) - \frac{L \alpha_t^2}{2} \left( \frac{K_3}{p_t^2} \right)
\end{align*}
\end{samepage}
Rearranging the terms, we obtain
\begin{align*}
& \alpha_t \mathbb{E}_{t} \left[ \| \nabla L(\theta_t) \|^2 \right] 
\leq \mathbb{E}_{t} [L(\theta_{t+1})] - \mathbb{E}_{t} [L(\theta_{t})] \\ & \phantom{\hskip 85pt}+ \alpha_t K_1 \left( 1 + \frac{2 \lambda R_{\text{max}}}{1 - \disc} \right) \left( \frac{d^{\frac{3}{2}} L_J p_t}{2} + \frac{d^{\frac{1}{2}} \fivmax K_2}{p_t \sqrt{m}} \right) \\
& \phantom{\hskip 85pt} + \lambda \alpha_t K_1 \left( \frac{d^{\frac{3}{2}} L_U p_t}{2} + \frac{d^{\frac{1}{2}} \fiumax K_2}{p_t \sqrt{m}} \right) + \frac{L_1 \alpha_t^2 K_3}{2 p_t^2} \\
& \stackrel{(a)}{\leq} \mathbb{E} [H_t] - \mathbb{E} [H_{t+1}] + \frac{\alpha_t K_1 d^{\frac{3}{2}}}{2} \left( L_J \left( 1 + \frac{2 \lambda R_{\text{max}}}{1 - \disc} \right) + \lambda L_U \right) p_t \\
& \quad + \alpha_t K_1 K_2 d^{\frac{1}{2}} \left( \left( 1 + \frac{2 \lambda R_{\text{max}}}{1 - \disc} \right) \left( \fivmax + \lambda \fiumax \right) \right) \frac{1}{p_t \sqrt{m}} + \frac{\alpha_t^2 L_1 K_3}{2 p_t^2}, 
\end{align*}
\begin{align*}
\mathbb{E}_{t} \left[ \| \nabla L(\theta_t) \|^2 \right] & \stackrel{(b)}{\leq}
\frac{1}{\alpha_t} \left( \mathbb{E} \left[ H_{t+1} \right] - \mathbb{E} \left[ H_{t} \right] \right) + \frac{K_1 d^{\frac{3}{2}}}{2} \left( L_J \left( 1 + \frac{2 \lambda R_{\text{max}}}{1 - \disc} \right) + \lambda L_U \right) p_t \\
& \quad + K_1 K_2 d^{\frac{1}{2}} \left( \left( 1 + \frac{2 \lambda R_{\text{max}}}{1 - \disc} \right) \left( \fivmax + \lambda \fiumax \right) \right) \frac{1}{p_t \sqrt{m}} + \frac{\alpha_t L_1 K_3}{2 p_t^2},
\end{align*}
where (a) follows from defining \(H_t = L(\theta_t) - L(\theta_{*})\), where \(\theta_{*}\) is the optimal policy, and (b) follows from dividing both sides by \(\alpha_t\).

Summing from \(t = 1\) to \(n\), and taking the total expectation, we get
\begin{align*}
\sum_{t=1}^{n} \mathbb{E} \left[ \|\nabla L(\theta_t)\|^2 \right] 
& \leq \frac{C_1}{\alpha_t} + C_2 \sum_{t=1}^{n} p_t + \frac{C_3}{\sqrt{m}} \sum_{t=1}^{n} \frac{1}{p_t} + C_4 \sum_{t=1}^{n} \frac{\alpha_t}{p_t^2}.
\end{align*}
Here, we obtain \(|L(\theta)| \leq C_{1} = \frac{2R_{\text{max}}}{1-\disc} \left( 1 + \frac{\lambda R_{\text{max}}}{1-\disc} \right)\) after a telescoping sum. 

\noindent Dividing by \(n\) on both sides and setting \(\alpha_t = \alpha, p_t = p\), we get
\begin{align*}
\frac{1}{n} \sum_{t=1}^{n} \mathbb{E} \left[ \|\nabla L(\theta_t)\|^2 \right] 
& \leq \frac{C_1}{n \alpha} + C_2 p + \frac{C_3}{\sqrt{m} p} + \frac{C_4 \alpha}{p^2}.
\end{align*}
Setting \(\alpha = n^a, \, p = n^b, \, m = n^c\), we have
\begin{align*}
\mathbb{E} \left[ \|\nabla L(\theta_R)\|^2 \right] 
& \leq C_1 n^{-1-a} + C_2 n^b + C_3 n^{-b-c/2} + C_4 n^{a-2b}.
\end{align*}
Optimizing for \(a, b, c\), we find their values to be \(a = -\frac{3}{4}, \, b = -\frac{1}{4}, \, c = 1\). Substituting these values, we get
\begin{align*}
\mathbb{E} \left[ \|\nabla L(\theta_R)\|^2 \right] 
& \leq C_1 n^{-1/4} + C_2 n^{-1/4} + C_3 n^{-1/4} + C_4 n^{-1/4} \\
& = O(n^{-1/4}).
\end{align*}
\end{proof}

\end{document}